\documentclass[Afour,sageh,times]{libraries/sagej}

\usepackage{moreverb,url}
\usepackage[colorlinks,bookmarksopen,bookmarksnumbered,citecolor=red,urlcolor=red]{hyperref}
\usepackage[noadjust]{cite} \usepackage[table]{xcolor} \usepackage{amsmath} \usepackage{colortbl} \usepackage{amssymb} \usepackage{amsthm} \usepackage{mathtools} \usepackage{booktabs} \usepackage{graphicx} \usepackage{subcaption} 
\usepackage{upgreek,bm}
\usepackage{multirow} \usepackage{epstopdf}
\usepackage{todonotes}
\usepackage{float} \usepackage{adjustbox}    
\usepackage{threeparttable} \usepackage{algorithm}\usepackage{algorithm,algpseudocode,array}
\newcommand\CustomComment[1]{\State $/\!/$ \parbox[t]{\columnwidth}{ #1}}
\usepackage{xfrac} \usepackage{makecell} \usepackage{fancyhdr}  

\PassOptionsToPackage{obeyspaces}{url} 

\DeclareMathSymbol{\shortminus}{\mathbin}{AMSa}{"39}

\newcommand{\R}[2]{{}^{#2}_{#1}\mathbf{R}}
\newcommand{\Rhat}[2]{{}^{#2}_{#1}\hat{\mathbf{R}}}
\newcommand{\ang}[2]{{}^{#2}_{#1}\boldsymbol{\theta}}

\newcommand{\angtilde}[2]{{}^{#2}_{#1}\tilde{\boldsymbol{\theta}}}

\newcommand{\p}[2]{{}^{#2}\mathbf{p}_{#1}}
\newcommand{\phat}[2]{{}^{#2}\hat{\mathbf{p}}_{#1}}
\newcommand{\ptilde}[2]{{}^{#2}\tilde{\mathbf{p}}_{#1}}
\newcommand{\vel}[2]{{}^{#2}\mathbf{v}_{#1}}

\newcommand{\skw}[1]{\lfloor {#1} \rfloor}
\newcommand{\x}[1]{\mathbf{x}_{#1}}
\newcommand{\xhat}[1]{\hat{\mathbf{x}}_{#1}}
\newcommand{\xtilde}[1]{\tilde{\mathbf{x}}_{#1}}
\newcommand{\Exp}[1]{\textrm{Exp}(#1)}
\newcommand{\Log}[1]{\textrm{Log}(#1)}
\newcommand{\Jr}[1]{\mathbf{J_r}(#1)}
\newcommand{\Jl}[1]{\mathbf{J_l}(#1)}
\newcommand{\Jrinv}[1]{\mathbf{J^{\shortminus 1}_r}(#1)}
\newcommand{\Jlinv}[1]{\mathbf{J^{\shortminus 1}_l}(#1)}
\newcommand{\cp}[2]{{}^{#2}\boldsymbol{\Pi}_{#1}}
\newcommand{\cphat}[2]{{}^{#2}\hat{\boldsymbol{\Pi}}_{#1}}
\renewcommand{\H}[1]{\mathbf{H}_{#1}}
\renewcommand{\O}[1]{\mathbf{0}_{#1}}
\newcommand{\I}[1]{\mathbf{I}_{#1}}

\renewcommand{\t}[2]{{}^{#2}{t}_{#1}}
\newcommand{\HS}[1]{\mathbf{H}_{#1}^{S}}
\newcommand{\HA}[1]{\mathbf{H}_{#1}^{A}}

\newcommand{\J}[2]{\frac{\partial #1}{\partial #2}}

\newcommand\undermat[2]{\makebox[0pt][l]{$\smash{\underbrace{\phantom{\begin{matrix}#2\end{matrix}}}_{\text{$#1$}}}$}#2}

\newcommand\scalemath[2]{\scalebox{#1}{\mbox{\ensuremath{\displaystyle #2}}}}
\newtheorem{thm}{Theorem}[section]

\newtheorem{lem}[thm]{Lemma}

\allowdisplaybreaks
\DeclareGraphicsExtensions{.PNG,.JPG,.pdf,.png,.jpg,.jpeg} \newcommand\BibTeX{{\rmfamily B\kern-.05em \textsc{i\kern-.025em b}\kern-.08emT\kern-.1667em\lower.7ex\hbox{E}\kern-.125emX}}
\def\volumeyear{2023}
\begin{document}

\runninghead{Woosik Lee et al.}

\title{MINS: Efficient and Robust Multisensor-aided \\Inertial Navigation System}

\author{Woosik Lee, Patrick Geneva, Chuchu Chen, and Guoquan Huang}

\affiliation{Robot Perception and Navigation Group,
University of Delaware, Newark, DE 19716, USA.
}

\corrauth{Woosik Lee, Department of Mechanical Engineering, University of
Delaware, 126 Spencer Lab, Newark, DE 19716, USA.
}
\email{woosik@udel.edu}

\markboth{Journal of \LaTeX\ Class Files,~Vol.~14, No.~8, August~2023}{Shell \MakeLowercase{\textit{et al.}}: A Sample Article Using IEEEtran.cls for IEEE Journals}

\begin{abstract}
Robust multisensor fusion of multi-modal measurements such as IMUs, wheel encoders, cameras, LiDARs, and GPS holds great potential due to its innate ability to improve resilience to sensor failures and measurement outliers, thereby enabling robust autonomy.
To the best of our knowledge, this work is among the first to develop a consistent tightly-coupled Multisensor-aided Inertial Navigation System (MINS) that is capable of fusing the most common navigation sensors in an efficient filtering framework, 
by addressing the particular challenges of computational complexity, sensor asynchronicity, and intra-sensor calibration.
In particular, 
we propose a consistent high-order on-manifold interpolation scheme to enable efficient asynchronous sensor fusion and state management strategy (i.e. dynamic cloning).
The proposed dynamic cloning leverages motion-induced information to adaptively select interpolation orders to control computational complexity while minimizing trajectory representation errors.
We perform online intrinsic and extrinsic (spatiotemporal) calibration of all onboard sensors to compensate for poor prior calibration and/or degraded calibration varying over time.
Additionally, we develop an initialization method with only proprioceptive measurements of IMU and wheel encoders, instead of exteroceptive sensors, which is shown to be less affected by the environment and more robust in highly dynamic scenarios.
We extensively validate the proposed MINS in simulations and large-scale challenging real-world datasets, outperforming the existing state-of-the-art methods, in terms of localization accuracy, consistency, and computation efficiency.
We have also open-sourced our algorithm, simulator, and evaluation toolbox for the benefit of the community: \url{https://github.com/rpng/mins}.

\end{abstract}


\maketitle \section{Introduction}

Robust, accurate, real-time localization is a fundamental capability for autonomous vehicles.
This is often addressed by leveraging multisensor fusion approaches, 
in part due to the fact that multi-modal sensors can offer complementary and/or redundant sensing under different environmental conditions \citep{lynen2013robust, hackett1990multi}.
Among all possible navigation sensors, inertial measurement units (IMUs), cameras, wheel encoders, global navigation satellite systems (GNSS), and LiDARs prevail in recent literature \citep{cao2022gvins, lv2023continuous, wang2021gr}.
While it appears to be straightforward in principle to fuse all these heterogeneous sensors to achieve optimal localization performance, there is only limited work that has used greater than three sensing modalities because of the inherent difficulties of robust and efficient multisensor fusion.

As multisensor systems -- in particular, those that forgo hardware-synchronization due to its significant cost -- inevitably produce asynchronous or out-of-sequence data, it remains challenging to optimally model and update such asynchronous and delayed multi-modal information.
The ability to seamlessly handle these is a key attribute of multisensor systems as it endows the systems with greater flexibility to use more sensors.
To this end, multisensor calibration is another essential but difficult component.
Many existing methods assume perfect offline calibration~\citep{sun2022multisensor, nguyen2021viral, meng2017robust}, which, however, inject unmodelled errors into the estimator and thus degrade localization performance if the prior calibration was poor or inevitably changes during long-term operation.
As incorporating more sensors increases not only the number of measurements to process but also the size of the state to be estimated,
highly-efficient fusion of all available multi-modal measurements is always challenging given stringent computational resources and latency requirements. 
Additionally, system initialization is necessary for robust estimation and failure recovery but practically is challenging due to constrained under-actuated motion and short time horizons.

To address the aforementioned challenges, 
in this work, we develop a novel Multisensor-aided Inertial Navigation System (MINS), 
which builds upon the inertial navigation system (INS) as the backbone and 
efficiently fuses the multi-modal measurements of wheel encoders, cameras, LiDARs, and GNSS while performing full online spatiotemporal calibration.
The proposed MINS employs a continuous-time state representation to reduce the computational complexity of asynchronous and delayed measurements and proposes robust initialization in challenging dynamic scenes, while achieving accurate real-time localization.

In particular, the main contributions of this work include:
\begin{itemize}

\item We, for the first time, design an efficient and robust filtering-based MINS estimator that fuses five most commonly-seen navigation sensors in a tightly-coupled fashion, including an IMU, a pair of wheel encoders, and an arbitrary number of cameras, LiDARs, and GNSS receivers. 
Thanks to the complementary and redundant sensing capabilities, the proposed approach is resilient to sensor failures and measurement depletion (e.g. in a dark room for cameras, open field for LiDARs, or indoors for GNSS) as well as measurement outliers, 
enabling it to output continuous and uninterrupted estimates for downstream applications.

\item The proposed MINS introduces a consistent high-order on-manifold state interpolation methodology to process asynchronous and delayed measurements, which enables flexibility in adding auxiliary sensors while ensuring computational efficiency. 
Furthermore, we investigate the long-standing loss of trajectory fidelity introduced by continuous-time representations and reformulate the measurement model to account for trajectory representation errors.
A dynamic cloning algorithm which adaptively balances the computational complexity, estimation accuracy, and trajectory representation errors is developed and validated.

\item The proposed MINS performs online intrinsic and extrinsic (spatiotemporal) calibration of all sensors onboard, which enables the system to be robust to poor prior calibration or time-varying calibration. 
Moreover, we develop a robust initialization technique using only the IMU and wheel-encoder proprioceptive measurements for the case of high speed and dynamic object filled environments, where conventional methods which leverage exteroceptive sensors fail to initialize.

\item We validate the proposed MINS extensively in both realistic simulations and large-scale real-world datasets,
showing that the proposed method can achieve high efficiency, robustness, consistency, and accuracy in many challenging scenarios.
We have also open-sourced our codebase, simulator, and evaluation toolbox for the benefit of the research community.
\end{itemize}

\begin{figure}[t]
\centering
\includegraphics[trim=0mm 0mm 0mm 0mm,clip,width=1\columnwidth]{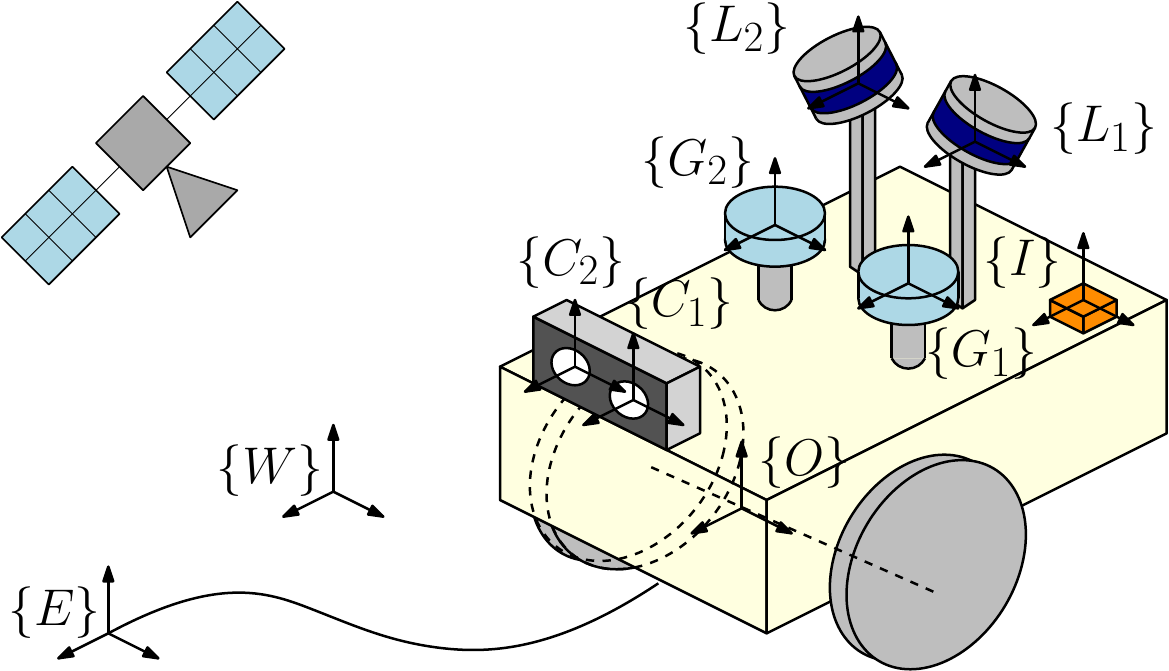}
\caption{Frames used in the proposed MINS. In this example, 2 cameras $\{C\}$, 2 GNSS sensors $\{G\}$, 2 LiDARs $\{L\}$, 1 IMU $\{I\}$, and 1 wheel odometry frame $\{O\}$ are shown along with local world frame $\{W\}$ and global ENU frame $\{E\}$.}
\label{fig:frames}
\end{figure} 

\section{Related Work}

Although there is a rich literature on multi-sensor fusion \citep{chen2022slam, xu2022review},
we here only review  the most relevant multi-sensor systems with three or more sensing modalities,
while focusing on sensor asynchronicity, delayed measurements, and online sensor calibration.

\subsection{Estimation with Three Sensing Modalities}
Visual-Inertial Navigation Systems (VINS) that combine the inertial reading of the IMU and the visual bearing information of the camera have become popular due to their complementary nature \citep{Huang2019ICRA}.
Many multi-sensor navigation systems encompass VINS as a key module of the whole system due to its compactness and rich information.

\subsubsection{LiDAR-VINS} ~~\\
LiDAR is one of the popular sensors to be fused with VINS due to its measurement being invariant to the light and its ability to provide depth information directly, which can compensate for the degeneracy of the camera in many scenarios, such as in a dark cave or facing featureless wall \citep{xu2022review}.
V-LOAM \citep{zhang2018laser} is one of the early works that loosely couples a LiDAR to the camera and IMU frame-to-frame leveraging VINS to get the initial matching of the LiDAR point cloud. 
\citet{wang2019robust} improved this by designing a graph-based estimator including loop closure, while similar approaches can be found in many other works \citep{camurri2020pronto, khattak2020complementary}. 
A focus of loosely coupled systems has been to address the degeneracy issues of the independent camera and LiDAR odometry algorithms.
For example LiDAR (+ IMU) odometry is combined with VINS in a graph formulation as separate sub-systems to enable robustness \citep{shan2021lvi, shao2019stereo, zhao2021super}.
Many studies adapt filtering frameworks to design more tightly coupled systems which come with the estimate consistency and accuracy gains.
MSCKF-based \citep{Mourikis2007ICRA} designs are popular and adapted by many works to great success \citep{Zuo2019IROS, Zuo2020IROS, Lee2021ICRA, Yang2019IROS}.
In those works, features, such as planes or lines, are extracted from the LiDAR measurements and fused with camera measurements in a tightly coupled way.
Similarly, the error state iterated Kalman filter (ESIKF) is adapted \citep{lin2021r, lin2022r, zheng2022fast} while fusing LiDAR without extracting the features (a.k.a. direct method) was also studied \citep{lin2022r, zheng2022fast}.

\subsubsection{GNSS-VINS} ~~\\
A GNSS sensor directly provides absolute position information, although the accuracy of it is highly dependent on the surrounding environment and the availability of external correction data.
For these reasons, GNSS fusion with VINS has been investigated to provide locally precise and globally accurate localization.
There are many works fusing GNSS with VINS in a graph formulation via both loosely \citep{qin2019general, merfels2016pose, merfels2017sensor} and tightly \citep{mascaro2018gomsf, cioffi2020tightly, rehder2012global, wang2021direct, han2022tightly}.
One of the most popular graph-based works is VINS-Fusion \citep{qin2019general} which takes both VINS estimation and GNSS measurements to build a pose graph for optimization.
However, GNSS measurements are often intermittent and asynchronous to other sensors which max graph construction a complex design.
\citet{cioffi2020tightly} addressed this issue by creating synthetic synchronized global measurements using both IMU and GNSS measurements, but introduced inconsistent measurement re-use of inertial information.
There exist many filter-based techniques \citep{lee2016camera, ramanandan2019systems, chambers2014robust, lynen2013robust} which fused GNSS with VINS, with \citet{Lee2020ICRA} being the first to address the asynchronicity issue by interpolating historical stochastically cloned IMU poses.
More recent works have focused on fusing lower-level GNSS measurements (GNSS satellite signals) with VINS \citep{dong2022tightly, liu2021optimization, Lee2022ICRA, cao2022gvins} showing robustness in limited FOV scenarios and improved localization performance.

\subsubsection{Wheel-VINS} ~~\\
Many of the simultaneous localization and mapping (SLAM) applications are on ground vehicles. 
Therefore, it appears to be straightforward in principle to fuse the wheel (encoder) with other sensors for ground robot navigation.
Especially for VINS on wheeled robots, it is known that their estimated state can have addititional unobservable directions beyond the standard 4 degrees of freedom (DOF) depending on the motion \citep{wu2017vins} (calibration parameters have additional challenges, \citep{Yang2019RAL}).
Thus the fusion of additional scale information from other sensors is necessary to build a consistent estimator.
To this end, wheel odometry fusion with VINS has been investigated in both graph form \citep{kang2019vins, dang2018tightly, zheng2019visual, liu2021bidirectional, Zuo2019ISRR} and the filter form \citep{wu2017vins, Lee2020IROS, serov2021visual, ma2019ack, ma2020consistent}. 

One of the problems of wheel fusion is that the readings typically provide 2D motion information while the estimated robot state is generally in 3D space.
Many studies address this issue through either tightly coupled integration of inertial and the wheel in 3D \citep{liu2019visual, su2021gr, quan2019tightly, liu2021bidirectional, jung2020monocular}, projecting the 3D state onto the local 2D plane \citep{wu2017vins, Geneva2020IROS, Lee2020IROS}, or applying planar motion constraint \citep{kang2019vins, zheng2019visual}.
Leveraging the ground shape information, such as plane model \citep{wu2017vins} or quadratic polynomial \citep{zhang2021pose, Zuo2019ISRR}, has shown to improve the robot pose estimation.
The most common kinematic model of the wheel adapted is the differential drive \citep{wu2017vins, Lee2020IROS} while other models, such as 2 DOF vehicle dynamics model \citep{kang2019vins}, Ackermann model \citep{ma2019ack, ma2020consistent}, or ICR model \citep{Zuo2019ISRR} are used for different types of platforms.

{
\newcommand{\AS}{\begin{tabular}{@{}c@{}} Asynchronous Sensors \end{tabular}}
\newcommand{\DM}{\begin{tabular}{@{}c@{}} Delayed Measurements \end{tabular}}
\newcommand{\OC}{\begin{tabular}{@{}c@{}} Online Calibration \end{tabular}}
\newcommand{\OA}{\begin{tabular}{@{}c@{}} Dynamic \\ Initialization \end{tabular}}
\newcommand{\tmt}[2]{\multirow{#1}{*}{\rotatebox[origin=c]{90}{#2}}}
\newcommand{\tmtt}[2]{\multirow{#1}{*}{\rotatebox[origin=c]{0}{#2}}}
\begin{table*}
\begin{threeparttable}
\centering
\caption{Summary of the state-of-the-art multisensor navigation systems} 
\setlength\tabcolsep{4.8pt}
\begin{tabular}{c|c|c|c|c|c|c||c|c|c} \toprule
    & \multirow{2}{*}{Algorithms} & \multirow{2}{*}{IMU} &  \multirow{2}{*}{Camera} &  \multirow{2}{*}{GNSS} &  \multirow{2}{*}{Wheel} &  \multirow{2}{*}{LiDAR} & Asynchronous & Delayed & Online \\
      &  &        &     &       &      &      &   Sensors         &   Measurements         &    Calibration    \\ \midrule
\tmt{13}{Filter}& \textbf{MINS} (proposed)                   & 1         & N         & N         & 1         & N         & $\bigcirc$    & $\bigcirc$    & $\bigcirc$  \\
                & \citet{kubelka2015robust}  & 1         & 1         & $\times$  & 1         & 1         & $\times$      & $\times$      & $\times$      \\
                & \citet{simanek2015improving}& 1        & 1         & $\times$  & 1         & 1         & $\times$      & $\times$      & $\times$      \\
                & \citet{meng2017robust}        & 1         & $\times$  & 1         & 1         & 1         & $\times$      & $\times$      & $\times$      \\ 
                & \citet{suhr2016sensor}        & 1         & 1         & 1         & 1         & $\times$  & $\times$      & $\times$      & $\times$      \\ 
                & MSF-EKF             & 1         & 1         & N         & $\times$  & $\times$  & $\bigcirc$    & $\bigcirc$    & $\triangle$   \\ 
                & \citet{hausman2016self}    & 1         & 1         & 1         & $\times$  & $\times$  & $\bigcirc$    & $\bigcirc$    & $\triangle$   \\ 
                & \citet{tessier2006real}    & 1         & $\times$  & 1         & $\times$  & 1         & $\bigcirc$    & $\bigcirc$    & $\times$      \\ 
                & MaRS               & 1         & 2         & 1         & $\times$  & $\times$  & $\bigcirc$    & $\bigcirc$    & $\triangle$   \\ 
                & Pronto           & 1         & 2         & $\times$  & 1         & 1         & $\bigcirc$    & $\bigcirc$    & $\times$      \\
                & \citet{shen2014multi}         & 1         & 2         & 1         & $\times$  & 1         & $\bigcirc$    & $\bigcirc$    & $\times$      \\
                & \citet{wu2022multi}             & 1         & 2         & 1         & 1         & 1         & $\bigcirc$    & $\bigcirc$    & $\times$      \\ 
                & \citet{Lee2021ICRA}            & 1         & 2         & 1         & 1         & 1         & $\bigcirc$    & $\bigcirc$    & $\bigcirc$    \\  \midrule
\tmt{9}{Graph}  & VINS-Fusion          & 1         & 2         & N         & $\times$  & $\times$ & $\times$      & $\bigcirc$    & $\triangle$   \\
                & SVIn2              & 1         & N         & $\times$  & $\times$  & $\times$  & $\times$      & $\bigcirc$    & $\times$      \\
                & Lvio-Fusion             & 1         & 2         & 1         & $\times$  & 1         & $\times$      & $\bigcirc$    & $\times$      \\
                & GR-Fusion               & 1         & 2         & 1         & 1         & 1         & $\times$      & $\bigcirc$    & $\triangle$   \\
                & \citet{chiu2014constrained}   & 1         & N         & 1         & $\times$  & 1         & $\bigcirc$    & $\bigcirc$    & $\times$      \\
                & VIRAL SLAM          & 1         & 2         & $\times$  & $\times$  & N         & $\triangle$   & $\bigcirc$    & $\times$      \\
                & VIRAL-Fusion  & 1         & N         & $\times$  & $\times$  & N         & $\bigcirc$   & $\bigcirc$    & $\times$      \\
                & VILENS              & 1         & 2         & $\times$  & 1         & 1         & $\bigcirc$    & $\bigcirc$    & $\times$      \\
                & VIL-Fusion       & 1         & 2         & 1         & $\times$  & 1         & $\bigcirc$    & $\bigcirc$    & $\times$      \\
\bottomrule
\end{tabular}
\label{tab:related}
\begin{tablenotes}\footnotesize
\raggedleft \item[*] N: arbitrary number of sensors supported, $\bigcirc$: fully supported, $\triangle$: partially supported, $\times$: not supported.
\end{tablenotes}
\end{threeparttable}
\end{table*}
}

\subsection{Estimation with Four or More Sensing Modalities} \label{ch:4_systems}

Despite many advantages, such as robustness, accuracy, and applicability, a relatively small number of systems equipped with 4 or more sensor modalities have been introduced in part due to the difficulty of handling large computation and sensor measurements that are either asynchronous, delayed, or both.
However, the ability to properly handle asynchronous and delayed sensor measurements is important for multi-sensor systems because it enables the system to flexibly add an arbitrary number of sensors that are either homogeneous or heterogeneous, model measurements more precisely, and include all the information without naively dropping it.
Table \ref{tab:related} summarizes the related literature.

\subsubsection{Graph-based Systems}~~\\
The graph formulations can easily handle the delayed measurements owing to their nature of carrying all/part of the history of measurements which allows for delayed states to be inserted into the optimization problem \citep{chiu2014constrained}.
The more recent VINS-Fusion \citep{qin2019general} employs a pose graph which loosely couples VINS (tightly coupled IMU and camera) with other global sensors, such as GNSS, magnetometer, or barometer but assumes sensors are perfectly synchronized.
There exist other graph formulations, such as SVIn2 \citep{rahman2022svin2}, which fused IMU, camera, water pressure sensor, and sonar to perform underwater SLAM in pose graph formulations, Lvio-fusion \citep{jia2021lvio} tightly fused IMU, camera, and LiDAR for trajectory estimation and in turn, loosely fused GNSS in a pose graph, or GR-Fusion \citep{wang2021gr}, which tightly fused IMU, camera, LiDAR, wheel, and GNSS in the graph along with ground estimation and mapping.
However, all of these works assumed the measurements are perfectly synchronized which prevents their plug-and-play use.
The direct extension of these graph-based methods to asynchronous systems requires the estimation of the state at every measurement time, which requires techniques to bound computation, such as performing updates with only a subset of sensors \citep{chiu2014constrained}, to maintain real-time performance.
Another challenge is the inclusion of temporal calibration parameter of sensors, which typically has been ignored or handled by hardware-synchronization (i.e. PTP network time protocol \citep{eidson2002ieee}).

On the other hand, VIRAL SLAM \citep{nguyen2021viral} which fused IMU and stereo camera with multiple UWBs and LiDARs handled asynchronicity by synchronizing all the sensors to the base LiDAR measurement time.
To do so, IMU and UWBs are interpolated, non-base LiDAR point clouds are unwarped to the base LiDAR measurement time, and camera measurements that are closest to the time are admitted while others are dropped.
The system is improved to a more general multi-sensor system, VIRAL-Fusion \citep{nguyen2021viralfusion}, synchronizing the sensors by creating the state independent of the sensor measurement timings, interpolating the states to handle UWB measurements, and generating synthetic measurements at the state timestamp with measurement interpolation.
Similarly, VILENS \citep{wisth2022vilens} which fused IMU, leg odometry, camera, and LiDAR resolved the asynchronicity issue by synchronizing all the sensors to the camera time, interpolating IMU and leg measurements, and unwarping LiDAR point cloud to the camera timestamp using the integrated IMU factor.
VIL-Fusion \citep{sun2022multisensor} which fused IMU, camera, LiDAR, and GNSS also synchronized IMU, camera, and LiDAR in the same fashion while linear interpolation was used to handle GNSS.
While these works look to address the asynchronicity problem without increasing the computational complexity, they introduce an additional unmodeled source of error from the approximation.

\subsubsection{Filter-based Systems}~~\\
Compared to graph-based systems, filter-based methods remain popular due to their efficiency, which is crucial to computationally complex multi-sensor systems.
However, handling asynchronous and delayed sensor measurements presents challenges as filters only keep the latest information and maybe a short period of historical stochastic poses.
In many early filtering approaches, such as \citet{kubelka2015robust} and \citet{simanek2015improving} which fused IMU, omnidirectional camera, wheel, and LiDAR in EKF, \citet{meng2017robust} which fused IMU, GNSS, wheel, and LiDAR in unscented Kalman filter (UKF), or \citet{suhr2016sensor} which fused IMU, camera, GNSS, wheel, symbolic road map in particle filter (PF), the sensor measurements are assumed to be synchronized and time-ordered which injects unmodeled error to the estimator in case the assumption does not hold.

MSF-EKF \citep{lynen2013robust} is among the first work to fuse generic relative and absolute measurements from IMU, camera, GNSS, and pressure sensor in EKF and also handled asynchronous and delayed sensor measurements.
In their work, the current state is re-predicted forward in time from the queried state when a delayed measurement comes in, which requires all subsequent measurement updates to be reapplied, and asynchronicity is handle by propagating to the temporally closest queried state to the measurement time
Similarly spirited works that handle asynchronous and delayed sensor measurements are by \citet{hausman2016self} which fused IMU, camera, GNSS, and UWB in iterated EKF, \citet{tessier2006real} which fused magnetometer, GNSS, radar, and gyroscope in PF, MaRS \citep{brommer2020mars} which fused IMU, stereo camera, GNSS, barometer in EKF, and Pronto \citep{camurri2020pronto} which fused IMU, leg odometry, camera, and LiDAR in EKF keeping 10 seconds of worth measurements, prior \& posterior states, and the corresponding covariances.
However, keeping the previous states and the measurements, and re-processing those measurements every time a delayed measurement is received, incurs a higher than needed computational cost which may result in losing real-time estimation performance.

One of the popular methods to improve efficiency by avoiding this re-computation is using a buffer to hold the measurements and processing them in a time-ordered manner.
\citet{shen2014multi} fused IMU, camera, GNSS, 2D LiDAR, pressure altimeter, and magnetometer in UKF buffering 0.1 seconds of the measurements to handle asynchronous and delayed sensor measurements.
A similar approach can be found in the work of \citet{wu2022multi} which fused IMU, stereo camera, wheel, LiDAR, and AHRS via a system which is composed of 3 subsystems (wheel+AHRS, VINS, LiDAR odometry) where the sensor measurements are abstracted into pose or velocity information and fused in EKF formulation.
The buffer logic is useful to sort the unordered measurements, but the buffering window cannot be too large because it delays every measurement update, thereby introducing a large latency to the real-time robotic system and is unacceptable in safety critical systems.
Additionally, fusing sensors with large delays may not be feasible with a buffer approach (e.g. GNSS measurement delay can be on the order of 0.85 seconds \citep{Lee2020ICRA}).

\subsubsection{Continuous-time Estimation} ~~\\
An alternative paradigm that elegantly allows the incorporation of asynchronous and delayed sensor measurements is the continuous-time estimation which represents the robot state at an arbitrary time using temporal basis functions. 
A popular temporal basis functions is the uniform B-spline \citep{furgale2013unified, zheng2015decoupled, droeschel2018efficient, cioffi2022continuous, lv2022observability} while there are other approaches, such as Gaussian process \citep{barfoot2017state, anderson2015batch}, wavelets \citep{anderson2014hierarchical}, or Taylor-series \citep{li2014vision}.
However, uniform B-spline, in general, requires a fixed frequency between control nodes and requires sufficiently high frequency to ensure accurate trajectory fidelity in highly dynamic motions.
Too high of control node frequency may inefficiently increase computational complexity due to the estimation of a larger state vector, while a too low frequency may hurt the estimation performance due to inaccurate trajectory representation (see \citep{cioffi2022continuous} for comparison to discrete-time formulations).
In part due to these reasons, many methods that utilize B-spline for the continuous-time estimation are offline \citep{rehder2016extending, furgale2012continuous, lang2022ctrl} or are applied to small systems \citep{mueggler2018continuous, quenzel2021real, rodrigues2021online}.

Recently, MIMC-VINS \citep{Eckenhoff2021TRO} investigated a trajectory formulation using high-order polynomials over stochastic poses. 
The major benefits of this method are that it can handle a non-uniform temporal pose distribution and has a low computational cost to reconstruct different orders of polynomials as the exact robot poses are the control nodes.
The 1-order linear interpolation variant of this has been leveraged by many works to great success \citep{paul2018alternating, Geneva2018ICRA, Lee2020ICRA, sun2022multisensor}.
Delayed measurements can be naturally incorporated since the pose at \textit{any} time instance within the historical stochastic clone window can be found.
However, in the case of highly dynamic motions, the polynomial trajectory representation still has a loss of fidelity, as do temporal basis functions, as it is still limited by the clone rate and can introduce unmodeled errors.
We will directly investigate the magnitude and impact of these unmodeled errors and properly model their characteristics in the proposed MINS to enable consistent estimation.

\subsection{Online Sensor Calibration}

\begin{figure*}[t]
\centering
\includegraphics[trim=11mm 90mm 41mm 35mm,clip,width=1\textwidth]{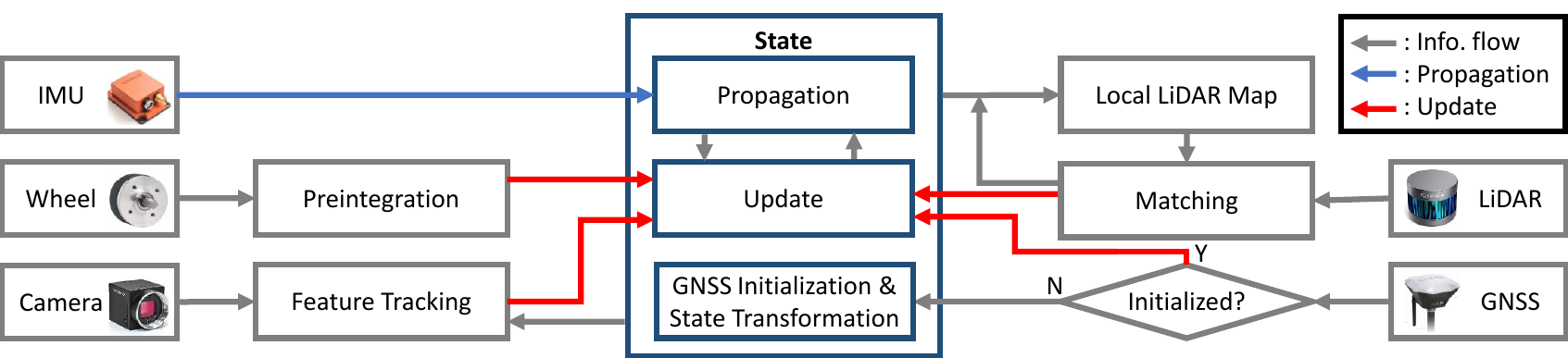}
\vspace{-5mm}
\caption{Overview of the proposed MINS showing each sensor component and their roles in state estimation.}
\label{fig:system}
\end{figure*}

Calibration methods can be classified into two methodologies: offline and online.
Offline methods estimate the calibration parameters in batch optimization fashion with the aim to provide the most accurate results at the cost of high computation and time.
There are few works that jointly calibrate 4 or more sensing modalities offline, such as a general framework \citep{rehder2016general} for the calibration of the spatiotemporal parameters, iCalib \citep{Yang2021ICRAws} for IMU, camera, LiDAR, and wheel calibration, or ATOM \citep{oliveira2022atom} for the camera, LiDAR, and RaDAR calibration.
One can find more offline calibration methods in pair-wise systems, such as Kalibr \citep{rehder2016extending}, COC \citep{Heng2013IROS}, MSG-cal \citep{owens2015msg}, or MVIS \citep{yang2023multivisualinertial}.
However, the offline calibration results may still have some errors that may result in estimator inconsistency if not properly modeled and are unable to cope with only platform reconfigurations.
To this end, instead of blindly trusting them as true, many multi-sensor systems append the calibration parameters in the state to improve robustness by modeling their uncertainties and jointly estimating them while performing navigation which is called online calibration.

Among those 4 or more sensor fusion systems introduced (Sec.~\ref{ch:4_systems}), MSF-EKF \citep{lynen2013robust}, MaRS \citep{brommer2020mars}, and the work of \citet{hausman2016self} supports the spatial calibration of onboard sensors while VINS-Fusion \citep{qin2019general, xu2022review} only supports the spatiotemporal extrinsic calibration between the camera and IMU.
GR-Fusion \citep{wang2021gr} further calibrated the spatiotemporal extrinsic of the LiDAR, intrinsic and spatial extrinsic of the wheel, and temporal extrinsic of the camera.
However, the camera intrinsic and spatial extrinsic calibration was done offline and the temporal offset of the wheel was not calibrated.
Our previous work \citep{Lee2021ICRA} was the only work that showed full spatiotemporal calibration of all onboard sensors (camera, GNSS, wheel, and LiDAR), but it was limited to a single sensor per modality besides cameras.
There exist systems that fully/partially calibrate the calibration parameters of smaller sensing modalities online, such as LiDAR-camera-IMU \citep{Zuo2019IROS, Zuo2020IROS, Lee2021ICRA, ye2021keypoint}, GNSS-camera-IMU \citep{girrbach2019towards, Lee2020ICRA, han2022tightly}, wheel-camera-IMU \citep{liu2019visual, jung2020monocular, Lee2020IROS}.

\subsection{Extension of Our Previous Publications}
While this article evolved from our prior works \citep{Lee2020ICRA,Lee2020IROS,Lee2021ICRA, Geneva2020ICRA,Eckenhoff2021TRO}, there are significant contributions differentiating this work.
To be more specific, MINS leverages the high-order polynomial ideology of MIMC-VINS to enable high-frequency fusion of asynchronous and delayed arbitrary numbers of sensors within a unified continuous-time architecture.
On top of this, we provide a thorough investigation of the error induced by the interpolation, propose the incorporation of the error model that improves system consistency, and further extend this idea to dynamic cloning to enable significant computational saving while maintaining localization accuracy.
MINS enables spatiotemporal calibration of all onboard sensors,
which differentiates MINS from OpenVINS \citep{Geneva2020ICRA} which only focused on VIO intrinsics and spatiotemporal and \citet{Lee2020ICRA} which performed GNSS fusion with linear interpolation alongside VIO.
Compared to VIWO \citep{Lee2020IROS} which presented wheel encoder angular velocity fusion with VINS, MINS is additionally capable of fusing linear velocities of the wheel encoder and angular \& linear velocities of wheel odometry frame along with planar motion constraint.
Furthermore, this paper provides a novel IMU-wheel-based system initialization technique that is robust to very challenging scenarios.
The LiDAR fusion technique of MINS has been completely redesigned as compared to \citet{Lee2021ICRA} to enable multi-LiDARs fusion along with spatiotemporal calibration.
MINS adapts the direct method \citep{xu2022fast} of alignment instead of extracting LOAM features, improving the efficiency and consistency of the update process, and adding a mapping functionality.
We additionally stress that the comprehensive accuracy, consistency, computation, and ablation studies of each part of the system in realistic simulations and in real-world experiments are non-trivial for multi-sensor fusion frameworks.

 \section{MINS State Estimation}

Building upon the backbone INS, the proposed MINS propagates the state with the IMU measurements
and updates in an efficient filtering framework with multi-modal measurements of a number of exteroceptive sensors such as cameras, LiDARs, GNSS receivers, and wheel encoders.
As shown in Fig.~\ref{fig:frames}, we use 
$\{I\}$ to denote the IMU frame,
$\{C\}$ for the camera frame, 
$\{O\}$ for the wheel odometry frame,  
$\{L\}$ for the LiDAR frame, 
$\{G\}$ for the GNSS frame, 
while $\{E\}$ for the East-North-Up (ENU) frame corresponding to the GNSS.
Fig.~\ref{fig:system} depicts the the proposed MINS architecture.

Extended from the MSCKF-based VINS~\citep{Mourikis2007ICRA,Geneva2020ICRA}, 
the state vector of the proposed MINS  includes the IMU navigation state $\mathbf{x}_{I_k}$  and a set of historical IMU poses $\x{H_k}$ (a.k.a. clones):
\begin{align}
    \mathbf{x}_k &= (\x{I_k}, ~\x{H_k}) \label{eq:state}\\
\x{I_{k}} &= (\R{E}{I_{k}}, ~\p{I_{k}}{E}, ~\vel{I_{k}}{E}, ~\mathbf{b}_{g}, ~\mathbf{b}_{a}) \label{eq:state_imu}\\
\x{H_k} &= (\R{E}{I_{k \shortminus 1}}, ~\p{I_{k \shortminus 1}}{E}, ~\cdots, ~ \R{E}{I_{k \shortminus h}}, ~\p{I_{k \shortminus h}}{E}) \label{eq:state_clones}
\end{align}
where $\R{A}{B}$ is a rotation matrix from $\{A\}$ to $\{B\}$,
$\p{A}{B}$ and $\vel{A}{B}$ are the position and linear velocity of $\{A\}$ in $\{B\}$,
$\mathbf{b}_{g}$ and $\mathbf{b}_{a}$ are the biases of the gyroscope and accelerometer.
We define $\mathbf{x} = \hat{\mathbf{x}} \boxplus \tilde{\mathbf{x}}$, 
where $\mathbf{x}$ is the true state, $\hat{\mathbf{x}}$ is its estimate,  $\tilde{\mathbf{x}}$ is the error state, and the operation $\boxplus$ which maps the error state  to its corresponding manifold \citep{Hertzberg2011}.
Note that we represent the state variables in $\{E\}$ as the ``global'' GNSS coordinate after the system is initialized using the available GNSS measurements (see Sec.~\ref{sc:gnss_initialization}), 
otherwise, the state is expressed in the local world frame $\{W\}$ 
(which is often the ``global'' frame in  local navigation systems \citep{qin2018vins, xu2022fast}).

In what follows, we present in detail each of the five sensing modalities and the proposed fusion methods to propagate and update the state with their measurements.
 \subsection{IMU} \label{ch:IMU}

High-rate IMU measurements typically include the angular velocity $\bm \omega_m$ and local linear acceleration $\mathbf{a}_m$:
\begin{align}
\bm\omega_{m_k} &=  \bm \omega_k + \mathbf{b}_g +\mathbf{n}_g \\
\mathbf{a}_{m_k} &= \mathbf{a}_k + \R{E}{I_k} {}^E\mathbf{g} + \mathbf{b}_a + \mathbf{n}_a
\end{align}
where $\bm \omega_k$ and $\mathbf{a}_k$ are the true angular velocity and linear acceleration at time $t_k$; $\mathbf{b}_g$ and $\mathbf{b}_a$ are the biases of the gyroscope and the accelerometer; ${}^E\mathbf{g} \simeq [0~0~9.81]^\top$ is the global gravity; $\mathbf{n}_g$ and $\mathbf{n}_a$ are zero mean Gaussian noises. 
These measurements are used to propagate the IMU state $\x{I}$ from $t_k$ to $t_{k+1}$ based on the following generic nonlinear kinematic model \citep{Trawny2005_Q_TR}:
\begin{align} \label{eq:propagation}
\xhat{I_{k+1}} = f(\xhat{I_{k}}, \mathbf{a}_{m_k}, \bm\omega_{m_k})
\end{align}
In contrast, the historical pose state $\x{H_k}$ is static and does not evolve over time.
Thus, we only need to linearize Eq.~\eqref{eq:propagation} 
to propagate the corresponding covariance matrix, $\mathbf{P}_{I_{k}}$, of the IMU state, as follows:
\begin{align}
\hspace{-2mm}
    \mathbf{P}_{I_{k+1}} = \bm\Phi_I(t_{k+1},t_k) \mathbf{P}_{I_{k}} \bm\Phi_I(t_{k+1},t_k)^\top + \mathbf{G}_k \mathbf{Q}_d \mathbf{G}_k^\top
\hspace{-2mm}
\end{align}
where 
$[\mathbf{n}_g^\top ~ \mathbf{n}_a^\top ~ \mathbf{n}_{\omega g}^\top ~ \mathbf{n}_{\omega a}^\top]^\top \sim \mathcal{N}(\mathbf{0}, \mathbf{Q}_d)$; $\mathbf{n}_{\omega g}$ and $\mathbf{n}_{\omega a}$ are the white Gaussian noises of gyroscope and accelerometer bias random walk model; 
$\bm{\Phi}_I(t_{k+1},t_k)$ and 
$\mathbf{G}_{k}$ are the Jacobians of $f(\cdot)$ respect to the state and noise vector, respectively,
which can be found in \citep{hesch2012observability}.
 \subsection {Camera} \label{ch:Camera}

Consider a 3D feature is detected from an arbitrary camera image at time $t_k$, whose $uv$ measurement (i.e. the corresponding pixel coordinates) on the image plane is given by (see \citet{Geneva2020ICRA}):
\begin{align}
    \mathbf{z}_{C}
&= \mathbf h_C(\mathbf x_k) + \mathbf{n}_{C} \label{eq:cam_meas_general} \\
&= \mathbf h_d(\mathbf h_\rho(\mathbf h_t(\p{F}{E}, ~\R{E}{C_k}, ~\p{C_k}{E})), ~\x{CI}) + \mathbf{n}_{C} \label{eq:cam_meas_nested}
\end{align}
where $\mathbf{n}_{C}$ is the zero mean white Gaussian noise;
$\mathbf z_{n}$ is the normalized undistorted $uv$ measurement;
$\x{CI}$ is the camera intrinsic parameters such as focal length and distortion parameters;
$\p{F}{E}$ is the feature position in global;
$(\R{E}{C_k}, ~\p{C_k}{E})$ denotes the current camera pose in the global.
In the above expression, we decompose the measurement function into multiple concatenated functions corresponding to different operations, which map the states into the raw $uv$ measurement on the image plane.
Each function (i.e. $\mathbf{h}_d,~ \mathbf{h}_\rho,~ \mathbf{h}_t$) is explained in the following:

\subsubsection[Distortion Function]{Distortion Function $\mathbf{h}_d$} \label{ch:cam_distortion}
To get the normalized coordinates of the 3D feature on the image plane $\mathbf{z}_{n} = [ x_n ~ y_n ]^\top$, we apply a distortion model which depends on the camera lens type.
To be more specific, MINS supports radial-tangential and fisheye camera models \citep{kannala2006generic} which map the normalized coordinates into the raw pixel coordinates.
As an example, we employ distortion function $\mathbf{h}_d$ with the radial-tangential model:
\begin{align}
    \begin{bmatrix} u \\ v \end{bmatrix} := ~
\mathbf{z}_{C} = \mathbf h_d(\mathbf{z}_{n}, ~\x{CI}) 
= \begin{bmatrix}  f_x x + c_x \\
    f_y y + c_y \end{bmatrix} \label{eq:cam_distortion_function}
\end{align}
where
\begin{align}
&\hspace{-2mm}x = x_n (1 + k_1 r^2 + k_2 r^4) + 2 p_1 x_n y_n + p_2(r^2 + 2 x_n^2) \hspace{-2mm}\\
&\hspace{-2mm}y = y_n (1 + k_1 r^2 + k_2 r^4) + p_1 (r^2 + 2 y_n^2) + 2 p_2 x_n y_n \hspace{-2mm}\\
&\hspace{-2mm}r^2 = x_n^2 + y_n^2    \hspace{-2mm}\\
   &\hspace{-2mm}\x{CI} = (f_x, ~ f_y, ~ c_x, ~ c_y, ~ k_1, ~ k_2, ~ p_1, ~ p_2)\hspace{-2mm}
\end{align}
See \citet{kannala2006generic} for the camera intrinsic parameter $\x{CI}$ definition.

\subsubsection [Perspective Projection Function]{Perspective Projection Function $\mathbf{h}_\rho$}
The standard pinhole camera model is used to project a 3D point $\p{F}{C_k} = \begin{bmatrix} {}^{C_k}x & {}^{C_k}y & {}^{C_k}z \end{bmatrix}^\top$ into the normalized image plane (with unit depth):
\begin{align}
    \mathbf{z}_{n} &= \mathbf h_\rho  (\p{F}{C_k}) =
\begin{bmatrix}
    {}^{C_k}x/{}^{C_k}z \\
    {}^{C_k}y/{}^{C_k}z
    \end{bmatrix}
\end{align}

\subsubsection[Euclidean Transformation]{Euclidean Transformation $\mathbf{h}_t$}
We employ the 6DOF rigid-body Euclidean transformation to transform the 3D feature position in $\{E\}$ to the current camera frame $\{C_k\}$ based on the current camera pose:
\begin{align}
    \p{F}{C_k} 
    &= \mathbf h_t (\p{F}{E},~ \R{E}{C_k},~ \p{C_k}{E})  \label{eq:cam_euclidean_trans}\\
    &= \R{E}{C_k}(\p{F}{E} - \p{C_k}{E})    
\end{align}
where we in turn represent the camera pose $(\R{E}{C_k}, \p{C_k}{E})$ using camera extrinsic $\x{CE} = (\R{I}{C}, ~ \p{I}{C})$ and IMU pose $(\R{E}{I_k}, \p{I_k}{E})$:
\begin{align}
    \R{E}{C_k} &= \R{I}{C}\R{E}{I_k}\\
    \p{C_k}{E} &= \p{I_k}{E} + \R{I_k}{E} \p{C}{I}
\end{align}

\subsubsection{MSCKF Update}
To perform the MSCKF update \citep{Mourikis2007ICRA}, 
we linearize the measurement function (see Eq.~\eqref{eq:cam_meas_general}) at current state estimate $\xhat{k}$ and $\phat{F}{E}$ to get the following residual:
\begin{align}
\tilde{\mathbf{z}}_{C} :=&~ \mathbf{z}_{C} - \mathbf{h}_C(\xhat{k}, \phat{F}{E}) \label{eq:cam_residual_ori}\\
\approx &~ \H{C} \xtilde{k} + \H{F}\ptilde{F}{E} + \mathbf{n}_{C} \label{eq:cam_residual}
\end{align}
where $\H{C_k}$ and $\H{F}$ are the Jacobian matrix of $\mathbf{h}_C(\cdot)$ in respect to $\xhat{k}$ and $\phat{F}{E}$. The detailed derivation of these Jacobians can be found in Appendix~\ref{ch:apdx_cam}.
After stacking the Jacobians and residuals for all camera measurements, 
we find the left nullspace of $\H{F}$ and project Eq.~\eqref{eq:cam_residual} onto the nullspace to obtain a new measurement residual that is independent of the feature position:
\begin{align}
\tilde{\mathbf{z}}_{C}' = \H{C}' \xtilde{k} + \mathbf{n}_{C}' \label{eq:cam_res_nullspace_prj}
\end{align}
Then we perform measurement compression (\citet{golub2013matrix}, Algorithm 5.2.4) which leads to substantial computational savings before the EKF update.
  \subsection{Wheel Encoder} \label{ch:WheelEncoder}

2D wheel encoder measurements  are commonplace on a ground vehicle that is often driven by two differential wheels (left and right) mounted on a common axis (baselink), each equipped with an encoder providing local angular rate readings \citep{siegwart2011introduction}:
\begin{align}
    \omega_{ml_k} = \omega_{l_k} + n_{\omega_l}, ~~~~ \omega_{mr_k} = \omega_{r_k} + n_{\omega_r} \label{eq:wheel_meas_raw}
\end{align}
where $\omega_{l_k}$ and $\omega_{r_k}$ are the true angular velocities of each wheel at time $t_k$, and $n_{\omega_l}$ and $n_{\omega_r}$ are the corresponding zero-mean white Gaussian noises.
These encoder readings can be combined to provide 2D  linear and angular velocities about odometer frame $\{O\}$ at the center of the baselink:
\begin{align}
    {}^{O_k}{\omega} = (\omega_{r_k}r_r - \omega_{l_k}r_l)/b ~, ~~~~ {}^{O_k}{v} = (\omega_{r_k}r_r + \omega_{l_k}r_l)/2
    \label{eq:wheel_meas_bodyframe}
\end{align}
where $\x{OI} := [ r_l  ~ r_r ~ b]^\top$ are the left and right wheel radii and the baselink length, respectively.
Note that instead of \eqref{eq:wheel_meas_raw}, different forms of wheel encoder's measurements might be used, for example, 
linear velocities of the left and right wheels (i.e. wheel radii have been taken into account), 
or linear and angular velocities directly of the baselink.

\subsubsection{Wheel Odometry Integration} \label{ch:wheel_preint}
As the wheel encoders typically provide measurements of a high rate (e.g. 100-500 Hz), it would be too expensive to perform EKF updates at their rate.
On the other hand, as the state (see Eq.~\eqref{eq:state}) has the historical poses, we naturally integrate the wheel odometry measurements (see Eq.~\eqref{eq:wheel_meas_bodyframe}) between the two latest poses and then use this integrated 2D motion measurement for the update.

Consider integrating wheel odometry measurements between two IMU times $t_{k \shortminus 1}$ and $t_{k}$.
The continuous-time 2D kinematic model for $t_\tau \in [t_{k \shortminus 1},t_{k}]$ is given by:
\begin{align}
    \hspace{-3mm}
    \begin{bmatrix}
    {}^{O_{\tau}}_{O_{k \shortminus 1}}\dot{\theta}\\
    {}^{O_{k \shortminus 1}}\dot{x}_{O_{\tau}}\\
    {}^{O_{k \shortminus 1}}\dot{y}_{O_{\tau}}
    \end{bmatrix}
    \hspace{-1mm}
    :=
    \hspace{-1mm}
    \begin{bmatrix}
    -{}^{O_\tau}{\omega}\\
    {}^{O_\tau}v\text{cos}({}^{O_{k \shortminus 1}}_{O_{\tau}}\theta)\\
    {}^{O_\tau}v\text{sin}({}^{O_{k \shortminus 1}}_{O_{\tau}}\theta)
    \end{bmatrix}
    \hspace{-1mm}
    =
    \hspace{-1mm}
    \begin{bmatrix}
    -{}^{O_\tau}{\omega}\\
    {}^{O_\tau}v\text{cos}({}^{O_{\tau}}_{O_{k \shortminus 1}}\theta)\\
    -{}^{O_\tau}v\text{sin}({}^{O_{\tau}}_{O_{k \shortminus 1}}\theta)
    \end{bmatrix} \label{eq:wheel_kinematic_model}
    \hspace{-2mm}
\end{align}
where ${}^{O_{\tau}}_{O_{k \shortminus 1}}\theta$ is the local yaw angle, ${}^{O_{k \shortminus 1}}x_{O_{\tau}}$ and ${}^{O_{k \shortminus 1}}y_{O_{\tau}}$ are the 2D position of $\{O_\tau\}$ in the starting integration frame $\{O_{k \shortminus 1}\}$.
Note that we use $-{}^{O_\tau}{\omega}$ and $-{}^{O_\tau}v\text{sin}({}^{O_{\tau}}_{O_{k \shortminus 1}}\theta)$ because we follow global-to-local orientation representation.
We then integrate the model from $t_{k \shortminus 1}$ to $t_{k}$ and obtain the 2D relative pose measurement as follows:
\begin{align} 
\mathbf{z}_{O} 
&=: \begin{bmatrix} {}^{O_{k}}_{O_{k \shortminus 1}}\theta \\ {}^{O_{k \shortminus 1}}\mathbf{d}_{O_{k}} \end{bmatrix} 
=\begin{bmatrix}
-\int^{t_{k}}_{t_{k \shortminus 1}}{}^{O_t}{\omega}dt \\
\int^{t_{k}}_{t_{k \shortminus 1}}{}^{O_t}v~\text{cos}({}^{O_{t}}_{O_{k \shortminus 1}}\theta)dt \\
-\int^{t_{k}}_{t_{k \shortminus 1}}{}^{O_t}v~\text{sin}({}^{O_{t}}_{O_{k \shortminus 1}}\theta)dt
\end{bmatrix} \label{eq:wheel_preintegration}\\
&=: \mathbf g_O (\{\omega_l, \omega_r\}_{k \shortminus 1:k}, \x{OI})
\label{eq:wheel_preint_function}
\end{align}
where $\{\omega_l, \omega_r\}_{k \shortminus 1:k}$ denote all the wheel angular velocities integrated between $t_{k \shortminus 1}$ and $t_{k}$.

\subsubsection{Measurement Update} \label{ch:wheel_update}
Note that the integrated wheel measurement, Eq.~\eqref{eq:wheel_preint_function}, provides {\em only} the 2D relative motion on the odometer's plane, while the state (see Eq.~\eqref{eq:state}) contains the 3D poses. 
The measurement can be expressed as a function of the relative pose of the odometer frame by projection:
\begin{align}
    \mathbf{z}_{O} 
    = 
    \mathbf{h}_{O} (\x{k}) 
    :=
    \begin{bmatrix}
        \mathbf{e}_3^\top \Log{ \R{E}{O_{k}} \R{O_{k \shortminus 1}}{E}}\\
        \Lambda \R{E}{O_{k \shortminus 1}} (\p{O_{k}}{E} - \p{O_{k \shortminus 1}}{E})
    \end{bmatrix} \label{eq:wheel_update_model}
\end{align}
where $\Lambda = [ \mathbf{e}_1 ~ \mathbf{e}_2 ]^\top$, $\mathbf{e}_i$ is the $i$-th standard unit basis vector, and $\Log{\cdot}$ is the $SO(3)$ matrix logarithm function \citep{chirikjian2011stochastic}.
The odometry pose $(\R{E}{O_{k}}, \p{O_{k}}{E})$ can be derived with IMU pose $(\R{E}{I_{k}}, \p{I_{k}}{E})$ and extrinsic $(\R{I}{O}, \p{I}{O})$:
\begin{align}
    \R{E}{O_{k}} &= \R{I}{O} \R{E}{I_{k}}\\
    \p{O_{k}}{E} &= \p{I_{k}}{E} + \R{I_{k}}{E} \p{O}{I}
\end{align}
At this point, we have obtained the integrated wheel odometry measurements along with their corresponding Jacobians which are readily used for the EKF update after linearization:
\begin{align}
    \tilde{\mathbf{z}}_{O}
    &:= \mathbf g_O (\{\omega_{ml}, \omega_{mr}\}_{k \shortminus 1:k}, \xhat{OI}) - \mathbf{h}_O(\xhat{k}) \\
    &\approx  \H{O}\xtilde{k} + \mathbf{n}_{O}
\end{align}
where $\H{O}$ is the Jacobian matrix of ($\mathbf{h}_O(\cdot) - \mathbf{g}_O(\cdot)$) in respect to the state $\xhat{k}$ and $\mathbf{n}_{O}$ is the zero-mean white Gaussian noise.
Detailed derivations of $\H{O}$ and $\mathbf{n}_{O}$ can be found in Appendix~\ref{ch:apdx_wheel}.

 \subsection{LiDAR} \label{ch:LiDAR}

LiDAR provides 3D point clouds of the surroundings. 
Given a new measurement point cloud in LiDAR frame $\{L\}$ at time $t_k$, for each point $\p{F}{L_k}$ we transform it to a local map frame $\{M\}$ and find a number of neighboring points,  $\p{n_i}{M} = [x_{n_i} ~ y_{n_i} ~ z_{n_i}]^\top, i \in \{1, \cdots, m\}$.
We compute the plane $\cp{}{M} = [a ~ b ~ c]^\top$ where the neighboring points are residing as follows:
\begin{align}
    \underbrace{
    \begin{bmatrix}
        x_{n_1} & y_{n_1} & z_{n_1}\\
        \vdots & \vdots & \vdots\\
        x_{n_m} & y_{n_m} & z_{n_m}\\ 
    \end{bmatrix}
    }_{\mathbf{A}_{m\times 3}}
\begin{bmatrix}
        a \\ b \\ c
    \end{bmatrix}
    =
    \underbrace{
    \begin{bmatrix}
        1 \\
        \vdots\\
        1
    \end{bmatrix}
    }_{\mathbf{b}_{m\times 1}}
\end{align}
Its linear least-square solution is given by $\cp{}{M} = (\mathbf{A}^\top \mathbf{A})^{-1} \mathbf{A}^\top \mathbf{b}$.
After finding the plane $\cp{}{M}$, we formulate the following point-on-plane measurement for all the planar points including $\p{F}{L_k}$ and $\p{n_i}{M}$:
\begin{align}
    \mathbf{z}_L := 
\begin{bmatrix}
        0\\
        \vdots\\
        0\\
        0
    \end{bmatrix}
= \mathbf{h}_L(\x{k}) = 
    \begin{bmatrix}
        \cp{}{M}^\top \p{n_1}{M} - 1\\
        \vdots\\
        \cp{}{M}^\top \p{n_m}{M} - 1\\
        \cp{}{M}^\top \p{F}{M} - 1\\
    \end{bmatrix} \label{eq:lidar_measurement_model}
\end{align}
where $\p{F}{M}$ can be computed by transforming $\p{F}{L_k}$ to the map frame using the IMU-LiDAR extrinsic calibration $\x{L} = (\R{I}{L}, ~ \p{I}{L})$, the IMU pose $(\R{E}{I_k}, \p{I_k}{E})$, and the map pose $(\R{E}{M}, \p{M}{E})$:
\begin{align}
    \hspace{-2mm}\p{F}{M} = \R{E}{M} (\p{I_k}{E} + \R{I_k}{E} (\p{L}{I} + \R{L}{I} \p{F}{L_k})  - \p{M}{E}) \label{eq:lidar_tr}
\end{align}

To perform EKF update, we linearize Eq.~\eqref{eq:lidar_measurement_model} and have the following residual:
\begin{align}
    \tilde{\mathbf{z}}_L &= 
\begin{bmatrix}
        - \cphat{}{M}^\top \phat{n_1}{M} + 1\\
        \vdots\\
        - \cphat{}{M}^\top \phat{n_m}{M} + 1\\
        - \cphat{}{M}^\top \phat{F}{M} + 1
    \end{bmatrix} \label{eq:lidar_linsys_mat}\\
&\approx
    \H{L}\xtilde{k} + \H{\boldsymbol{\Pi}}{}^M\tilde{\boldsymbol{\Pi}} + \mathbf{n}_{L} \label{eq:lidar_linsys}
\end{align}
where $\H{L}$ and $\H{\boldsymbol{\Pi}}$ are the Jacobian matrix of $\mathbf{h}_L(\cdot)$ in respect to the state $\x{k}$ and the plane $\cp{}{M}$; 
$\mathbf{n}_{L}$ is zero-mean Gaussian noise.
The detailed derivations of the Jacobians and the noise can be found in Appendix~\ref{ch:apdx_lidar}.

In analogy to the visual feature marginalization in the MSCKF, 
we now perform null space projection of Eq.~\eqref{eq:lidar_linsys} onto the left nullspace of $\H{\boldsymbol{\Pi}}$ to reduce the computation and bound the size of the state.
Note that it is possible that not all the points $\p{n_i}{M}$ are actually on the plane. 
We therefore only pass the plane to the next step after checking the condition number of the matrix $\mathbf{A}$ and the point-to-plane distance of each points $\p{n_i}{M}$ on  $\cp{}{M}$.
Also note that this process is actually equivalent to initialize $\cp{}{M}$ in the state with infinite covariance, update with Eq.~\eqref{eq:lidar_linsys}, and marginalize \citep{Yang2017IROS}, thus allowing us to properly handle the uncertainty of the plane without explicitly computing it.
Then we have the following measurement model dependent only on the state ready for the EKF update:
\begin{align}
    \tilde{\mathbf{z}}_L' = \H{L}'\xtilde{k} + \mathbf{n}'_{L} \label{eq:lidar_res_nullspace_prj}
\end{align}
Note that we also perform measurement compression in the same way as in MSCKF-based VINS for further efficiency.

\subsubsection{Local Mapping} \label{ch:lidar_mapping}

As having a global map has major drawbacks as the map grows: 
(i) the computational complexity of tracking the uncertainty of the map points, 
(ii) the computational complexity of finding the matches between the map and a new point cloud, 
and (iii) the size of the memory to store all the map points,
many existing approaches keep the map points sparse to slow down the map growth to mitigate these issues.
However, the sparse map may break the measurement model assumptions.
For example, the planes extracted from a sparse map of a forest would not accurately represent the plane where the measurement point is on.
We hence chose to have a dense map to ensure accurate matching while keeping only a local map to address the aforementioned issues.

Specifically, we employ the ikd-tree \citep{cai2021ikd} to manage the map points and keep them for a fixed temporal window.
We assume the map frame $\{M\}$ as the LiDAR pose anchored at one of our historical IMU poses $\{I_i\}$ (see Eq.~\eqref{eq:state_clones}).
and its pose in the global frame, $(\R{E}{M}, \p{M}{E})$, can be represented using the extrinsic calibration $\x{L}$:
\begin{align}
    \R{E}{M} =& \R{I}{L} \R{E}{I_i}\\
    \p{M}{E} =& \p{I_i}{E} + \R{I_i}{E} \p{L}{I}
\end{align}
As the map frame can be represented with our historical IMU pose and the LiDAR extrinsics, there is no need to add the map pose in the state, which saves computation.
Note that the map frame should change over time as the MSCKF  marginalizes the old IMU poses.
In this case, we change  the map frame anchored to the latest IMU frame (e.g. $\{I_j\}$) using the following relative pose:
\begin{align}    
\hspace{-2mm}
    \R{M_{i}}{M_{j}} &= \R{I}{L}\R{E}{I_{j}}\R{I_{i}}{E}\R{L}{I}\\
\hspace{-2mm}
    \p{M_{i}}{M_{j}} &= \R{I}{L}\R{E}{I_j} (\p{I_j}{E} - \p{I_i}{E} + (\R{I_j}{E}- \R{I_i}{E})\p{L}{I})    
\hspace{-2mm}
\end{align}
After using a new point cloud to update the state,  
we register the points $\p{F}{L_k}$ in the map by transforming them to the map frame (see Eq.~\eqref{eq:lidar_tr}).
By doing so, the MINS estimator can maintain dense local maps, enabling consistent point-cloud matching and reducing outliers.

 \subsection{GNSS} \label{ch:GNSS}
A GNSS receiver such as GPS provides latitude, longitude, and altitude readings in a geodetic coordinate frame, 
which are typically converted to a local ENU or NED frame $\{E\}$ for outdoor navigation, 
e.g. by simply setting the first measurement location as the datum or using a base station.
A GNSS measurement $\mathbf{z}_{G}$ at time $t_k$ is the receiver's global position $\p{G_k}{E}$ 
-- we here do not consider deeply-coupled fusion where pseudo-range and/or carrier phase measurements are used --
which can be modeled with the IMU pose $(\R{E}{I_k}, \p{I_k}{E})$ and the extrinsic calibration $\p{G}{I}$:
\begin{align}
    \mathbf{z}_{G} = \mathbf{h}_{G}(\x{k}) =&~ \p{G_k}{E} + \mathbf{n}_{G} \\
    =&~ \p{I_k}{E} + \R{I_k}{E}\p{G}{I} + \mathbf{n}_{G} \label{eq:gps_measurement}
\end{align}
where $\mathbf{n}_{G}$ is a white Gaussian noise.
Now we linearize this measurement to perform the EKF update:
\begin{align}
    \tilde{\mathbf{z}}_{G} := \mathbf{z}_{G} - \mathbf{h}_{G}(\xhat{k}) \approx \H{G} \xtilde{k} + \mathbf{n}_{G} \label{eq:gps_residual}
\end{align}
where $\H{G}$ is the Jacobian matrix of $\mathbf{h}_{G}(\cdot)$ in respect to $\mathbf x_k$ (see Appendix~\ref{ch:apdx_gps}).

Note that the proposed MINS also supports general global pose measurements,
which may be provided by sensors such as other estimators running independently or loop closure constraints.
Assuming that the measurement carries information of a sensor X pose $(\R{E}{X}, \p{X}{E})$, 
we model the measurement using the IMU pose $(\R{E}{I_k}, \p{I_k}{E})$ and the corresponding extrinsic calibration $(\R{I}{X}, \p{I}{X})$:
\begin{align}
    \hspace{-2mm}
    \mathbf{z}_{X}
    \hspace{-0.5mm}
    := 
    \hspace{-0.5mm}
    \begin{bmatrix}
        \ang{E}{X_{k}} \\
        \p{X_{k}}{E}
    \end{bmatrix}
    \hspace{-0.5mm}
    =
    \hspace{-0.5mm}
    \mathbf{h}_X(\x{k})
    \hspace{-0.5mm}
    =
    \hspace{-0.5mm}
    \begin{bmatrix}
        \Log{\R{I}{X}\R{E}{I_{k}}}\\
        \p{I_{k}}{E} + \R{I_{k}}{E}\p{X}{I}
    \end{bmatrix} 
    +
    \mathbf{n}_X 
    \hspace{-2mm} \label{eq:X_pose_measurement}
\end{align}
where $\mathbf{n}_X$ is the zero-mean Gaussian noise.
We again linearize this global-pose measurement for the EKF update:
\begin{align}
    \tilde{\mathbf{z}}_{X} := \mathbf{z}_{X} - \mathbf{h}_X(\xhat{k}) \approx \H{X}\xtilde{k} + \mathbf{n}_X 
\end{align}
where $\H{X}$ is the Jacobian matrix (see Appendix~\ref{ch:apdx_sync_jacobian}). \section{Adaptive On-Manifold Interpolation} \label{ch:interpolation}

\begin{figure}[t]  
    \centering
    \includegraphics[trim={0mm 0mm 0mm 0mm},clip, width=1\columnwidth]{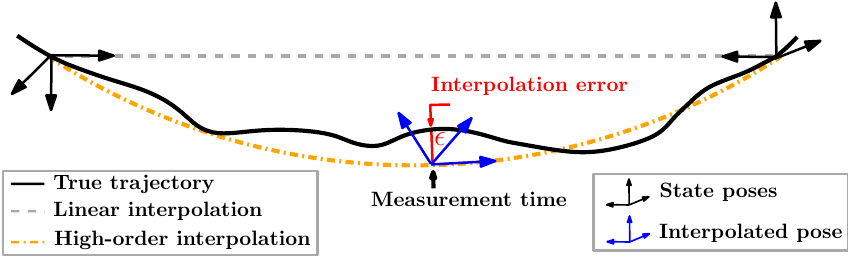}
\caption{Different interpolation methods.}
    \label{fig:interpolation}
\end{figure}

Different sensors sample at different rates, such as a camera at 30 Hz and a LiDAR at 10 Hz, 
and their measurements are in general asynchronous if without (hardware) synchronization.
This would incur difficulty to establish their measurement models as the corresponding states (of the IMU poses) are not available at the exact measurement times (see Eq.~\eqref{eq:state}).
A naive solution  is to add IMU poses into the state at every sampling time of each sensor,
which would increase the size of the state and incur significant computational cost.
\begin{figure*}[t] 
    \includegraphics[trim={0mm 0mm 0mm 0mm},clip,width=0.49\textwidth]{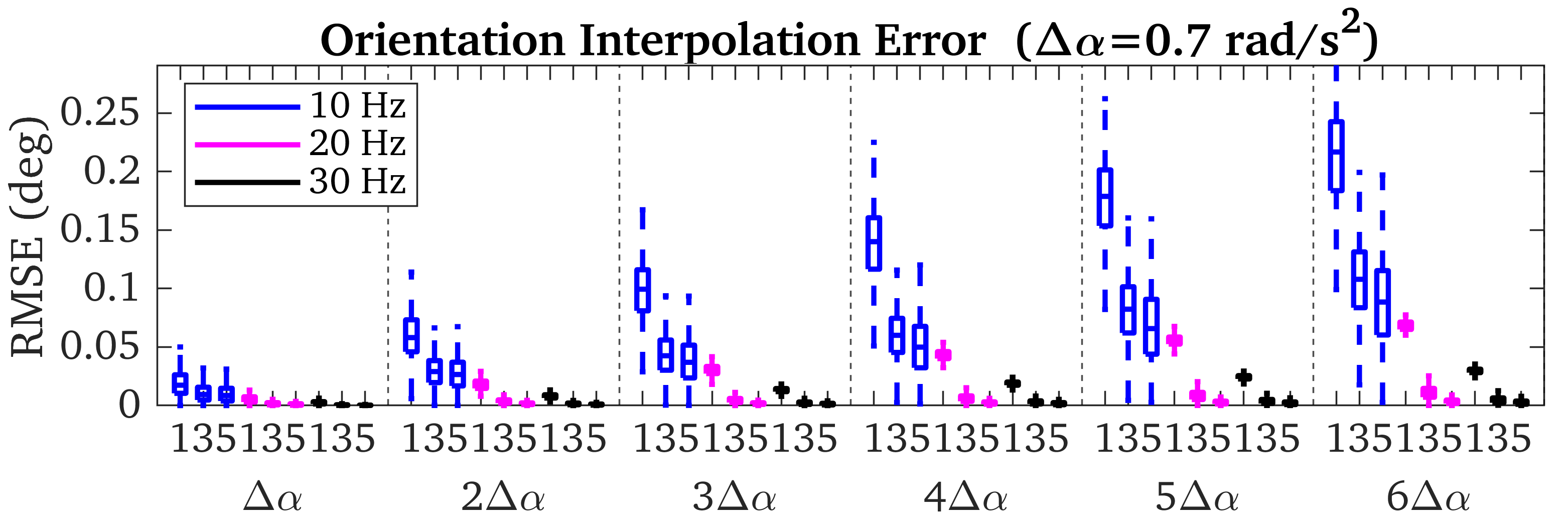}
    \includegraphics[trim={0mm 0mm 0mm 0mm},clip,width=0.49\textwidth]{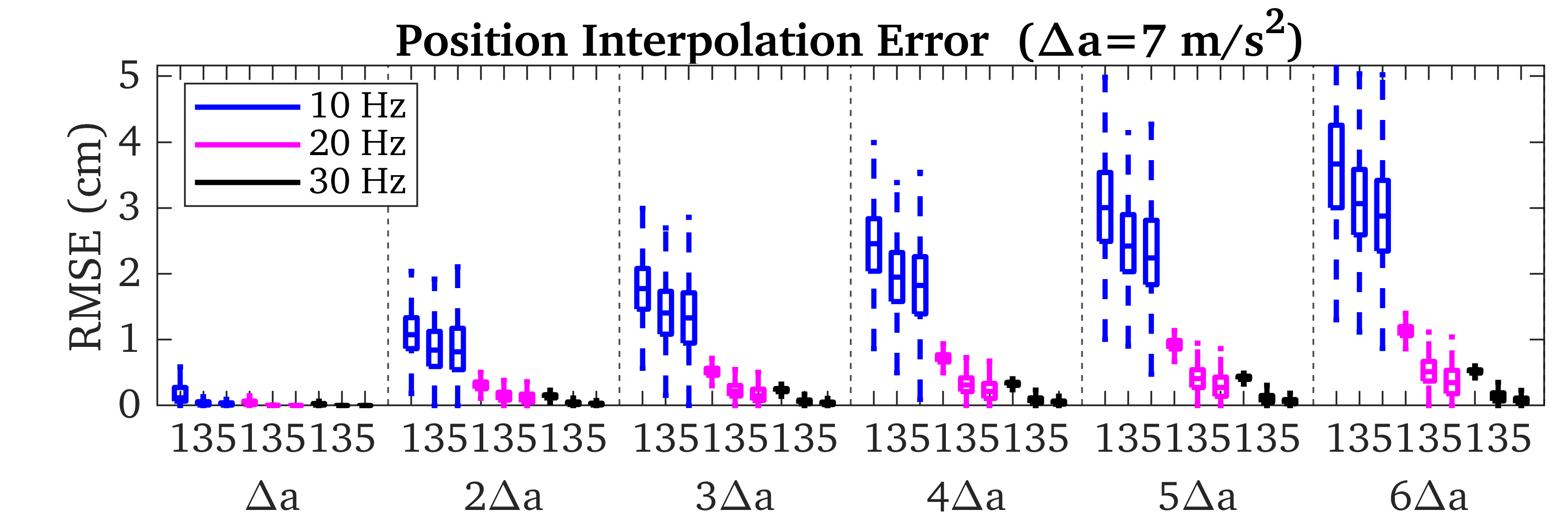}
    \caption{Interpolated pose errors depend on angular/linear accelerations (binned with $0.7~ rad/s^2$ and $7~ m/s^2$ resolution, respectively), interpolation order (1, 3, and 5), and temporal distance (10, 20, and 30 cloning Hz).} \label{fig:intr_error_pose}
\end{figure*}
\begin{figure*}[t] 
    \includegraphics[trim={5mm 5mm 0mm 0mm},clip,width=0.33\textwidth]{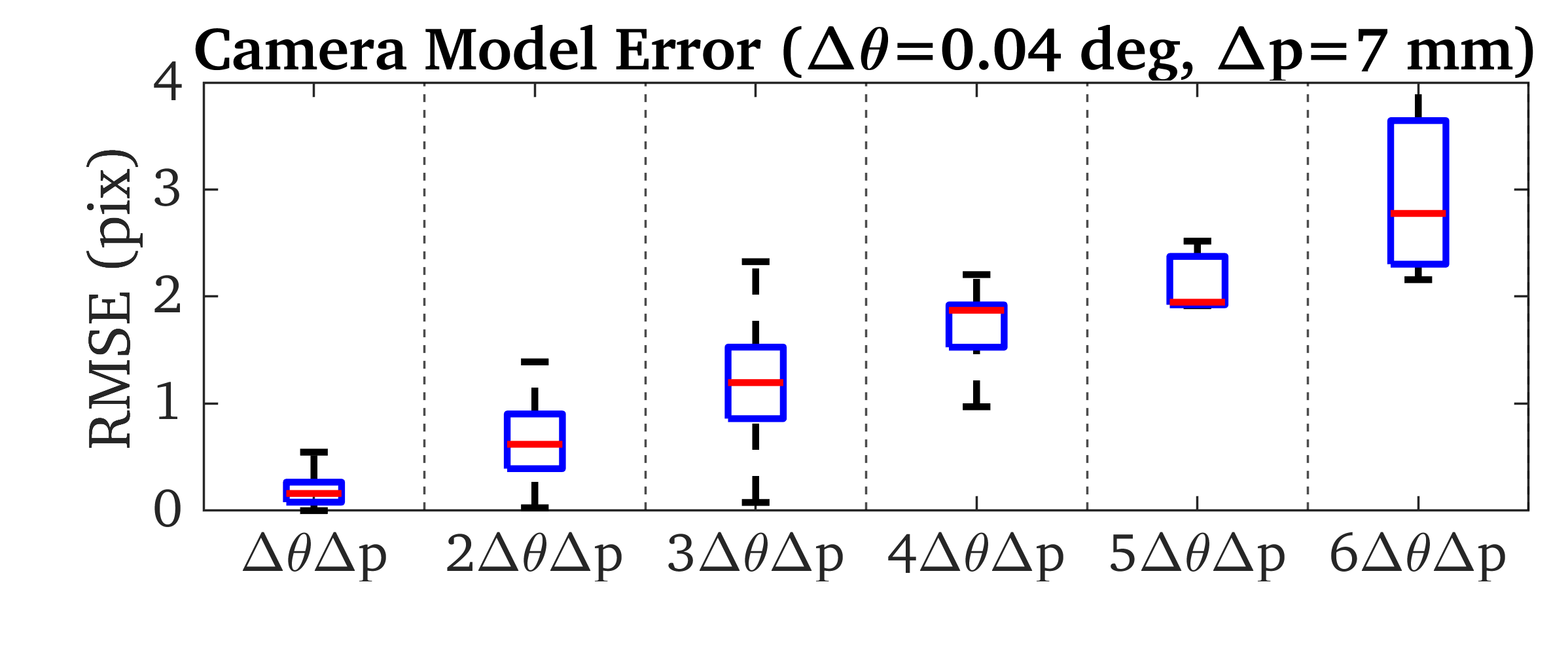}
    \includegraphics[trim={5mm 5mm 0mm 0mm},clip,width=0.33\textwidth]{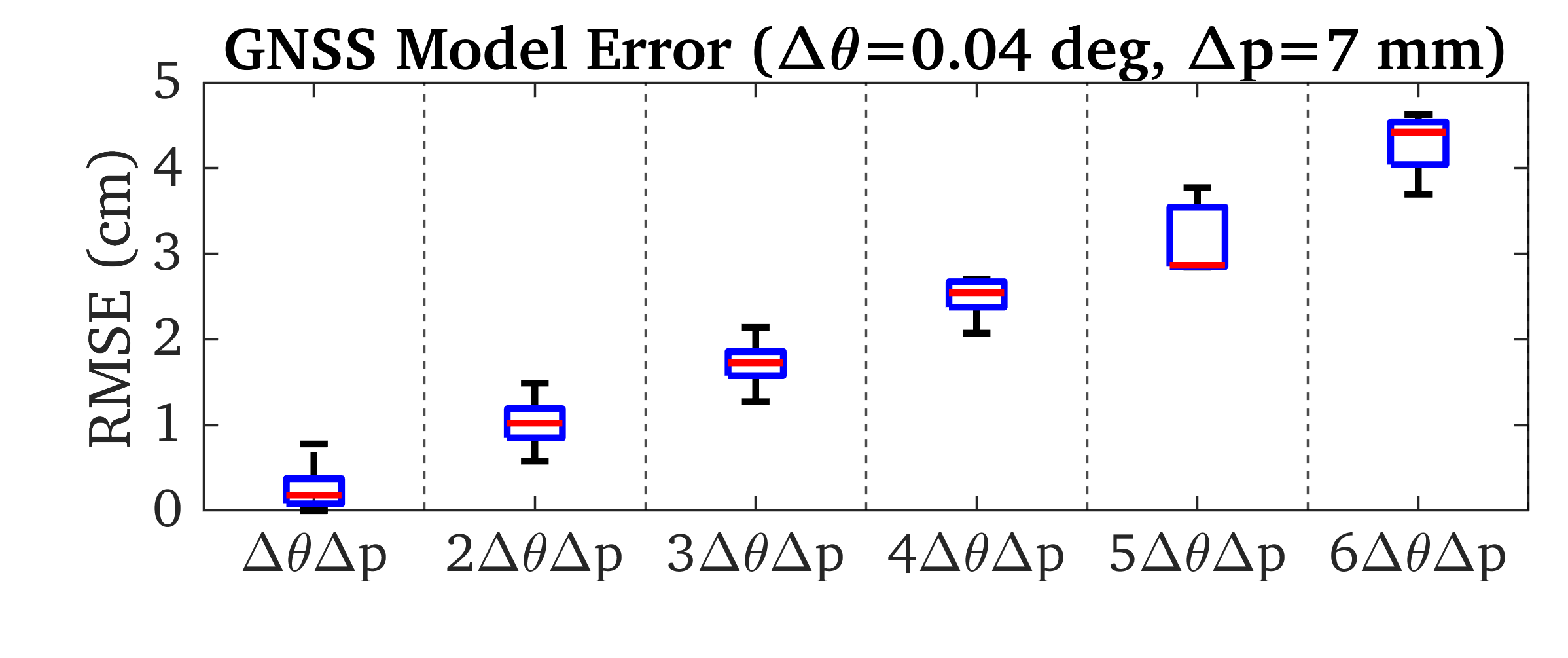}
    \includegraphics[trim={5mm 5mm 0mm 0mm},clip,width=0.33\textwidth,height=2.35cm]{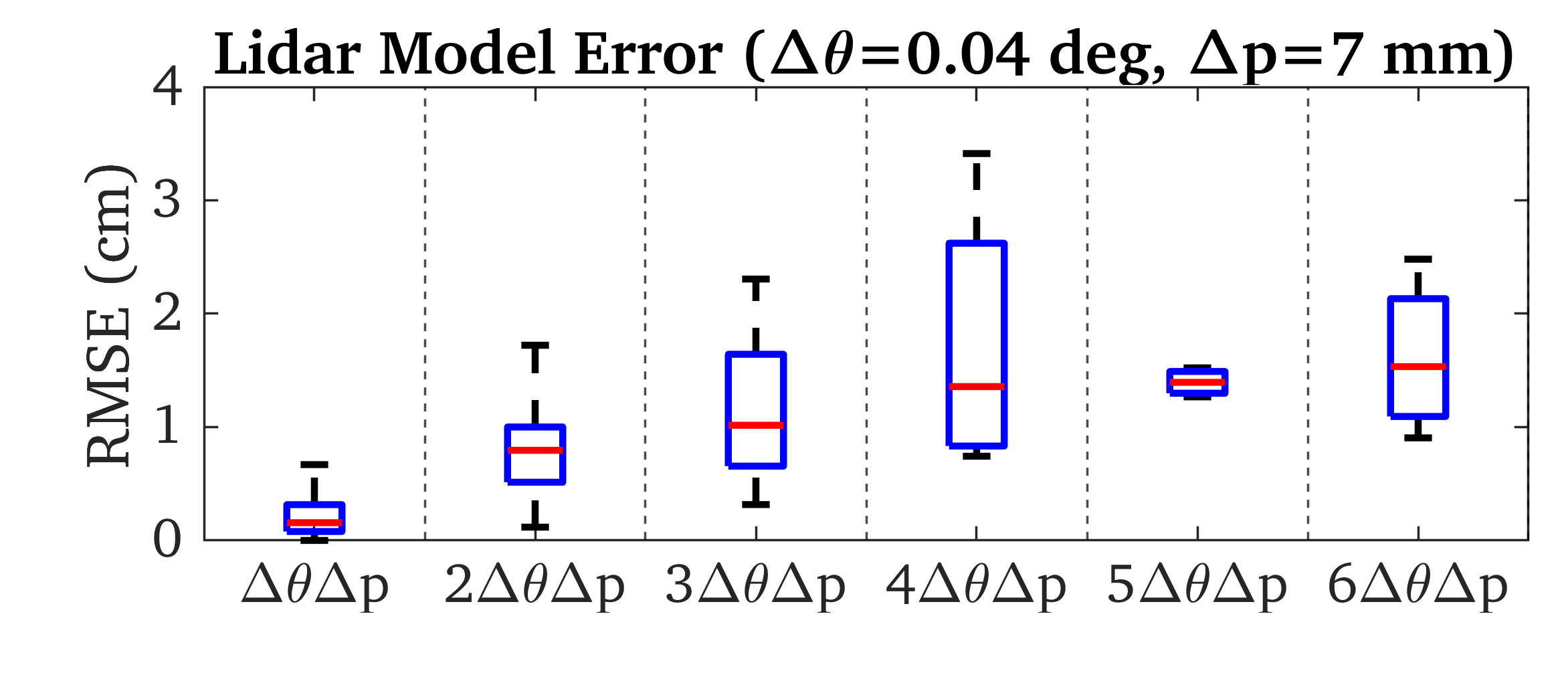}
    \caption{Sensor measurement model errors depend on interpolated pose errors (binned with $0.04~deg$ and $7~mm$ resolutions).} \label{fig:intr_error_sensors}
\end{figure*}

In contrast, we propose to represent the 6DOF poses with on-manifold interpolation and do not require any additional state augmentation.  
In our prior work \citep{Lee2020IROS, Lee2021ICRA}, we employed the 1st-order linear interpolation to handle asynchronous and delayed measurements:
\begin{align}
    \R{E}{I_k} &= \textrm{Exp}(\lambda \textrm{Log}(\R{E}{I_{b}}\R{I_{a}}{E}))\R{E}{I_{a}}\\
\p{I_k}{E} &= (1 - \lambda)\p{I_{a}}{E} + \lambda\p{I_{b}}{E}\\
\lambda &= (t_k + \t{I}{X} - t_{a})/(t_{b} - t_{a}) \label{eq:interpolation}
\end{align}
where the pose $(\R{E}{I_k}, \p{I_k}{E})$ at measurement time $t_k$ was represented by two bounding IMU poses ($ t_a \leq (t_k+\t{I}{X}) \leq t_b$), $\t{I}{X}$ is the sensor time offset, and $\Exp{\cdot}$ is the $SO(3)$ matrix exponential function \citep{chirikjian2011stochastic}.
The Jacobians of the interpolated IMU pose with respect to each bounding pose can be found in Appendix~\ref{ch:apdx_interpolation}.

However, 
this linear interpolation relies on the assumption of 
constant linear and angular velocity,
which may not hold in the case of highly-dynamic motion or for some sensitive sensors, 
causing inconsistency to the estimator due to interpolation errors, as illustrated in Fig.~\ref{fig:interpolation}.
As such, the proposed MINS adopts a high-order interpolation method to capture more complex motion profiles in practice.

Specifically, we represent the IMU pose $(\R{E}{I_k}, \p{I_k}{E})$ at time $t_k$ with a polynomial of degree $n$ \citep{Eckenhoff2021TRO}:
\begin{align}
    \R{E}{I_k} &= \Exp{\sum_{i = 1}^{n} \mathbf{a}_i \Delta t_k^i} \R{E}{I_0} \label{eq:high_intr_ori}\\
    \p{I_k}{E} &= \p{I_0}{E} + \sum_{i = 1}^{n} \mathbf{b}_i\Delta t_k^i \label{eq:high_intr_pos}
\end{align}
where $\Delta t_k = t_k + \t{I}{X} - t_0$,
and $\mathbf{a}_i$ and $\mathbf{b}_i$ are the polynomial coefficients.
To compute the coefficients, we fit the polynomial at a set of $n+1$ historical IMU poses ($\{I_{j}\}, j \in \{0, 1, \cdots, n\}$) that are closest to the measurement time $t_k$. Specifically, we compute the coefficients as follows:
\begin{align}
    \R{E}{I_j} &= \Exp{\sum_{i = 1}^{n} \mathbf{a}_i \delta t_j^i} \R{E}{I_0}, ~~ \forall j \in \{0, 1, \cdots, n\} \hspace{-1mm}\\
    \hspace{-1mm}\Rightarrow \begin{bmatrix}
        \mathbf{a}_1 \\
        \vdots \\
        \mathbf{a}_n
    \end{bmatrix}
    &=
    \begin{bmatrix}
        \delta t_1 & \hdots & \delta t_1^n\\
        \vdots  & \ddots & \vdots\\
        \delta t_n^n & \hdots & \delta t_n^n
    \end{bmatrix}^{\shortminus 1}
\begin{bmatrix}
        \Log{\R{E}{I_1}\R{I_0}{E}}\\
        \vdots\\
        \Log{\R{E}{I_n}\R{I_0}{E}}
    \end{bmatrix} \\
    \p{I_j}{E} &= \p{I_0}{E} + \sum_{i = 1}^{n} \mathbf{b}_i\delta t_j^i, ~~ \forall j \in \{0, 1, \cdots, n\} \hspace{-1mm} \\
\hspace{-1mm}\Rightarrow \begin{bmatrix}
        \mathbf{b}_1 \\
        \vdots \\
        \mathbf{b}_n
    \end{bmatrix}
    &=
    \begin{bmatrix}
        \delta t_1 & \hdots & \delta t_1^n\\
        \vdots  & \ddots & \vdots\\
        \delta t_n^n & \hdots & \delta t_n^n
    \end{bmatrix}^{\shortminus 1}
\begin{bmatrix}
        \p{I_1}{E} - \p{I_0}{E}\\
        \vdots\\
        \p{I_n}{E} - \p{I_0}{E}
    \end{bmatrix}
\end{align}
where $\delta t_j = t_j - t_0$,
The above equations show the relation between the interpolated pose and the $n + 1$ poses in the state. 
Note that the interpolated pose is additionally a function of the unknown time offset, and thus we can further calibrate them in case needed. 
The Jacobians of the interpolated pose in respect to the $n + 1$ poses and the time offset can be found in Appendix \ref{ch:apdx_interpolation_high}.
\subsection{Numerical Analysis}
The order of the interpolation polynomial is directly related to the accuracy and complexity of the interpolated pose \citep{cioffi2022continuous}.
It is hard to analytically determine an optimal order of interpolation, as there are many factors affecting the interpolated pose, such as the interpolation distance or motion of the sensor platform (robot).
To understand these factors, we numerically examine the error characteristics of the interpolated pose using realistic simulated motions
and identify four key factors of interpolation errors including
angular acceleration, linear acceleration, interpolation order, and temporal distance of interpolation.
Fig.~\ref{fig:intr_error_pose} depicts the box plots of the interpolation errors in different simulation setups.
In particular, we found that 
the temporal distance of interpolation  impacts the most on interpolation errors. 
The higher interpolation order has smaller interpolation errors, especially large accuracy gain between the first and  third orders while  having diminishing returns with higher orders.
Interestingly, the interpolation error grows {linearly} with the increase of acceleration.

To understand the significance of the interpolation errors on different sensor measurement models, 
we  use the interpolated poses, instead of the IMU poses, in the camera~\eqref{eq:cam_meas_nested}, GNSS~\eqref{eq:gps_measurement}, and LiDAR~\eqref{eq:lidar_measurement_model} measurement models in the numerical studies where the same sensor parameters as in the ensuing simulations are used as shown in Table \ref{tab:sim_setup}.
Note that the wheel measurement model is not considered as  interpolation is not needed for the wheel-integrated measurements (see Section~\ref{ch:wheel_preint}).
Fig.~\ref{fig:intr_error_sensors} shows the errors of the sensor measurement models 
(see Eq.~\eqref{eq:cam_residual_ori}, Eq.~\eqref{eq:lidar_linsys_mat}, and Eq.~\eqref{eq:gps_residual})
with different levels of interpolation errors, which are binned with $0.04~ deg$ and $7~ mm$ resolution, respectively.
It is clear from these results that  the camera is sensitive to the interpolation error, as a small error can easily cause more than 1 pixel error which is generally the noise level of camera measurement (e.g. second column of the left figure of Fig.~\ref{fig:intr_error_sensors}).
Similarly, the LiDAR error easily exceeds 1 cm which is generally the noise level of LiDAR measurement, showing its sensitivity to the interpolation error.
The interpolation errors do not impact much on the GNSS measurement model as single-point GPS measurement noise  is  around 1 m,
which has orders of magnitude larger  while it may become sensitive when using RTK-GPS with lower measurement noise.
Sensor calibration can also affect these errors,
although reasonable values are used in Table \ref{tab:sim_setup}.

As evident, small interpolation errors can cause large errors for the measurement models and  in turn may result in estimation failure.
Based on the above extensive numerical studies,
the proposed MINS by default chooses to employ the 3rd-order polynomials and 20 Hz cloning to suppress the interpolation error and balance the efficiency.

\subsection{Incorporating Interpolation Error}
While a higher-order setting (e.g. 3rd-order polynomials and 20 Hz cloning) may approximate the IMU pose more accurately, there always exists the interpolation error (see $\boldsymbol{\epsilon}$ in Fig.~\ref{fig:interpolation}) that can hurt estimation performance.
To address this issue,
we explicitly incorporate the interpolation error into the interpolated pose in the measurement model:
\begin{align}
    \mathbf{z}_k 
    &= \mathbf{h}(\R{G}{I_k}, \p{I_k}{G}) + \mathbf{n}_k \\
    &= \mathbf{h}(\mathbf{g} (\mathbf{x}) + \boldsymbol{\epsilon}) + \mathbf{n}_k
\end{align}
where $\mathbf{z}_k$ is the sensor measurement at $t_k$, $\mathbf{h}(\cdot)$ is the corresponding measurement model with respect to the IMU pose at $t_k$, $\mathbf{g}(\cdot)$ is the interpolation function, $\boldsymbol{\epsilon} \sim \mathcal{N}(\mathbf{0, \mathbf{R}_{\boldsymbol{\epsilon}}})$ is the interpolation error modeled as zero-mean Gaussian, and $\mathbf{n}_k$ is the measurement noise.
The linearized model in turn can be shown as:
\begin{align}
    \mathbf{r}_k := \mathbf{z}_k - \mathbf{h}(\mathbf{g} (\xhat{})) = \J{\mathbf{h}}{\mathbf{g}}\J{\mathbf{g}}{\mathbf{x}}\xtilde{} + \J{\mathbf{h}}{\mathbf{g}}\boldsymbol{\epsilon} + \mathbf{n}_k \label{eq:interpolation_erro_model}
\end{align}
where the additional error term $\J{\mathbf{h}}{\mathbf{g}}\boldsymbol{\epsilon}$  compensates
for the interpolation error.
To model the variance of the interpolation error, we leverage the preceding numerical findings.
The orientation/position interpolation error is shown to have a linear relationship with angular/linear accelerations, 
while their slope coefficients depend on the cloning frequency and the interpolation order (see Fig.~\ref{fig:intr_error_pose}).
We thus model the variance of the error as follows:
\begin{align}
    \mathbf{R}_{\boldsymbol{\epsilon}}
    =
    diag((\alpha \times s_o (\bar c, \bar o))^2 \I{3}, ~(a \times s_p (\bar c, \bar o))^2 \I{3}) \label{eq:interpolation_error_var}
\end{align}
where $\alpha$ and $a$ are the size of angular and linear accelerations, $s_o(\cdot)$ and $s_p(\cdot)$ are the slope coefficients of orientation and position given cloning frequency $\bar c$ and interpolation order $\bar o$.
In practice, the proposed MINS has lookup tables ranging in cloning frequency between 4-30 Hz and interpolation order between 1-9 for both orientation and position to achieve the slopes directly.

\subsection{Dynamic Cloning}
As the interpolation error is shown to be proportional to the accelerations, 
we propose dynamic cloning to adaptively change the cloning frequency based on the robot's motion 
in order to achieve computational efficiency and maintain a comparable level of accuracy to the high fixed-rate cloning strategy.
Specifically, we decrease the cloning frequency while the robot is in slow motion which will also reduce the total size of the state, and thus reduce the computation of the EKF update.
On the other hand, when the robot undergoes dynamic motions, the low cloning frequency would cause large interpolation errors, thus we increase the cloning frequency to reduce them.
To this end, we set thresholds for the variance model of the interpolation errors (see Eq.~\eqref{eq:interpolation_error_var}) and find the lowest cloning frequency that passes the threshold:
\begin{gather}
    \bar c = \min \bar C \\
    s.t. ~~ \alpha \times s_o (\bar c, \bar o) < \gamma_o, ~~a \times s_p (\bar c, \bar o) < \gamma_p \nonumber
\end{gather}
where $\bar C$ is the set of cloning frequencies, and $\gamma_o$ and $\gamma_p$ are the thresholds for orientation and position part, respectively.
Note that the proposed MINS does not create a clone if there is no measurement between the last clone and the current desired clone time, i.e.  clone up to 20 Hz if the maximum sensing rate is 20 Hz to avoid having redundant clones.

\subsection{Online Calibration}  \label{ch:online_calibration}
One of the prerequisites to successfully run a multisensor fusion system is to have accurate sensor calibration as its failure will directly lead to low estimation performance or even divergence.
Even if one precisely calibrates the sensor parameters offline, their values may change over time for environmental or mechanical reasons.
Therefore, online calibration is necessary to achieve robustness to the poor calibration and  improve localization performance.

To that end,  we include the spatial extrinsics (see Eq.~\eqref{eq:cam_euclidean_trans}, \eqref{eq:wheel_update_model}, \eqref{eq:lidar_tr}, \eqref{eq:gps_measurement}) and intrinsics (see Eq.~\eqref{eq:cam_distortion_function}, \eqref{eq:wheel_preintegration}) of each sensor when modeling its measurements. 
Also, the temporal extrinsics (time offset, see Eq.~\eqref{eq:interpolation}) is naturally modeled while handling the asynchronicity of sensor measurement.
To perform online calibration, we augment our state (see Eq.~\eqref{eq:state}) with these sensor calibration parameters and jointly estimate them.
To be more specific, for the system that carries $c$-number of cameras, $g$-number of GNSS sensors, $l$-number of LiDAR sensors, and a pair of wheel encoders, we add sensor parameter state $\x{S}$\footnote{We reused some of the notations introduced before to describe an arbitrary number set of the sensor calibration parameters.}:
\begin{align}
    \mathbf{x}_k &= (\x{I_k}, ~\x{H_k}, ~\x{S}) \label{eq:state_all}\\
\x{S} &= (\x{C},~ \x{G},~ \x{O},~ \x{L})\\
\x{C} &= (\x{C_1},~ \cdots \x{C_c})\\
\x{C_\tau} &= (\x{CE_\tau},~ \x{CI_\tau}) ~ \forall \tau \in \{1, \cdots, c\}\\
\x{CE_\tau} &= (\R{I}{C_\tau},~ \p{I}{C_\tau},~ \t{I}{C_\tau}) \\
\x{CI_\tau} &= ({}^\tau f_x, ~{}^\tau f_y, ~{}^\tau c_x, ~{}^\tau c_y, ~{}^\tau k_1, ~{}^\tau k_2, ~{}^\tau p_1, ~ {}^\tau p_2)\\
\x{G} &= (\x{G_1},~ \cdots \x{G_g})\\
\x{G_\tau} &= (\p{G_\tau}{I},~ \t{I}{G_\tau}) ~ \forall \tau \in \{1, \cdots, g\}\\
\x{L} &= (\x{L_1},~ \cdots,~ \x{L_l})\\
\x{L_\tau} &= (\R{I}{L_\tau},~ \p{I}{L_\tau},~ \t{I}{L_\tau}) ~ \forall \tau \in \{1, \cdots, l\}\\
\x{O} &= (\x{OE},~ \x{OI})\\
\x{OE} &= (\R{I}{O},~ \p{I}{O},~ \t{I}{O})\\
\x{OI} &= (r_l,~ r_r,~ b)
\end{align}
where $\x{CE}$ and $\x{CI}$ are the camera extrinsics and intrinsics, $\x{G}$ is the GNSS extrinsics, $\x{L}$ is the LiDAR extrinsics, and $\x{OE}$ and $\x{OI}$ are the wheel extrinsics and intrinsics.
Note that not all of these variables have to be included in the state, only those inaccurate ones are included and others can be treated as true in order to save computations if possible.
Detailed derivations about the measurement Jacobians for the calibration parameters can be found in Appendix~\ref{ch:apdx_cam_intrinsic}, Appendix~\ref{ch:apdx_wheel_intrinsic}, Appendix~\ref{ch:apdx_sync_jacobian}, Appendix~\ref{ch:apdx_interpolation}, and  Appendix~\ref{ch:apdx_interpolation_high}. \section{System Initialization} \label{ch:initialization}

Initialization aims to compute state variables directly from available  measurements to start estimation.
An aided-INS estimator is typically required to initialize at least 11 DOF navigation states including roll, pitch, linear velocity, and the biases of the gyroscope and accelerometer.
If the sensor platform (robot) starts from a stationary position, 
the initialization becomes relatively easy and can be solved by averaging the IMU measurements
\citep{lin2022r, shan2021lvi}.
However, in many scenarios, the estimator is required to perform  dynamic initialization when the robot is in motion.
This problem in VINS is often addressed by aligning visual structure-from-motion trajectory to the inertial preintegration measurements \citep{qin2017robust,qin2018vins},
while \citet{dong2011closed} formulated and solved 
a quadratic programming (QP) problem with quadratic constraints.
In order to successfully initialize the estimator, these methods assume fully excited 3D motions 
which may not be possible for motion-constrained systems.
For example, the VINS estimator may lose scale information when a ground vehicle moves straight with constant acceleration \citep{wu2017vins}. 
While LiDAR-based methods are more robust to these problems as they can directly recover 3D pose information through 3D point-cloud matching (e.g. iterated closest points, or ICP \citep{yuan2022sr, yuan2023liw}),
their performance heavily relies on the matching quality, which could be challenging when the robot undergoes dynamic motions or there are not many overlapping points between scans.
On the other hand, for ground vehicles, wheel encoders can be leveraged to initialize the state,
which appears to be straightforward but has not been sufficiently investigated in the literature.
As the wheel-encoder can provide standalone odometry, most of the existing methods integrate wheel-encoder readings to achieve pose information and use it to initialize gyroscope bias \citep{liu2019visual, gang2020robust}, scale \citep{jiang2021panoramic, feng2021initialization}, or linear velocity \citep{gang2020robust, hou2022robust}. As these methods solve the initialization problem by setting relative-pose constraints which require double integration of the IMU and single integration of wheel-encoder readings, they may fail due to high nonlinearity.

To address these issues and enable robust estimation,
the proposed MINS supports not only the IMU-only static initialization~\citep{Geneva2020ICRA} 
but the IMU-wheel dynamic initialization which requires single IMU integration and no wheel integration.
After the initialization, we may perform navigation in the local frame $\{W\}$  
until GNSS measurements are available to initialize the system in the global frame $\{E\}$.
Fig.~\ref{fig:init} shows the overall procedure of the proposed MINS initialization.

\begin{figure}[t]  
    \centering
    \includegraphics[trim={0mm 0mm 0mm 0mm},clip, width=1\columnwidth]{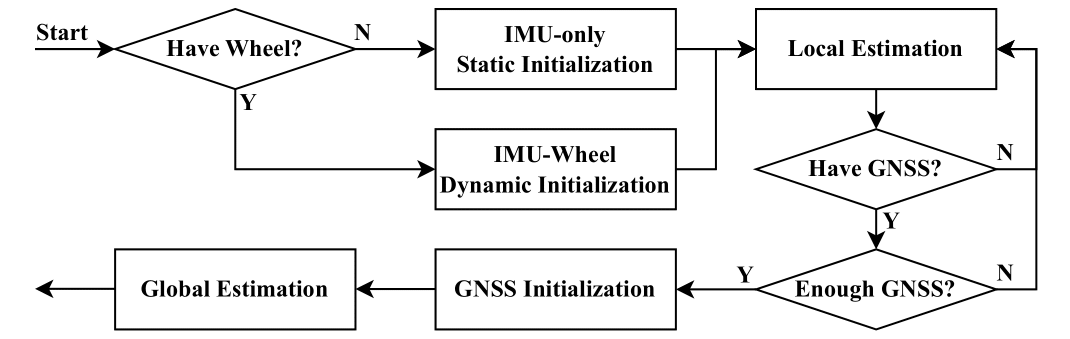}
    \caption{The proposed MINS initialization procedure.}
    \label{fig:init}
\end{figure}

\subsection{IMU-Wheel Dynamic Initialization} \label{sc:state_initialization}
To initialize, both IMU and wheel measurements are collected for a short period of time (e.g. less than 0.2 secs) and the angular and linear velocities of the wheel frame are pre-computed (see Eq.~\eqref{eq:wheel_meas_bodyframe}).
Then, $\mathbf{b}_g$ can be recovered as:
\begin{align}
    \mathbf{b}_g = {}^{I}\boldsymbol{\omega}_{avg} - \R{O}{I} {}^{O}\boldsymbol{\omega}_{avg}
\end{align}
where ${}^{O}\boldsymbol{\omega}_{avg}$ and ${}^{I}\boldsymbol{\omega}_{avg}$ are the average angular velocity of the wheel and IMU frames. 
Next, we can compute the linear velocity of the IMU frame $\vel{I_k}{I_0}$ in the following two ways:
\begin{align}
    \vel{I_k}{I_0} 
    &= \vel{}{I_0} + \sum_{i = 0}^{k - 1}( \R{I_i}{I_0} {}^{I_i}\mathbf{a} - \R{I_i}{I_0} \mathbf{b}_a  -  {}^{I_0}\mathbf{g}) \Delta t_i \label{eq:init_imu_int}\\
&= \R{I_k}{I_0}\R{O}{I}(\vel{}{O_k} + \skw{{}^{O_k} \boldsymbol{\omega}}\p{I}{O}) \label{eq:init_wheel_vel}
\end{align}
where $\{I_0\}$ is the first IMU frame among collected measurements and $\Delta t_i := t_{i} - t_{i - 1}$ is the integration period.
Eq.~\eqref{eq:init_imu_int} computes $\vel{I_k}{I_0}$ by first-order IMU integration and Eq.~\eqref{eq:init_wheel_vel} computes the same property directly from the wheel velocities. 
Based on the above relations, the linear system $\mathbf{A}\mathbf{x} = \mathbf{b}$ is created to recover the gravity in the first IMU frame and the accelerometer bias:
\begin{align}
    &\underbrace{\begin{bmatrix}
        -\sum_{i = 0}^{k - 1}\R{I_i}{I_0}\Delta t_i & -\sum_{i = 0}^{k - 1}\mathbf{I}_{3}\Delta t_i 
    \end{bmatrix}}_{\mathbf{A}}
    \underbrace{\begin{bmatrix}
        \mathbf{b}_a \\
        {}^{I_0}\mathbf{g}
    \end{bmatrix}}_{\mathbf{x}} \\
    &=
    \underbrace{\begin{bmatrix}
        \R{O_k}{I_0}(\vel{}{O_k} + \skw{{}^{O_k} \boldsymbol{\omega}}\p{I}{O}) - \vel{}{I_0} - \sum_{i = 0}^{k - 1}\R{I_i}{I_0} {}^{I_i}\mathbf{a} \Delta t_i  \notag
    \end{bmatrix}}_{\mathbf{b}}
\end{align}
In analogy to \citep{dong2011closed}, we can solve the above QP problem assuming the magnitude of the gravity is known.
Finally, the initial rotation $\R{W}{I_0}$ is computed with the Gram-Schmidt process using ${}^{I_0}\mathbf{g}$.
At this point, we have initialized the navigation state, including the orientation ($\R{W}{I_0}$), the biases ($\mathbf{b}_g$, $\mathbf{b}_a$), the linear velocity $\vel{}{I_0}$, and the zero position in the local frame $\{W\}$.

\subsection{Local-Global Frame Transformation}
\label{sc:gnss_initialization}

The GNSS measurement model~\eqref{eq:gps_measurement} assumes the IMU pose in the global frame $\{E\}$.
However, before the GNSS initialization, the system performs navigation in the gravity-aligned local frame $\{W\}$.
Thus, in order to process the GNSS measurements, the 4 DOF frame transformation  $(\R{W}{E}, \p{W}{E})$ must be known
(note that $\{E\}$ is also aligned with the gravity).
To this end, we collect two sets of estimates of the GNSS receiver positions in $\{E\}$ and $\{W\}$ and formulate a non-linear optimization problem to align them.
Note that if inaccurate GNSS measurements (e.g. single-point GPS) with large noise,
this alignment based on a short trajectory  may result in a poor transformation.
However, the current sliding window of the MSCKF typically contains a  most recent short trajectory and causes trouble in finding accurate local-global frame transformation.
Therefore, we augment our state by selectively keeping the cloned poses (or keyframes) at the GNSS measurement times and performing initialization once we reach the desired trajectory length.

Specifically, given a set of the GNSS position measurements in the ENU frame $\{\p{G_1}{E}, \cdots, \p{G_n}{E}\}$ within the keyframe window and the corresponding interpolated positions in the local frame $\{\p{G_1}{W}, \cdots, \p{G_n}{W}\}$, we use the following geometric constraints to determine the frame transformation:
\begin{align}
\p{G_i}{E} &= \p{W}{E} + \R{W}{E} \p{G_i}{W} , \forall i = 1\hspace{-0.5mm}\cdots\hspace{-0.5mm} n \label{eq:gps_init_stack}\\
\hspace{-2mm}
\Rightarrow \p{G_j}{E} 
\hspace{-0.5mm}
- 
\hspace{-0.5mm}
{}^E\mathbf{p}_{G_1} &= \R{W}{E} (\p{G_j}{W} 
\hspace{-0.5mm}
-
\hspace{-0.5mm}
\p{G_1}{W}) , \forall j=2\hspace{-0.5mm}\cdots\hspace{-0.5mm} n \label{eq:pG_in_E_and_V_sub}
\hspace{-2mm}
\end{align}
Note that if we have more than one GNSS receivers, we also stack their measurements in Eq.~\eqref{eq:gps_init_stack} with the corresponding extrinsics to allow more robust initialization.
Note also that due to the 4 DOF, instead of 6 DOF, transformation, we can  use the rotation about the global yaw  $\theta$:
\begin{align} \label{eq:z_rot}
    \R{W}{E} = \begin{bmatrix}
    \cos\theta && -\sin\theta && 0\\
    \sin\theta && \cos\theta && 0\\
    0 && 0 && 1
    \end{bmatrix}
\end{align}
With this, Eq.~\eqref{eq:pG_in_E_and_V_sub} becomes:
\begin{align} 
    \mathbf{A}_j\begin{bmatrix}
    \cos\theta\\
    \sin\theta
    \end{bmatrix}
    :=
    \mathbf{A}_j \mathbf{w}
    = \mathbf{b}_j , \forall j=2\cdots n
\end{align}
Stacking all these constraints yields the following linear least-squares with quadratic constraints:
\begin{align} \label{eq:qclsp}
    \min_{\mathbf{w}} ~ ||\mathbf{Aw}-\mathbf{b}||^2, ~~ {\rm s.t.} ~||\mathbf{w}||^2 = 1
\end{align}
We substitute the solution of Eq.~\eqref{eq:qclsp} into Eq.~\eqref{eq:gps_init_stack} and solve for $\phat{W}{E}$:
\begin{align}
    \phat{W}{E} = \frac{1}{n}\sum^n_{i=1}
    \left[ \p{G_i}{E} - \Rhat{W}{E} \p{G_i}{W} \right]
\end{align}
This initial guess $(\Rhat{W}{E}, \phat{W}{E})$  is further refined with  delayed initialization~\citep{li2014visual, dutoit2017consistent}.
Specifically, by augmenting the state vector with the transformation along with an {\em infinite} covariance prior to these new variables, we perform the EKF update with all the GNSS measurements:
\begin{align}
    \mathbf{z}_{G} = \p{W}{E} + \R{W}{E} (\p{I_k}{W} + \R{I_k}{W} \p{G}{I}) + \mathbf{n}_{G} \label{eq:gnss_meas_in_world}
\end{align}
After initialization, we marginalize all the keyframes to reduce the state to the original state size.

\begin{figure*}[t]
\centering
\begin{subfigure}{.30\textwidth}
\includegraphics[trim=12mm 5mm 15mm 10mm,clip,width=\linewidth]{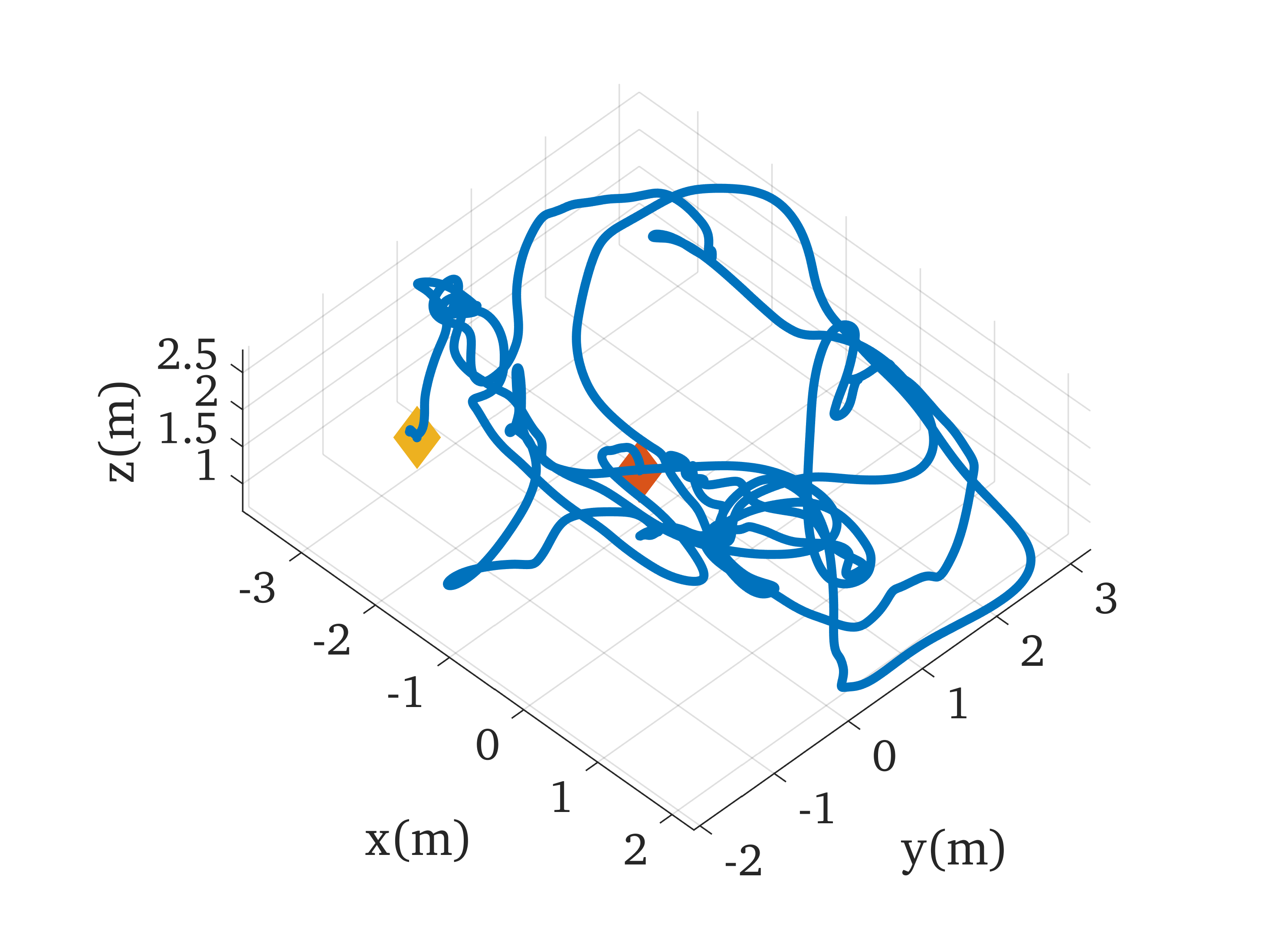}
\caption{\textit{EuRoc Vicon Room2 02}}
\label{fig:sim_traj:euroc_vicon2}
\end{subfigure}
\begin{subfigure}{.30\textwidth}
\includegraphics[trim=9mm 5mm 15mm 10mm,clip,width=\linewidth]{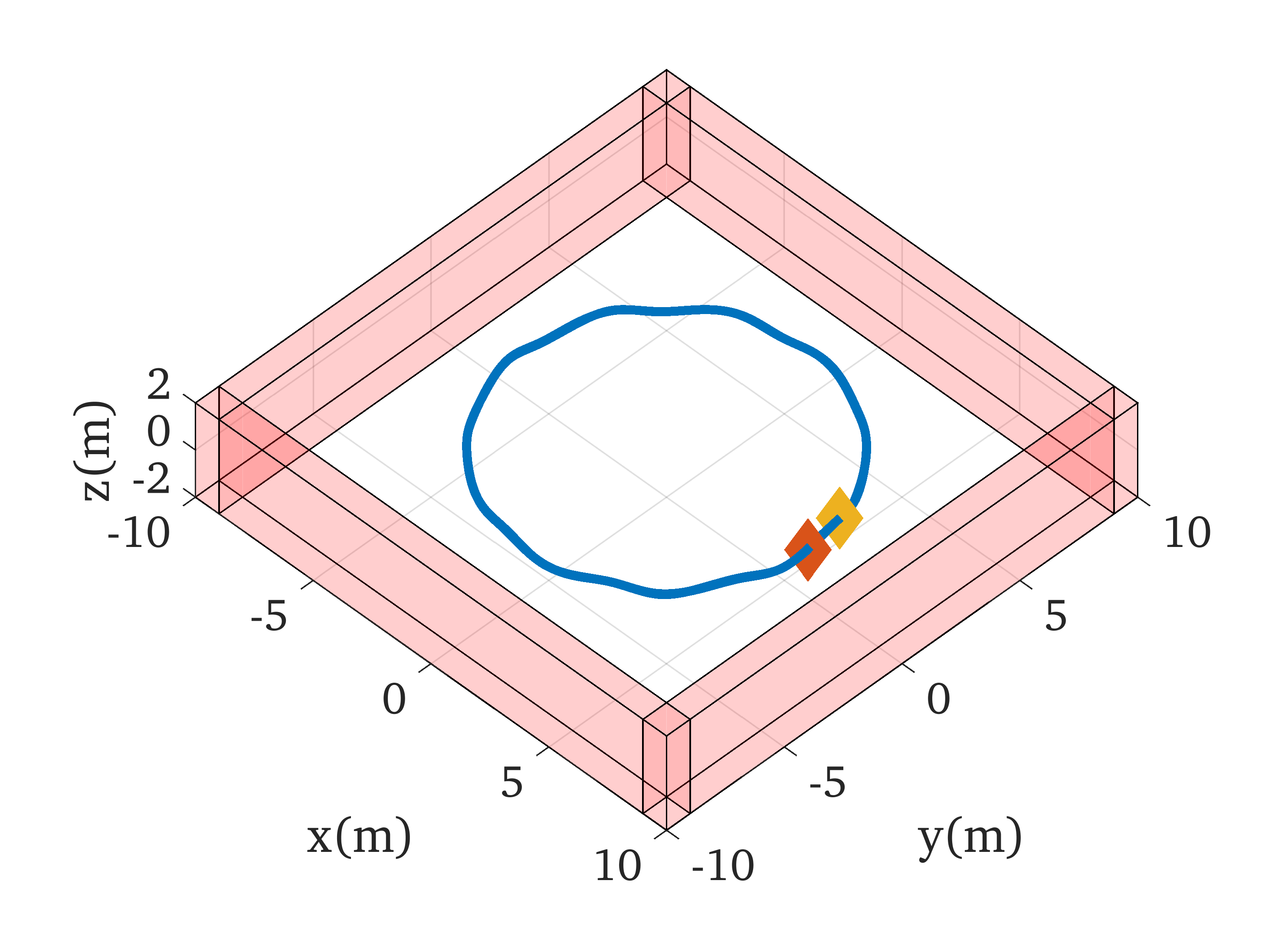}
\caption{\textit{UD Small}}
\label{fig:sim_traj:ud_small}
\end{subfigure}
\begin{subfigure}{.30\textwidth}
\includegraphics[trim=7mm 5mm 15mm 10mm,clip,width=\linewidth]{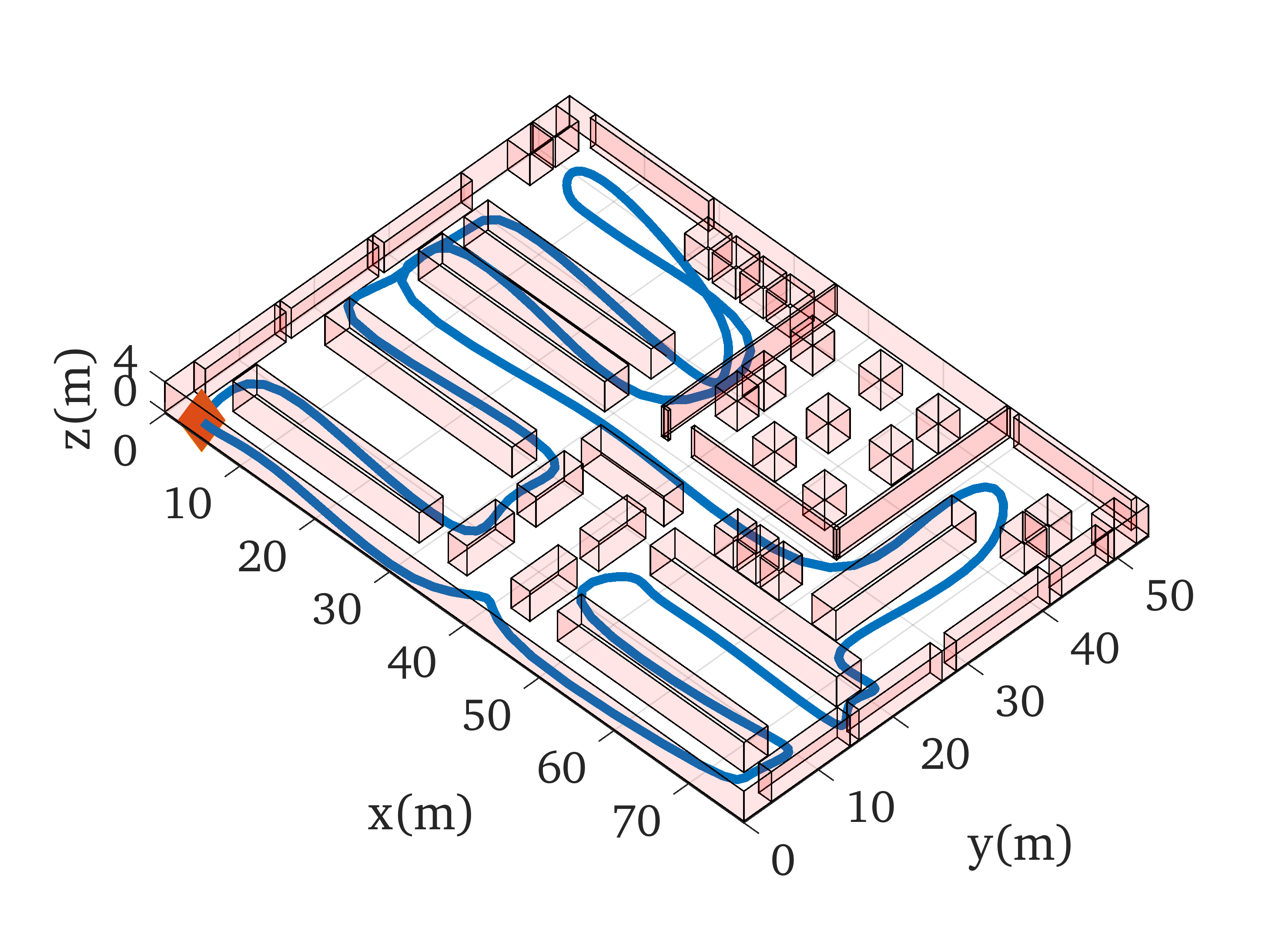}
\caption{\textit{UD Warehouse}}
\label{fig:sim_traj:ud_warehouse}
\end{subfigure}
\caption{
Simulated trajectories for Monte-Carlo simulations.
(a): \textit{EuRoc Vicon Room2 02} (115s, \citet{Burri25012016}) with fully excited 3D motion;
(b) \& (c): \textit{UD Small} (60s) \& \textit{UD Warehouse} (52s) with nonholonomic constrained 3D motion along with simulated walls, floor, and ceiling (floor and ceiling are not shown for better visibility).
The red and yellow diamonds denote the beginning and end of these trajectories, respectively. 
}
\label{fig:sim_traj}
\end{figure*}

\subsection{Global Navigation}
\label{ch:gps_tr_global}

As shown in our prior work \citep{Lee2020ICRA}, the GNSS-aided INS in a local frame $\{W\}$ using the global-to-local transformation  $(\R{W}{E}, \p{W}{E})$ to fuse GNSS measurements is not fully observable, although it gains global information.
\begin{lem}
If estimating states in the local frame, even with global GNSS measurements, the system remains unobservable and has four unobservable directions.
\end{lem}
\begin{proof}
See Appendix~\ref{ch:apdx_gps_obs_local}.
\end{proof}

While this result is somewhat counter-intuitive,
the root cause of this unobservability is the gauge freedom of the 4DOF frame transformation between $\{W\}$ to $\{E\}$, which is inherited from INS (i.e. 4 unobservable directions \citep{hesch2012observability}). 
Even though we use the global measurements, the system has a non-trivial nullspace.
As unobservable linearized systems may have erroneous nullspace due to improper linearization points,
and the corresponding linearized estimator thus may gain information in spurious directions, causing inconsistency~\citep{hesch2012observability,Huang2008ISER,Chen2022ICRA}.
To address this issue, we perform state estimation directly in the global frame  once initialized, 
which is shown to make the system fully observable.

\begin{lem}

If estimating states in the global frame, the system is fully observable.
\end{lem}
\begin{proof}
See Appendix~\ref{ch:apdx_gps_obs_global}.
\end{proof}
After initialization, we hence transform the state from $\{W\}$ to  $\{E\}$ as well as propagate the corresponding covariance.
Specifically, we transform the state in local ${}^W\x{k}$ to global ${}^E\x{k}$ as:
\begin{align}
    {}^E\x{k} = \mathbf{g}({}^W\x{k},\R{W}{E},\p{W}{E}) \label{eq:trans_state_func}
\end{align}
where
\begin{align}
    \R{E}{I_k} &= \R{W}{I_k} \R{E}{W} \label{eq:trans_state_start}\\
    \p{I_k}{E} &= \p{W}{E} + \R{W}{E} \p{I_k}{W}  \label{eq:trans_state_end}\\
    \vel{I_k}{E} &= \R{W}{E} \vel{I_k}{W}
\end{align}
Note that the above equations are used to warp all historical poses $\x{H_k}$. 
Now we linearize Eq.~\eqref{eq:trans_state_func} at the current estimate to get the Jacobian  $\bm\Psi$ and propagate the covariance as follows:
\begin{align}
    {}^E\xtilde{k} &= \bm\Psi{}^W\xtilde{k} ~,~~
    {}^E\mathbf{P}_k = \bm\Psi {}^W\mathbf{P}_k \bm\Psi^\top
\end{align}
where the detailed derivations of  $\bm\Psi$ can be found in Appendix~\ref{ch:apdx_gps_state_transform}.
After transformation, we marginalize ($\R{W}{E}$, $\p{W}{E}$) from the state as they are no longer needed.

\section{Simulation Results} \label{ch:simulation}

\begin{table}[t]
\centering
\caption{Simulation setup parameters.}
\begin{threeparttable}
\footnotesize
\setlength{\tabcolsep}{1pt}
\begin{tabular}{@{}cc|ccc@{}}
\toprule
\textbf{Parameter} & \textbf{Value} & \textbf{Parameter} & \textbf{Value}\\ \midrule
Clone Freq. (Hz)        & 20    & Wheel Mode                & \textit{Wheel3DAng} \\
Window Size. (s)        & 1     & Wheel Ang. Noise (rad/s)  & 1e-2 \\
Intr. Order             & 3     & LiDAR Freq. (Hz)          & 10 \\
Dyn. Ori. Thresh (rad)  & 7e-3  & \# of LiDAR               & 2 \\
Dyn. Pos. Thresh (m)    & 3e-3  & LiDAR Noise (m)           & 1e-2 \\
IMU Freq. (Hz)          & 200   & $\ang{I}{C_1}$ (rad)      & 1.57, 0.00, 0.00 \\
IMU Acc. Noise          & 2e-2  & $\p{I}{C_1}$ (m)          & -0.01, 0.01, 0.01 \\
IMU Acc. Bias           & 3e-2  & $\ang{I}{C_2}$ (rad)      & 1.57, 0.00, 0.00 \\
IMU Gyr. Noise          & 2e-3  & $\p{I}{C_2}$ (m)          & 0.01, 0.01, 0.01 \\
IMU Gyr. Bias           & 2e-4  & $\p{G_1}{I}$ (m)          & 1.00, 1.00, 1.00 \\
Cam Freq. (Hz)          & 30    & $\p{G_2}{I}$ (m)          & -1.00, -1.00, -1.00 \\
\# of Cam               & 2     & $\ang{I}{O}$ (rad)        & 0.00, 0.00, 0.00 \\
Cam Noise (pix)         & 1     & $\p{I}{O}$ (m)            & 0.00, 0.00, 0.00 \\
GNSS Freq. (Hz)         & 1     & $\ang{I}{L_1}$ (rad)      & 0.26, 0.00, 0.00 \\
\# of GNSS              & 2     & $\p{I}{L_1}$ (m)          & 0.30, 0.30, 0.50 \\
GNSS Noise (m)          & 0.1     & $\ang{I}{L_2}$ (rad)      & 0.00, 1.57, 0.00 \\
Wheel Freq. (Hz)        & 100   & $\p{I}{L_2}$ (m)          & -0.01, -0.01, -0.01 \\ \bottomrule
\end{tabular}\begin{tablenotes}\footnotesize
\raggedleft \item[*] All sensor time offsets are set to 0.00 s.
\end{tablenotes}
\end{threeparttable}
\label{tab:sim_setup}
\end{table}

We perform extensive Monte-Carlo simulation tests to validate the  proposed MINS estimator.
Built on top of the OpenVINS simulator \citep{Geneva2019ICRA,Geneva2018IROS}, our MINS simulator generates synthetic  measurements of IMU, camera, GNSS, wheel encoder and LiDAR based on the realistic trajectories  as shown in Fig. \ref{fig:sim_traj}.
Table~\ref{tab:sim_setup} summarizes the setup parameters used in the ensuing simulations.

\subsection{Effects of Interpolation Errors} \label{ch:sim_intr_error_model}

\begin{table}[t]
\newcommand{\tmt}[2]{\multirow{#1}{*}{\rotatebox[origin=c]{90}{#2}}}
\caption{Orientation/position RMSE and NEES results of MINS with/without incorporating interpolation error model. The inconsistent results (NEES $>$ 4) are highlighted in red.}
 \begin{adjustbox}{width=1.0\columnwidth,center}
  \begin{tabular}{c|c|cc}
   \toprule
    & \textbf{Hz} & \textbf{RMSE (deg / m)} & \textbf{NEES} \\ \midrule
    \tmt{5}{w Model} 
    & 04 & 0.550 $\pm$ 0.084 / 0.061 $\pm$ 0.022 & 3.0 $\pm$ 1.1 / 2.1 $\pm$ 1.5 \\ 
    & 06 & 0.459 $\pm$ 0.262 / 0.064 $\pm$ 0.014 & 3.9 $\pm$ 1.4 / 2.2 $\pm$ 1.0 \\ 
    & 10 & 0.224 $\pm$ 0.032 / 0.036 $\pm$ 0.011 & 3.8 $\pm$ 0.8 / 1.4 $\pm$ 0.8 \\ 
    & 20 & 0.172 $\pm$ 0.060 / 0.023 $\pm$ 0.012 & 3.9 $\pm$ 1.2 / 1.0 $\pm$ 0.9 \\ 
    & 30 & 0.178 $\pm$ 0.110 / 0.019 $\pm$ 0.011 & 3.2 $\pm$ 1.3 / 0.8 $\pm$ 1.0 \\   \midrule 
    \tmt{5}{w/o Model} 
    & 04 & 1.242 $\pm$ 1.084 / 1.479 $\pm$ 1.113 & \textcolor{red}{2e1} $\pm$ 2e1 / \textcolor{red}{3e3} $\pm$ 2e3 \\ 
    & 06 & 1.982 $\pm$ 0.815 / 0.631 $\pm$ 0.632 & \textcolor{red}{3e1} $\pm$ 2e1 / \textcolor{red}{2e3} $\pm$ 2e3 \\ 
    & 10 & 0.301 $\pm$ 0.106 / 0.033 $\pm$ 0.028 & \textcolor{red}{5.1} $\pm$ 4.0 / 2.0 $\pm$ 2.0 \\ 
    & 20 & 0.200 $\pm$ 0.127 / 0.024 $\pm$ 0.018 & \textcolor{red}{4.5} $\pm$ 1.2 / 1.1 $\pm$ 1.0 \\ 
    & 30 & 0.178 $\pm$ 0.110 / 0.019 $\pm$ 0.011 & 3.2 $\pm$ 1.3 / 0.8 $\pm$ 0.6 \\  \bottomrule
  \end{tabular}
 \end{adjustbox}
 \label{tab:intr_erorr_model_results}
\end{table}

To verify our proposed interpolation error model, we tested the proposed MINS with and without incorporating the error model at different cloning frequencies (4, 6, 10, 20, and 30 Hz) on the simulated \textit{EuRoC Vicon Room2 02} (Fig. \ref{fig:sim_traj:euroc_vicon2}).
We simulated a stereo camera and IMU for this test as the camera processes multiple measurements tracked over time which requires the interpolation of all over the MSCKF window.
Table~\ref{tab:intr_erorr_model_results} shows the orientation and position root mean squared error (RMSE) and normalized estimation error-squared (NEES) results of each cloning frequency averaged over 10 Monte-Carlo runs. 
Note that the 30 Hz cloning frequency results of both with and without the error model are the same because no interpolation was performed in both cases as the clone frequency and camera measurement rate are the same.

Clearly, the MINS estimators without the error model are shown to be inconsistent and overconfident for all the cases except 30 Hz as their NEESs are above 4. 
Note that ideal NEES of orientation and position should be around 3 if the estimator is consistent as their dimensions are 3.
The estimators tend to show higher NEES and larger RMSE with lower cloning frequencies which indicates the unmodelled interpolation error has a larger impact on the lower setup.
On the other hand, all the estimators with the error model were shown to be consistent with NEES below 4. 
The RMSE results are also lower than those without the error model showing that improved estimator consistency can benefit the localization performance.
\subsection{Dynamic Cloning}

Using the same setup, we now investigate how the proposed dynamic cloning balances localization performance and computational burden.
We control the thresholds of dynamic cloning by multiplying different coefficients (0.01 - 100) to the default values (see Table~\ref{tab:sim_setup}).
Fig. \ref{fig:sim_pose_error_dyn} shows the pose RMSE and total computation time results of dynamic cloning (red circles) which are averaged over 10 runs.
The results of the estimator without dynamic cloning are also plotted with blue dots (fixed-rate cloning).

It is clear from these results  that if we set the threshold very low (e.g. 0.01 in the figure) the performance of dynamic cloning is almost the same as the highest fixed-rate cloning frequency (30 Hz).
This is because the dynamic cloning will enforce the MINS to maintain the interpolation error small which will end up setting the cloning frequency the highest most of the time.
On the other hand, in the case we set the threshold very high (e.g. 100 in the figure), the performance of dynamic cloning is close to the lowest fixed-rate cloning frequency (4 Hz).
The impact of dynamic cloning on how it balances accuracy and efficiency is more clear within the range of 0.1 - 10.
In the case of 0.1 and 1, the accuracies of dynamic cloning were comparable to the 30 Hz fixed-rate cloning, while the computation times were only 75.9\% and 61.9\%, respectively.
Also, in the case of 10, dynamic cloning was able to show computation time comparable to 4 Hz fixed-rate cloning, while the accuracy was better than 6 Hz fixed-rate cloning.
Table~\ref{tab:dyn_table} summarizes the pose RMSE and computation times of dynamic cloning with each threshold coefficient compared to 30 Hz fixed-rate cloning.
The right column of Fig. \ref{fig:sim_pose_error_dyn} shows an exemplary case of how the MINS changed its cloning frequency with given angular/linear accelerations and coefficient 1.

\begin{figure}[t]  
    \centering
    \begin{subfigure}{.49\columnwidth}
    \includegraphics[trim=0mm 4mm 0mm 7mm,clip,width=1\columnwidth]{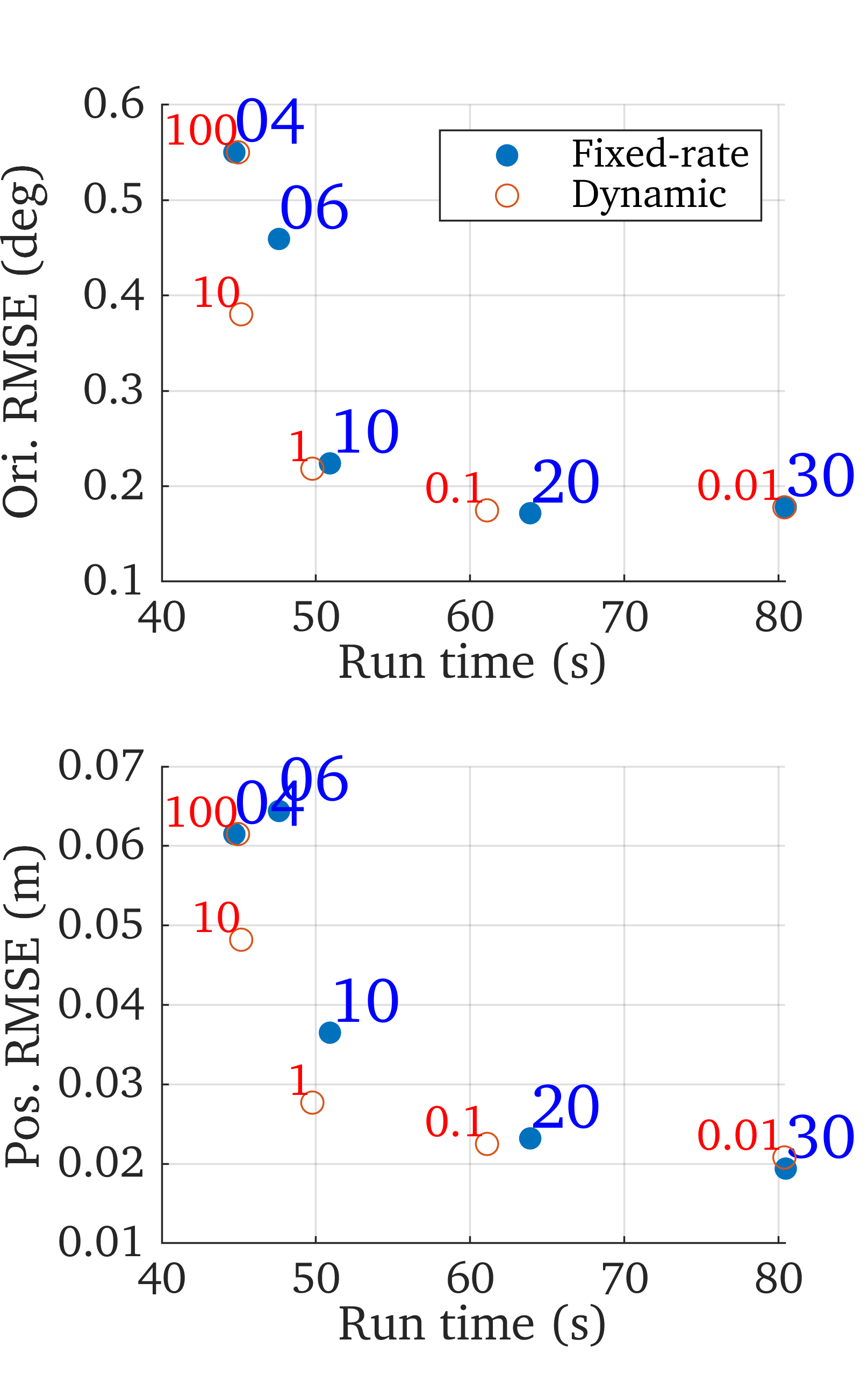}
    \end{subfigure}
    \begin{subfigure}{.49\columnwidth}
    \includegraphics[trim=0mm 4mm 0mm 7mm,clip,width=1\columnwidth]{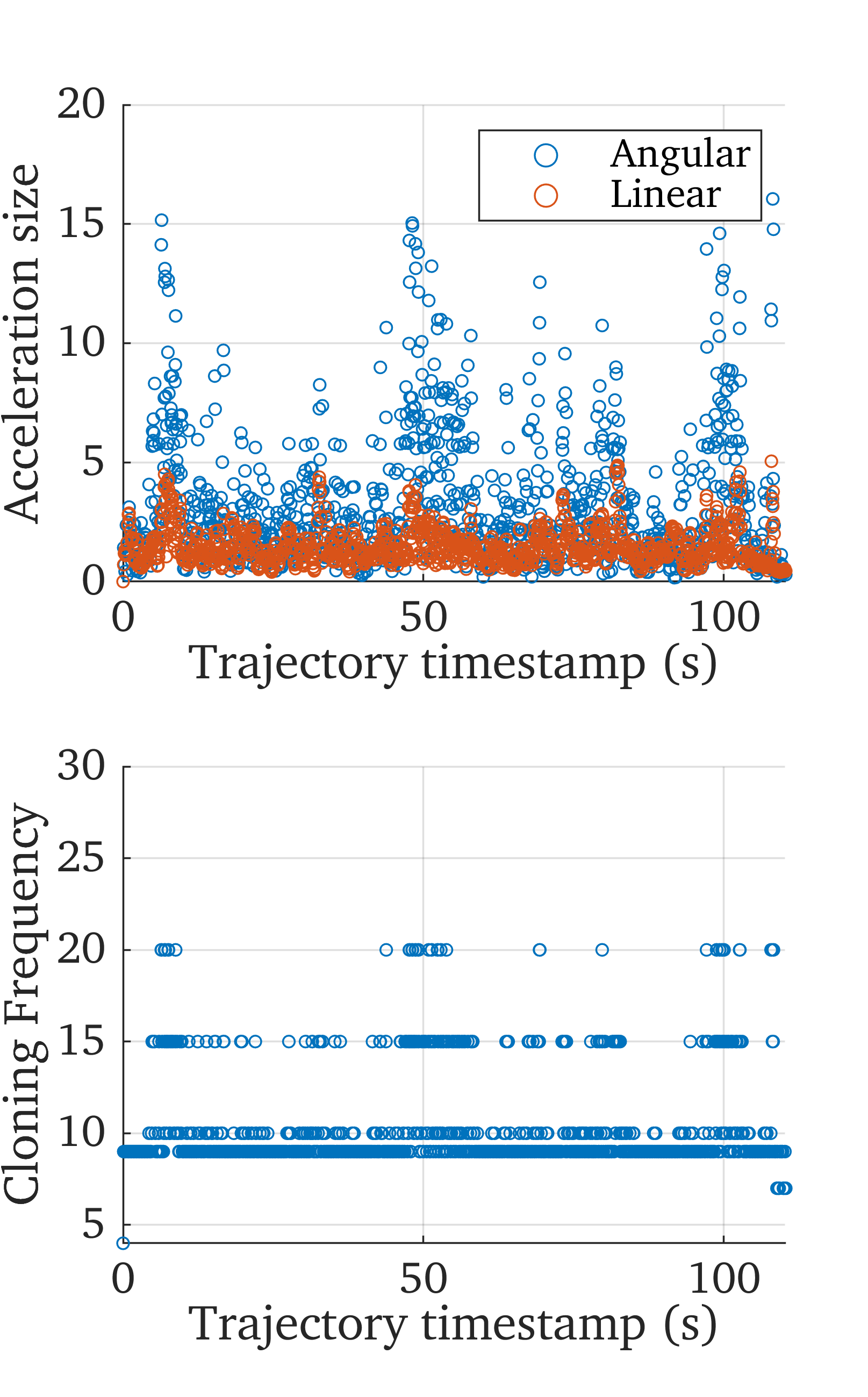}
    \end{subfigure}
    \caption{
    \textit{Left}: Pose RMSE-total run time comparison between fixed-rate cloning and dynamic cloning. The fixed-rate clonings are plotted with blue dots along with frequencies, and dynamic clonings are shown with red circles with threshold coefficients.
    \textit{Right}: Accelerations and cloning frequencies chosen by dynamic cloning with coefficient 1.
    }
    \label{fig:sim_pose_error_dyn}
\vspace{-0.3cm}
\end{figure}

\begin{table}[t]
\caption{Pose RMSE (deg/m) and total computation time (s) of fixed-rate and dynamic cloning with different coefficients.}
 \begin{adjustbox}{width=1.0\columnwidth,center}
  \begin{tabular}{c|c|ccccc}
   \toprule
    & \textbf{Fixed} & \multicolumn{5}{c}{\textbf{Dynamic cloning threshold coefficients}}\\
    & \textbf{30 Hz} & \textbf{0.01} & \textbf{0.1} & \textbf{1} & \textbf{10} & \textbf{100} \\ \midrule
    \textbf{Ori.} & 0.178 & 0.178 & 0.175 & 0.218 & 0.380 & 0.550 \\ 
    \textbf{Pos.} & 0.019 & 0.021 & 0.023 & 0.028 & 0.048 & 0.062\\ \midrule
    \multirow{2}{*}{\textbf{Time}} & 80.5 & 80.4 & 61.1 & 49.8 & 45.1 & 44.9\\   
    & (100\%) & (99.9\%) & (75.9\%) & (61.9\%) & (56.1\%) & (55.9\%)\\\bottomrule
  \end{tabular}
 \end{adjustbox}
 \label{tab:dyn_table}
\end{table}

\subsection{Fusing Different Sensors}

Based on the interpolation strategy, we validate the proposed MINS with different combinations of available sensors.
We test it  on the \textit{UD Small} trajectory with non-holonomic constraints  (see Fig. \ref{fig:sim_traj:ud_small}), as two differential wheeled vehicles cannot perform the holonomic motion.
Note that the GNSS measurement noise level is set to 0.1 m (see Table \ref{tab:sim_setup}), differential GPS quality to get comparable results with other sensor suits.
Table \ref{tab:sensor_combi} shows the orientation and position RMSE and NEES, and the run time results of each sensor combination averaging 10 runs. 
All results are collected with an Intel i7 CPU single-threaded.
Overall, the proposed MINS with all sensor combinations show consistent performance with NEES under 4.
The IMU-camera and IMU-LiDAR pairs show comparable accuracies to each other and both ran 2.6 times faster than real time, while IMU-GNSS and IMU-wheel have lower accuracies with extremely small computation.
The accuracies are shown to be improved with more sensors used in this case and the best accuracy was achieved when fusing all the sensors, confirming the benefits of multi-sensor fusion.

\begin{figure}[t]  
    \centering
    \includegraphics[trim=5mm 5mm 10mm 5mm,clip,width=1\columnwidth]{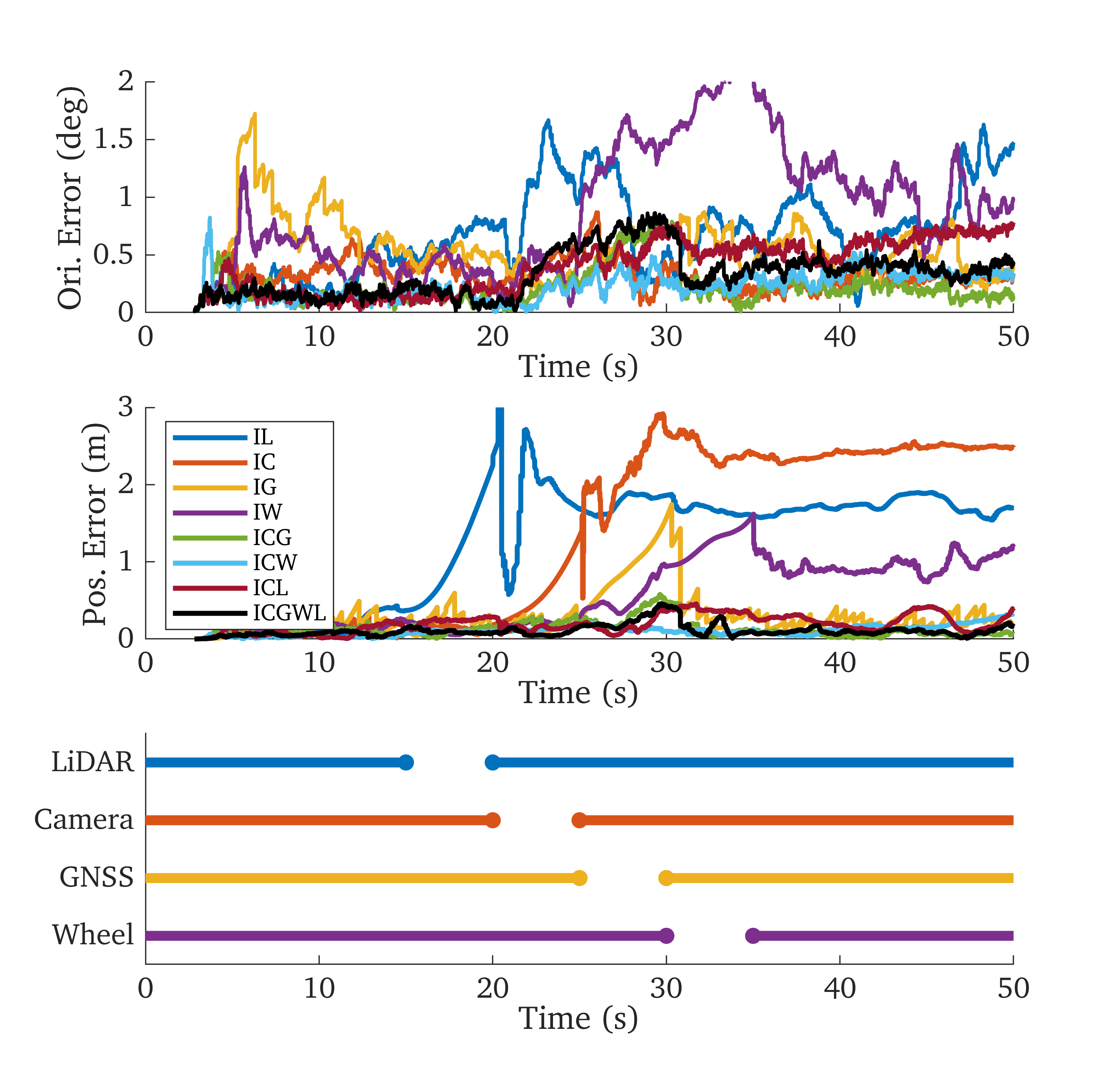}
    \caption{
    The performance of each algorithm under sensor failure scenarios. Pose estimation errors over time are displayed in the top two figures, and the times each sensor measurement was available are shown in the bottom figure (I: IMU, C: camera, G: GNSS, W: Wheel, L: LiDAR).
    }
    \label{fig:sim_sensor_drop}
\vspace{-0.3cm}
\end{figure}

\begin{table}[t]
\caption{Orientation/position RMSE and NEES, and run time results of MINS with different combinations of sensors. The best RMSE results are highlighted in bold text.}
\def\ck{ \tikz\draw[black,fill=black] (0,0) circle (0.7ex);} 
 \begin{adjustbox}{width=0.5\textwidth,center}
\begin{threeparttable}
\setlength{\tabcolsep}{4pt}
  \begin{tabular}{c|c|c|c|c|ccc}
   \toprule
    \textbf{I} & \textbf{C} & \textbf{G} & \textbf{W} & \textbf{L} & \textbf{RMSE (deg / m)} & \textbf{NEES} & \textbf{Time (s)}\\ \midrule
    \ck & \ck & & & & 0.505 $\pm$ 0.226 / 0.139 $\pm$ 0.054 & 3.2 $\pm$ 1.4 / 2.4 $\pm$ 1.4 & 23.6 $\pm$ 0.3 \\ \midrule 
    \ck & & \ck & & & 1.244 $\pm$ 0.250 / 0.191 $\pm$ 0.015 & 3.2 $\pm$ 1.3 / 3.6 $\pm$ 0.7 & 0.5 $\pm$ 0.0 \\ \midrule 
    \ck & & & \ck & & 3.053 $\pm$ 1.603 / 0.636 $\pm$ 0.159 & 2.4 $\pm$ 1.0 / 0.9 $\pm$ 0.3 & 1.0 $\pm$ 0.0 \\ \midrule 
    \ck & & & & \ck & 0.474 $\pm$ 0.106 / 0.098 $\pm$ 0.019 & 1.9 $\pm$ 0.6 / 1.1 $\pm$ 0.3 & 22.3 $\pm$ 0.6 \\ \midrule 
    \ck & \ck & \ck & & & 0.318 $\pm$ 0.042 / 0.057 $\pm$ 0.008 & 3.8 $\pm$ 1.4 / 3.6 $\pm$ 1.2 & 23.6 $\pm$ 0.2 \\ \midrule 
    \ck & \ck & & \ck & & 0.505 $\pm$ 0.231 / 0.102 $\pm$ 0.029 & 3.2 $\pm$ 1.4 / 2.2 $\pm$ 0.8 & 24.0 $\pm$ 0.3 \\ \midrule 
    \ck & \ck & & & \ck & 0.414 $\pm$ 0.174 / 0.084 $\pm$ 0.038 & 2.6 $\pm$ 1.2 / 2.3 $\pm$ 1.4 & 49.6 $\pm$ 0.4 \\ \midrule 
    \ck & \ck & \ck & \ck & \ck & \textbf{0.261} $\pm$ 0.054 / \textbf{0.050} $\pm$ 0.005 & 3.0 $\pm$ 1.3 / 3.9 $\pm$ 1.1 & 49.9 $\pm$ 0.2 \\ \bottomrule
  \end{tabular}
\begin{tablenotes} \footnotesize
\raggedleft \item[*] I: IMU, C: Camera, G: GNSS, W: Wheel, L: LiDAR
\end{tablenotes}
\end{threeparttable}
 \end{adjustbox}
 \label{tab:sensor_combi}
\end{table}
\subsection{Robust to Sensor Failure}
To show the robustness of the proposed MINS to sensor failures, 
we simulate the scenarios that measurements of each sensor are dropped for 5 seconds at different times.  
We use \textit{UD Warehouse} (Fig. \ref{fig:sim_traj:ud_warehouse})  for this test, to be more realistic and ensure the LiDAR is not necessarily getting the map matching after the sensor failure (no loop-closure to the map built before sensor drop). 
The pose errors of different combinations of sensors and the period of time when the sensor measurements are available are shown in Fig. \ref{fig:sim_sensor_drop}.
Clearly, two sensor-paired systems are shown to have large error increases during sensor failure and maintained the error even after the sensors are recovered (except IG which has GNSS).
On the other hand, those systems combining more than 2 sensors were able to bound the error growth by leveraging the auxiliary sensors, showing the robustness to the sensor failure and providing more accurate estimation.

\begin{figure*}[t]
    \centering
    \includegraphics[trim={65mm 0mm 50mm 0mm},clip,width=1\textwidth]{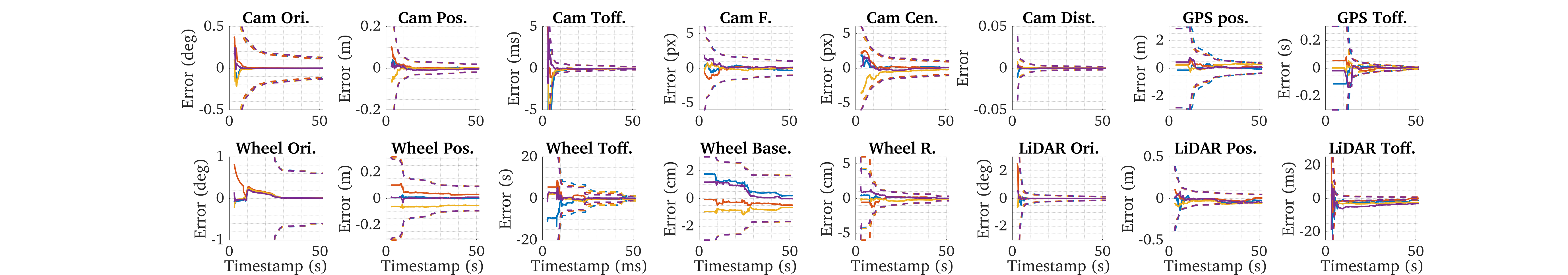}
    \caption{Calibration error of 4 runs with different initial guesses (different colors). The errors and their 3$\sigma$ envelopes are shown with solid lines and dotted lines, respectively.}
    \label{fig:sim_calib_convergences}
\end{figure*}
\begin{figure*}[t]
\centering
\includegraphics[trim=0mm 0 0mm 0,clip,width=\textwidth]{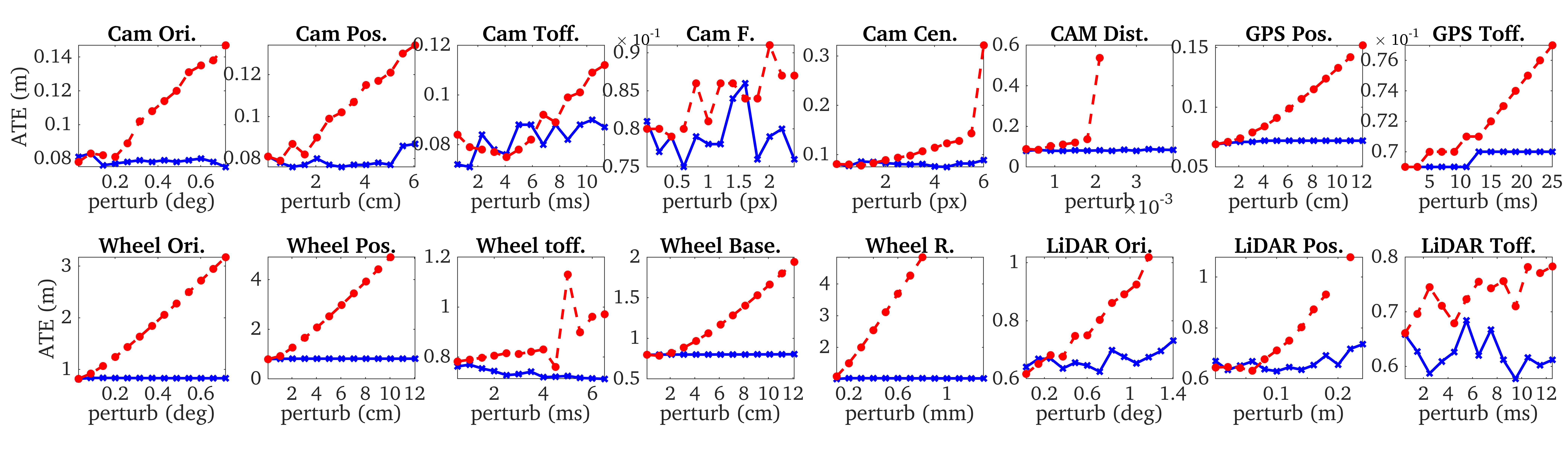}
\caption{
Position ATE under different levels of sensor calibration parameter perturbations. Results of with online calibration and without calibration are shown with red and blue lines.
}
\label{fig:sim_perturb_errors_ate}
\end{figure*}
\begin{figure*}[t]
\centering
\includegraphics[trim=0mm 0 0 0,clip,width=\textwidth]{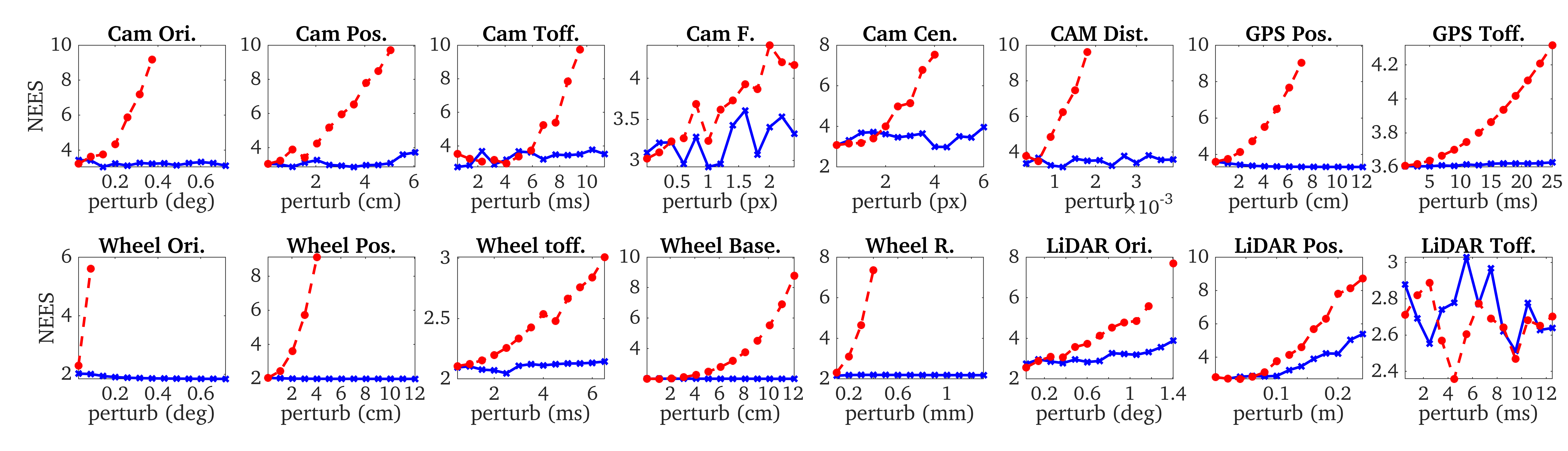}
\caption{
Position NEES under different levels of sensor calibration parameter perturbations. Results of with online calibration and without calibration are shown with red and blue lines.
}
\label{fig:sim_perturb_errors_nees}
\end{figure*} 
\subsection{Robust to Poor Calibration}

To evaluate the online calibration of the proposed MINS, we test each calibration giving 4 different initial errors.
Fig. \ref{fig:sim_calib_convergences} shows the calibration errors of spatiotemporal extrinsic and intrinsic of each sensor over time.
It is clear that each error converges quickly to near zero after the system is initialized and remains bounded by 3$\sigma$ envelopes showing the consistency of the online calibration.
Note we showed 1-dimensional errors by summing the error of each axis of the calibration parameter, e.g. the position error is the summation of x, y, z errors, for concise visualization while the results of the individual axis also shows the same consistent results.

We further validate the necessity of online calibration for accurate and consistent localization by enabling/disabling the calibration at different levels of initial calibration errors.
To be more specific, we investigated how robust the system is to the initial perturbations and whether the use of online sensor calibration enables improvements in accuracy and consistency.
As shown in Fig. \ref{fig:sim_perturb_errors_ate} and Fig. \ref{fig:sim_perturb_errors_nees}, for each of the different calibration parameters we perturb it with different levels of noise (note that we also change the initial covariance of the variable corresponds to the noise level used for perturbation).
We can see that the proposed estimator is relatively invariant to the initial inaccuracies of the parameters and is able to output a near-constant trajectory error and consistency.
An estimator, which does not perform this online estimation, has its trajectory estimation error and NEES values quickly increase to non-usable levels which validate the use of online calibration.

\section{Experimental Results} \label{ch:realworld}

We further evaluated the proposed MINS on both the public KAIST Urban Dataset \citep{jeong2019complex} and our own collected UD Husky Dataset, whose details are the following:
\begin{itemize}
    \item \textbf{KAIST Urban Dataset}: collected on ground vehicles in challenging scenarios -- highways (18-25, 35-37) and cities (26-34, 38-39) -- and have a 100 Hz Xsens MTi-300 IMU, a 10 Hz Pointgrey Flea3 stereo camera, two 10 Hz Velodyne 16 channel LiDARs, a 10Hz u-blox EVK-7P GNSS, and a 100 Hz RLS LM13 wheel encoders.
    \item \textbf{UD Husky Dataset}: collected on a Clearpath  Husky robot in structured environments (indoors and outdoors) and unstructured environments (trails). The dataset contains a T265 stereo camera (two 30 Hz cameras and a 200 Hz IMU), a 1 Hz Garmin GPS 18x,
    a 5 Hz Emlid reach m+ GPS (RTK), a 10 Hz Ouster 64 channel LiDAR, and a 10 Hz wheel encoders. We considered the OptiTrack information (indoor) or RTK-GPS measurements (outdoor, trail) as ground truth for evaluation (see Fig. \ref{fig:husky_dataset}).
\end{itemize}
We compared our MINS to the following state-of-the-art (SOTA) methods as  benchmarks:
\begin{figure*}[t]
\centering
\begin{subfigure}{.32\textwidth}
\includegraphics[trim=0mm 0mm 8mm 15mm,clip,width=\linewidth]{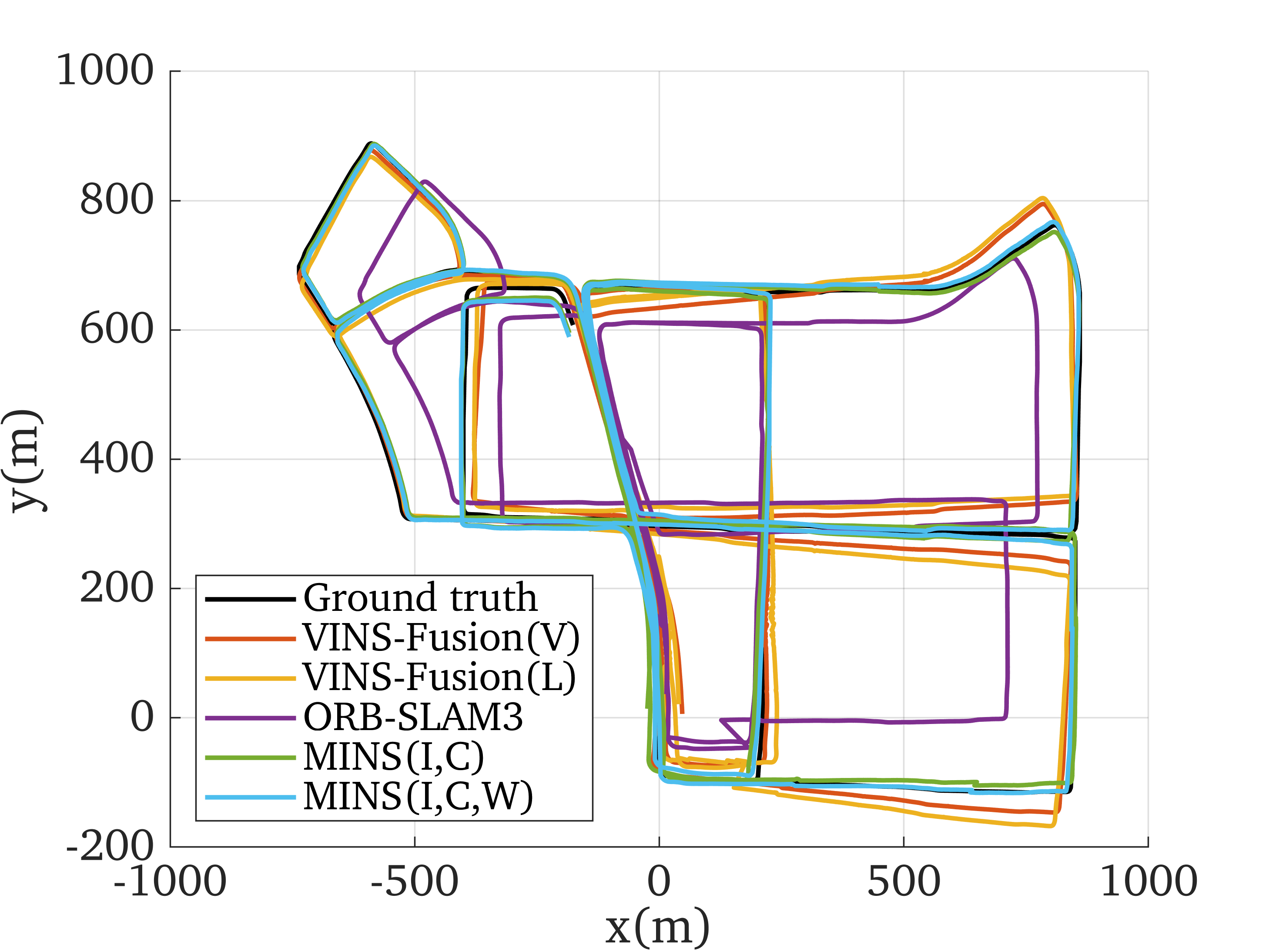}
\end{subfigure}
\begin{subfigure}{.32\textwidth}
\includegraphics[trim=0mm 0mm 8mm 15mm,clip,width=\linewidth]{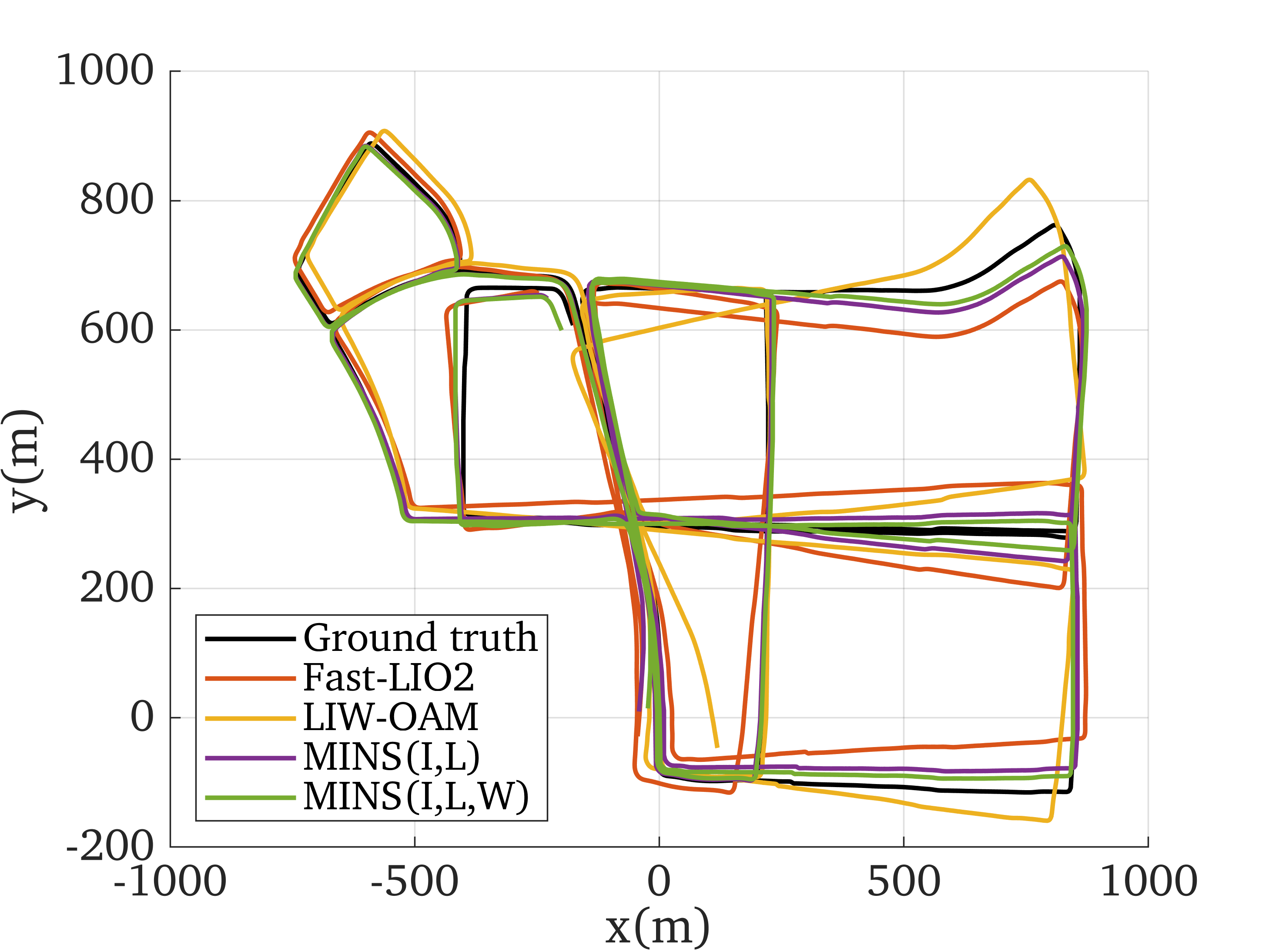}
\end{subfigure}
\begin{subfigure}{.32\textwidth}
\includegraphics[trim=0mm 0mm 8mm 15mm,clip,width=\linewidth]{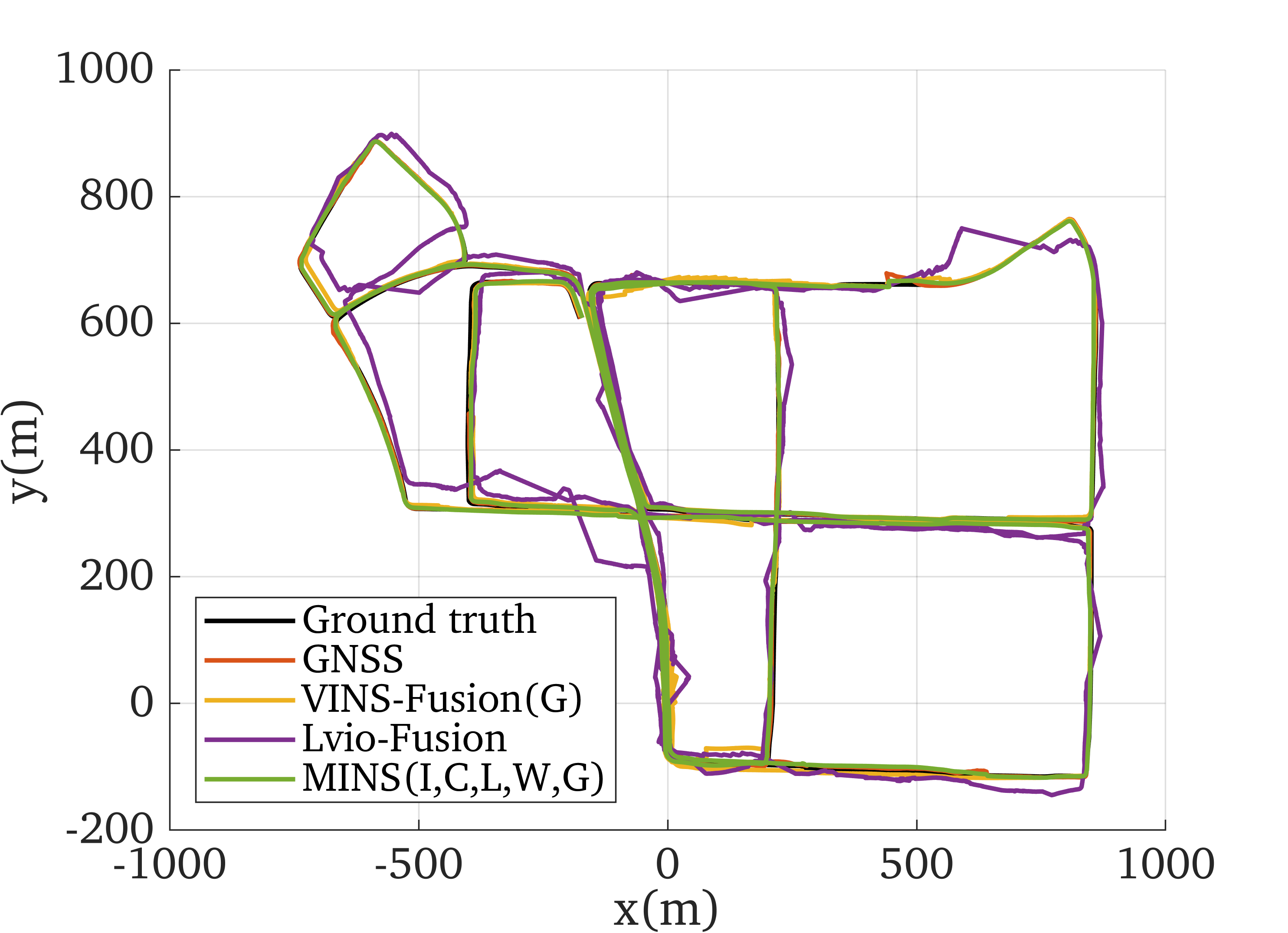}
\end{subfigure}
\caption{
Trajectories of each algorithm on the KAIST Urban 38 dataset. \textit{Left}: Camera-based algorithms; \textit{Middle}: LiDAR-based algorithms; \textit{Right}: GNSS-based algorithms.
}
\label{fig:kaist_traj}
\end{figure*}
\definecolor{Gray}{gray}{0.9}
\newcolumntype{g}{>{\columncolor{Gray}}c}
\begin{table*}[t]
\caption{Average (5 runs) position ATE (m) of each algorithm per 1 km on the KAIST Urban dataset (sequence 18 - 39). Above 30 m errors are not reported. Shaded sequences are highway scenarios and the others are city scenarios. The best results among the same sensor group are highlighted with bold text.}
\newcommand{\tmt}[2]{\multirow{#1}{*}{\rotatebox[origin=c]{90}{#2}}}
\begin{adjustbox}{width=1\textwidth,center}
\setlength{\tabcolsep}{3pt}
\begin{tabular}{ccggggggggcccccccccgggcc}
\toprule
\multicolumn{2}{c}{\multirow{2}{*}{\textbf{Algorithms}}}
&\textbf{18}&\textbf{19}&\textbf{20}&\textbf{21}&\textbf{22}&\textbf{23}&\textbf{24}&\textbf{25}&\textbf{26}&\textbf{27}&\textbf{28}&\textbf{29}&\textbf{30}&\textbf{31}&\textbf{32}&\textbf{33}&\textbf{34}&\textbf{35}&\textbf{36}&\textbf{37}&\textbf{38}&\textbf{39} \\
& & 4km & 3km & 3km & 4km & 3km & 3km & 4km & 3km & 4km & 5km & 11km  & 4km & 6km & 11km & 7km & 8km & 8km & 3km & 9km & 12km & 11km  & 11km  \\
\midrule
\tmt{5}{Camera}
& VINS-Fusion(V) & - & - & - & - & - & - & - & - & 8.05 & 10.34 & 1.54 & 10.03 & 24.83 & - & - & 11.81 & - & - & - & - & 2.23 & 2.11 \\
& VINS-Fusion(L) & - & - & - & - & - & - & - & - & 7.82 & 11.97 & 2.54 & 10.03 & 25.21 & - & - & 11.81 & - & - & - & - & 2.86 & 1.68 \\
& ORB-SLAM3 & 10.93 & - & 12.35 & - & - & 10.00 & - & - & 5.65 & 15.60 & 2.34 & 6.81 & 2.58 & \textbf{2.03} & \textbf{3.14} & 6.67 & - & - & - & 29.70 & 10.34 & 15.32 \\
& \textbf{MINS}(I,C) & - & - & - & - & - & - & - & - & 2.04 & \textbf{2.86} & 1.45 & 3.44 & 3.50 & 5.39 & 4.69 & 1.97 & \textbf{1.47} & - & - & - & 1.13 & \textbf{1.17}\\ 
& \textbf{MINS}(I,C,W) & \textbf{7.13} & \textbf{8.74} & \textbf{8.25} & \textbf{8.48} & \textbf{9.16} & \textbf{7.10} & \textbf{8.17} & \textbf{10.74} & \textbf{1.57} & 3.85 & \textbf{1.09} & \textbf{3.16} & \textbf{2.03} & 29.80 & 4.15 & 2.46 & 2.44 & \textbf{8.15} & \textbf{16.48} & \textbf{18.50} & \textbf{0.94} & 1.23\\  
\midrule 
\tmt{4}{LiDAR}
& FAST-LIO2 & - & - & - & - & - & - & - & - & 9.75 & 7.59 & - & \textbf{3.01} & \textbf{4.88} & - & 11.14 & 9.90 & - & - & - & - & 4.74 & \textbf{1.64}\\ 
& LIW-OAM & - & 13.47 & 24.60 & - & - & - & - & - & 3.59 & - & 6.42 & 5.84 & 20.49 & 15.60 & - & 20.39 & - & - & 18.70 & - & 6.86 & 25.99\\
& \textbf{MINS}(I,L) & - & - & - & - & - & - & - & - & 2.69 & \textbf{2.97} & \textbf{2.53} & 3.97 & 6.21 & 5.11 & 5.68 & 9.17 & \textbf{5.30} & - & - & - & 3.50 & 2.17\\ 
& \textbf{MINS}(I,L,W) & \textbf{6.96} & \textbf{8.11} & \textbf{4.21} & \textbf{2.99} & \textbf{6.91} & \textbf{4.25} & \textbf{6.85} & \textbf{3.76} & \textbf{2.24} & 4.15 & 2.71 & 4.40 & 5.47 & \textbf{5.01} & \textbf{4.72} & \textbf{7.92} & 7.86 & \textbf{6.13} & \textbf{12.01} & \textbf{7.96} & \textbf{3.24} & 2.50 \\ 
\midrule 
\tmt{4}{GNSS}
& GNSS & 2.16 & 4.13 & 1.53 & \textbf{0.64} & \textbf{0.53} & 2.61 & 2.61 & 4.10 & 1.74 & 1.90 & 0.87 & 3.07 & 2.59 & 0.70 & \textbf{0.69} & 1.21 & \textbf{0.47} & \textbf{2.46} & 2.98 & 0.67 & 0.68 & 0.67 \\ 
& VINS-Fusion(G) & - & - & - & - & - & - & - & - & 4.02 & - & 1.60 & 4.77 & 9.35 & 8.02 & - & 5.80 & - & - & - & - & 1.07 & 2.26 \\ 
& Lvio-Fusion & - & - & 27.43 & 20.77 & - & 24.93 & - & - & 2.70 & 5.20 & 1.32 & 13.27 & 6.64 & - & 2.55 & 2.37 & - & - & - & - & 5.57 & 5.70\\ 
& \textbf{MINS}(I,C,L,W,G) & \textbf{1.52} & \textbf{3.46} & \textbf{0.85} & 0.79 & 1.50 & \textbf{1.38} & \textbf{1.86} & \textbf{3.44} & \textbf{1.06} & \textbf{1.72} & \textbf{0.79} & \textbf{2.58} & \textbf{1.83} & \textbf{0.70} & 0.77 & \textbf{1.16} & 0.94 & 5.47 & \textbf{2.38} & \textbf{0.60} & \textbf{0.64} & \textbf{0.56} \\
\bottomrule 
\end{tabular}
\end{adjustbox}
\label{tab:kaist_ATE}
\end{table*}
\vspace{-0.3cm}
\begin{itemize}
    \item \textbf{VINS-Fusion} \citep{qin2019general}: Camera and IMU are fused based on the sliding window (VIO) and loop closure or GNSS information can be loosely coupled within pose graphs. The system supports dynamic initialization. We evaluated the VIO part VINS-Fusion(V), with loop closure VINS-Fusion(L), and with GPS VINS-Fusion(G). Note the algorithm assumes the GNSS and VIO are synchronized which is not the case in the dataset. We, therefore, allowed a maximum of 50 ms mismatch between VIO and GNSS.
\item \textbf{ORB-SLAM3} \citep{campos2021orb}: Camera and IMU are fused within a factor graph performing multimap-based local BA and loop closure. The system supports dynamic initialization.
    
    \item \textbf{FAST-LIO2} \citep{xu2022fast}: LiDAR and IMU are fused based on iterated Kalman filter. The system models LiDAR measurement as point-on-plane constraints building a global map using ikd-tree \citep{cai2021ikd}. The system supports static initialization.
    
\item \textbf{LIW-OAM} \citep{yuan2023liw}: LiDAR, IMU, and wheel are fused in local BA fashion using point-on-plane constraints as LiDAR measurement model. 
    The system builds a global voxel map and supports dynamic initialization using LiDAR and IMU. 

\item \textbf{Lvio-Fusion} \citep{jia2021lvio}: Camera, LiDAR, and IMU are tightly coupled within the factor graph while GNSS is coupled loosely within the pose graph. The initialization is conducted using the camera and IMU supporting dynamic initialization.
\end{itemize}
Note that all the above LiDAR-based methods can fuse only one LiDAR sensor which showed  poor performance on the KAIST dataset due to its limited overlapping points between the LiDAR scans. 
To make fair comparisons, we modified the KAIST dataset by transforming all the second LiDAR pointcloud to the first LiDAR's frame and created one synthetic 20 Hz LiDAR measurement while MINS used the original data.

\subsection{Localization Accuracy}
\begin{figure}[t]
\centering
\begin{subfigure}{.49\columnwidth}
\includegraphics[trim=0mm 0 0mm 0,clip,width=\linewidth]{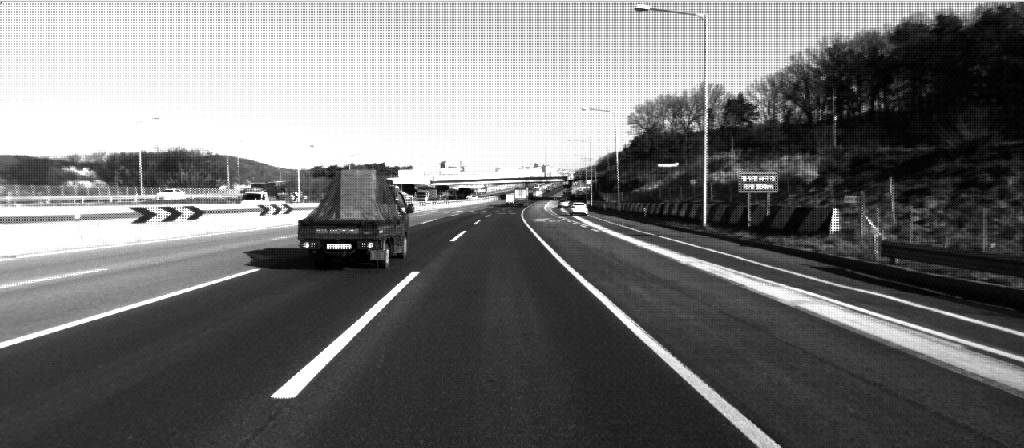}
\end{subfigure}
\begin{subfigure}{.49\columnwidth}
\includegraphics[trim=0mm 0 0mm 0,clip,width=\linewidth]{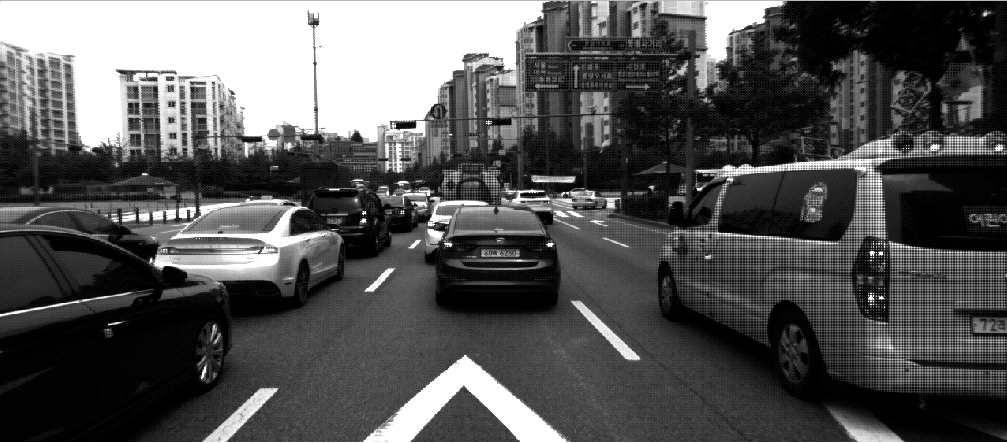}
\end{subfigure}
\caption{
Representative images from the datasets. The Left shows the beginning of sequence 18 which is on the highway and the right shows the city driving scene of sequence 38 where the cars stopped at the traffic light which will become dynamic objects of the camera and LiDAR when the light turns green.
}
\label{fig:kaist_camera}
\end{figure}

\subsubsection{KAIST Urban Dataset}

Table \ref{tab:kaist_ATE} shows the average position ATE of 5 runs of each algorithm on each sequence (above 30 m error not reported) and Fig. \ref{fig:kaist_traj} shows exemplary trajectories of all algorithms tested on the KAIST urban 38.
Note that we show the ATE values per 1 km to match the scale of different trajectory lengths.
The shaded sequences represent the datasets collected on highways while the vehicle was running at high speed and the other sequences are collected within the cities.

For the camera-based methods, the subsets of MINS using IMU and camera MINS(I,C) and IMU, camera, and wheel MINS(I,C,W) are evaluated along with SOTA. 
It is clear that all the methods that leverage only the camera and IMU performed poorly on highway datasets due to initialization failure (see Fig. \ref{fig:kaist_camera} left). 
As the datasets start with high speed and mostly straight-line motion, none of those methods was successfully initialized except ORB-SLAM3 on a few sequences with sill large errors.
On the other hand, MINS(I,C,W) leveraged wheel measurements to successfully initialize and perform accurate localization without failure.
For the city sequences, all VINS methods initialized successfully, but their performances were largely affected by the dynamic objects (see an exemplary case in Fig. \ref{fig:kaist_camera} right).
MINS handled the issue by performing all camera update with the MSCKF technique (see Eq.~\eqref{eq:cam_res_nullspace_prj}) which minimizes the effect of the dynamic objects due to shorter tracking period. 
To this end, the camera-based MINS modules are shown to outperform the other methods and are able to robustly handle hard scenarios.

For the LiDAR-based methods, the subsets of MINS using IMU and LiDAR MINS(I,L) and IMU, LiDAR, and wheel MINS(I,L,W) are evaluated along with SOTA. 
Similar to the camera-based methods, the initialization was the biggest issue on the highway sequences.
FAST-LIO2, MINS(I,L), and LIW-OAM mostly failed in initialization and quickly diverged while MINS(I,L,W) was able to initialize with wheel measurements and perform estimation consistently.
On city sequences, relatively large orientation drifts of both FAST-LIO2 and LIW-OAM (see Fig. \ref{fig:kaist_traj}) are observed when the vehicle is making turns.
The map density (0.3, 0.3, and 1.5 unit voxel size for FAST-LIO2, MINS, and LIW-OAM, respectively) and the lack of overlapping point clouds between the map and the new scan which creates an ill-constrained ICP problem are presumed to be the key reasons, while MINS was more robust to the scenarios even though all methods adopt the same point-on-plane measurement model.

For the GNSS-based methods, MINS using all the sensors MINS(I,C,L,W,G) is evaluated along with SOTA.
Note that we recorded the estimated poses of SOTA in real-time, not the final optimized trajectory at the end of the run.
We also showed the raw GNSS measurements for reference.
Overall, VINS-Fusion(G) and Lvio-Fusion were not able to show consistent estimation on both highway and city datasets even with GNSS global information due to their loosely coupling manner. 
This is because their sub-systems, VIO of VINS-Fusion and LIO of Lvio-Fusion suffered from initialization and dynamic objects that they kept injecting inconsistent odometry information to the pose graph making global estimation highly inconsistent and resulting in even worse estimation performance than just GNSS record.
On the other hand, MINS tightly coupled GNSS information with other sensors so that the state estimation was able to properly constrain its drift showing globally accurate and locally precise estimation performance.

\begin{figure}[t]
\centering
\begin{subfigure}{\columnwidth}
\includegraphics[trim=0mm 5mm 105mm 0,clip,width=\linewidth]{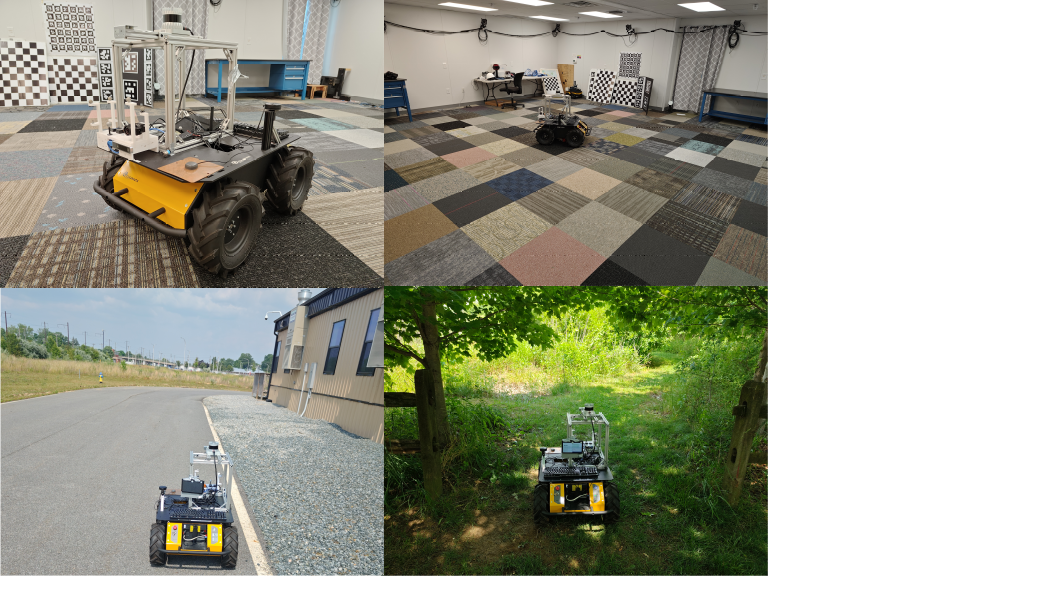}
\end{subfigure}
\caption{
Husky robot and exemplary scenes of UD Husky Dataset. \textit{Top-Left}: Husky robot with multiple sensors mounted. \textit{Top-Right}: Indoor dataset. \textit{Bottom-Left}: Outdoor dataset. \textit{Bottom-Right}: Trail dataset.
}
\label{fig:husky_dataset}
\end{figure}

\begin{figure}[t]
\centering
\begin{subfigure}{.67\columnwidth}
\includegraphics[trim=7mm 2mm 25mm 10mm,clip,width=\linewidth]{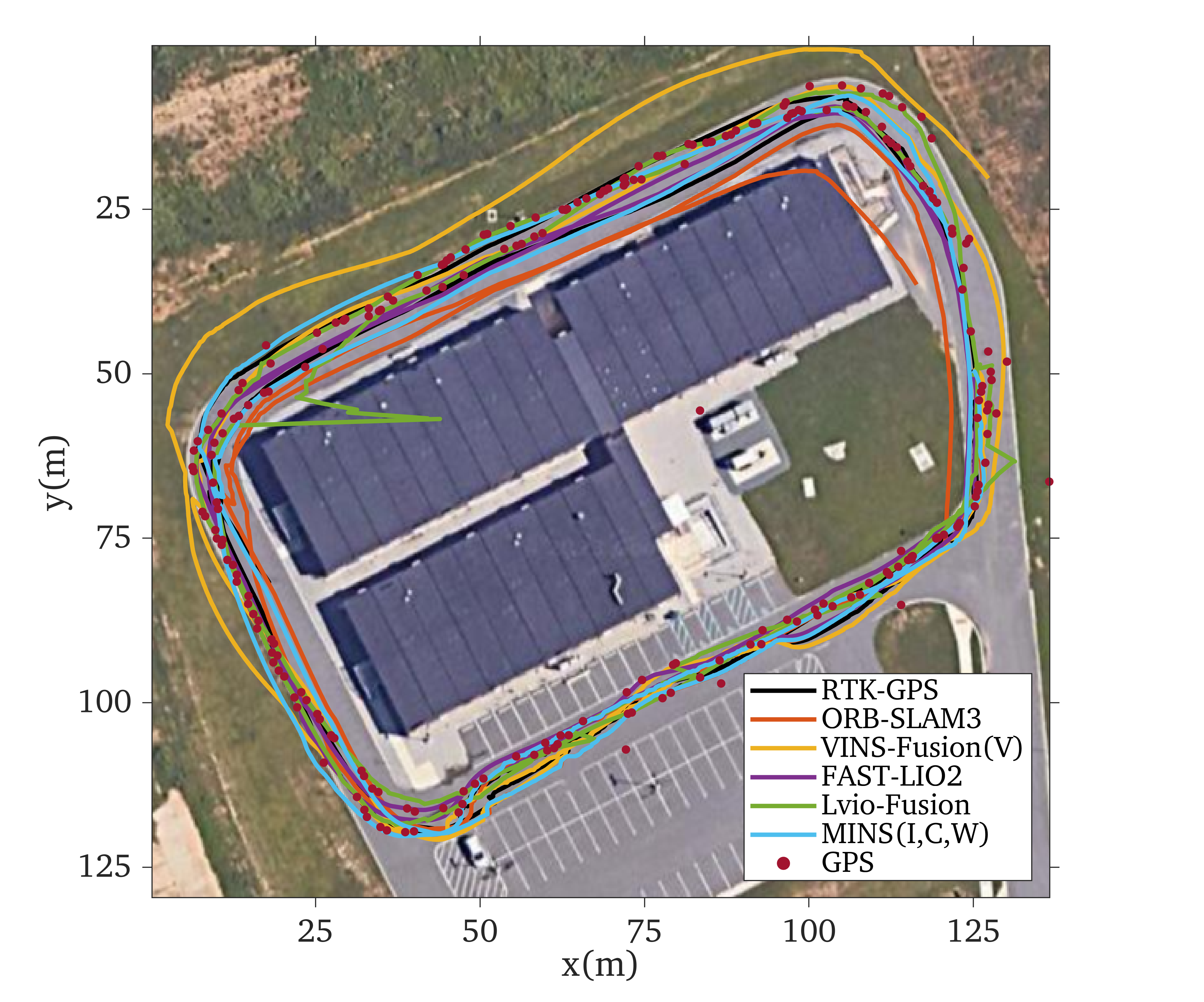}
\end{subfigure}
\begin{subfigure}{.32\columnwidth}
\includegraphics[trim=5mm 2mm 21mm 10mm,clip,width=\linewidth]{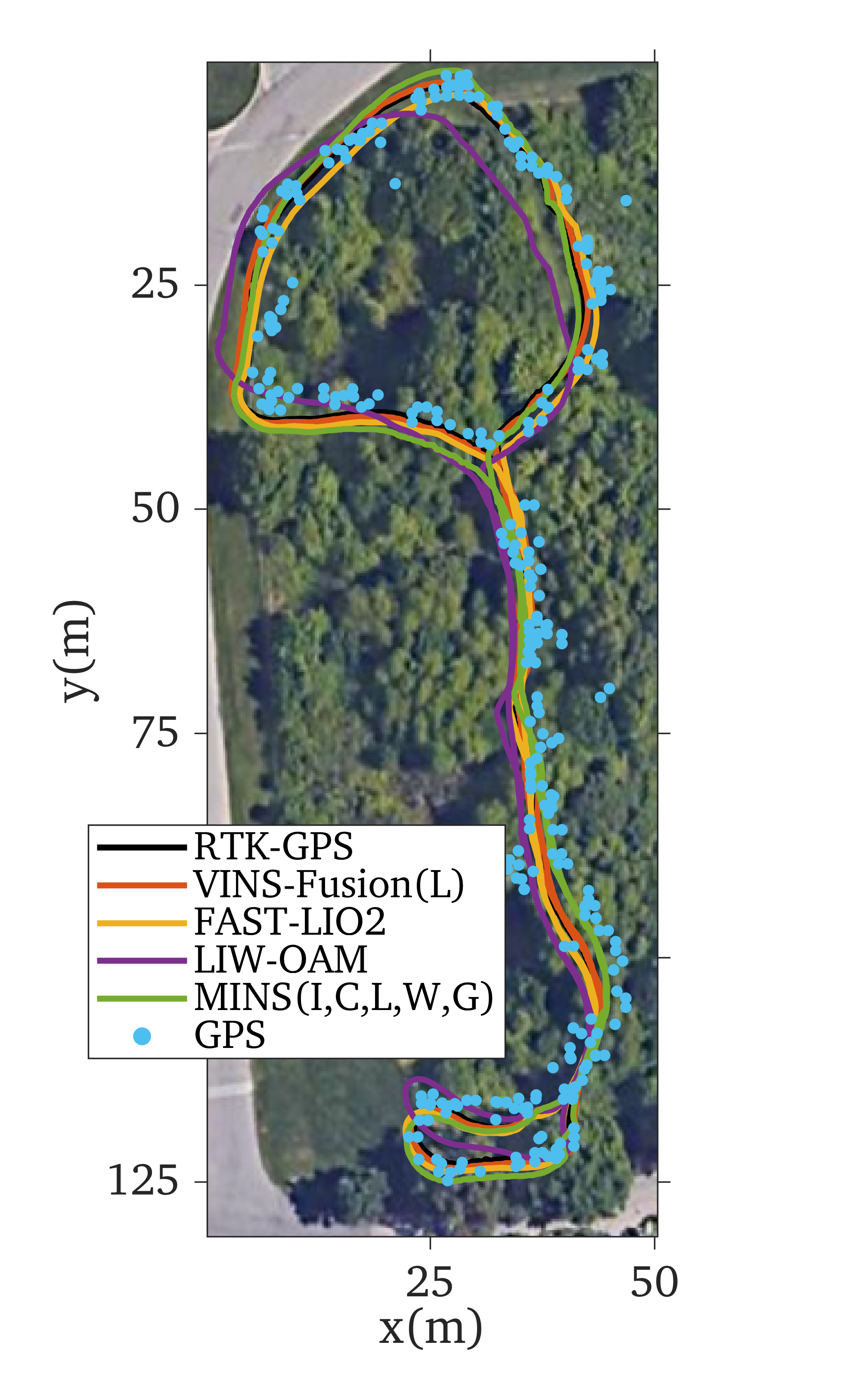}
\end{subfigure}
\caption{
Trajectories of each algorithm on UD Husky Dataset. \textit{left}: Outdoor 1; \textit{right}: Trail 3.
}
\label{fig:husky_example}
\end{figure}

\begin{table}[t]
\centering
\caption{Average (5 runs) position ATE (m) of each algorithm on the UD Husky dataset. Above 5 m errors are not reported (except GNSS for reference).}
\newcommand{\tmt}[2]{\multirow{#1}{*}{\rotatebox[origin=c]{90}{#2}}}
\begin{adjustbox}{width=1\columnwidth,center}
\begin{threeparttable}
\setlength{\tabcolsep}{3.3pt}
\begin{tabular}{cccccccccc}
\toprule
\multicolumn{2}{c}{\multirow{2}{*}{\textbf{Algorithms}}}
& \multicolumn{4}{c}{\textbf{Structured Env.}} 
& \multicolumn{4}{c}{\textbf{Unstructured Env.}}\\ 
& & \textbf{I1} & \textbf{I2} & \textbf{O1} & \textbf{O2} & \textbf{T1} & \textbf{T2} & \textbf{T3} & \textbf{T4} \\
\midrule
\tmt{5}{Camera}
& VINS-Fusion(V)        & 0.63  & 0.71  & -     & 4.39  & 4.82 & 4.92 & 1.60 & -\\
& VINS-Fusion(L)        & 0.55  & 0.49  & -  & 4.07  & 3.36 & 4.31 & \textbf{0.92} & 3.11\\
& ORB-SLAM3             & 0.50  & \textbf{0.48}  & 4.51  & -     & 2.50 & -    & 1.69 & \textbf{2.74} \\
& \textbf{MINS}(I,C)    & 0.33  & 0.73  & 2.93  & 3.92  & 3.16 & 3.19 & 2.64 & 4.66\\ 
& \textbf{MINS}(I,C,W)  & \textbf{0.22}  & 0.67  & \textbf{1.39}  & \textbf{2.57}  & \textbf{2.40} & \textbf{1.95} & 2.16 & 3.83\\  
\midrule 
\tmt{4}{LiDAR}
& FAST-LIO2             & 0.18 & 0.16 & 2.34 & 2.74 & \textbf{0.95} & 2.82    & 1.37 & 2.02 \\
& LIW-OAM               & 0.12 & 0.16 & \textbf{1.01} & 2.92 & 2.36 & 3.61 & 2.37 & 3.04\\
& \textbf{MINS}(I,L)    & \textbf{0.07} & \textbf{0.11} & 1.73 & 2.05 & 1.64 & 1.28 & 1.93 & 2.13\\
& \textbf{MINS}(I,L,W)  & 0.08 & 0.16 & 1.35 & \textbf{1.99} & 1.35 & \textbf{1.05} & \textbf{1.25} & \textbf{1.62}\\
\midrule 
\tmt{4}{GNSS}
& GNSS                      & & & 3.17  & 2.17 & 15.71  & 10.34 & 5.93 & 7.02 \\
& VINS-Fusion(G)            & & & -     & 3.25 & -      & -     & -    & -\\
& Lvio-Fusion               & & & 3.14  & 2.63 & - & - & - & 4.02\\
& \textbf{MINS}(I,C,L,W,G)  & & & \textbf{0.96}  & \textbf{1.07} & \textbf{3.03} & \textbf{2.89} & \textbf{1.33} & \textbf{1.30}\\
\bottomrule 
\end{tabular}
\begin{tablenotes} \footnotesize
\raggedleft \item[*] I: Indoor dataset, O: Outdoor dataset, T: Trail dataset
\end{tablenotes}
\end{threeparttable}
\end{adjustbox}
\label{tab:husky_ATE}
\end{table}

\subsubsection{UD Husky Dataset}

Table \ref{tab:husky_ATE} shows the average position ATE of 5 runs of the same set of algorithms on each UD Husky Dataset (above 5 m error not reported).
Note that all the sequences of the dataset start with the robot standing still enabling all the methods to successfully initialize.

For the camera-based methods, overall the algorithms that leverage only the camera and IMU reported poor localization performance showing large trajectory scale error (see Fig. \ref{fig:husky_example}).
This is because the ground robot mainly undergoes the degenerate motion \citep{wu2017vins} which makes the scale of VINS unobservable.
On the other hand, MINS(I,C,W) was able to constrain the scale from wheel information and showed higher accuracy.

The LiDAR-based methods were able to show smaller ATE under the same motion profiles as they directly gained the scale information from the 3D pointcloud, and MINS was able to outperform the others in most of the sequences.
An observation is that the wheel information did not improve the performance in the indoor dataset. 
The dataset contains many sharp-turning motions of the robot which our wheel model (see Eq. \eqref{eq:wheel_preint_function}) may not be able to accurately represent the actual robot motion, thus wheel information may harm the estimation performance especially when it is centimeter-level accuracy.

Compared to the GPS of the KAIST dataset, our GPS measurements were more intermittent and noisy as shown in the table which imposes a larger challenge in fusing the information.
The loosely-coupled methods, VINS-Fusion(G) and Lvio-Fusion, failed most of the sequences, due to their VIO drift or large GPS noise. 
On the other hand, MINS was able to show consistent estimation by tightly fusing all sensor information.
However, it is also shown that fusing all the sensors did not necessarily return the best localization performance. 
For example, in Trail 1 - 3 sequences, MINS(I,C,L,W,G) performed worse than both MINS(I,L,W) and MINS(I,C,W) which indicates a proper choice of sensors can lead to a better performance in certain scenarios.

\subsection{Computation Efficiency}

Table \ref{tab:time_benchmark} shows the timing breakdown (frontend, backend, and total) of each algorithm collected on ThinkPad P17 which has an Intel i7 CPU and 32 GB RAM.
The timings are reported as the average time taken per the function call on the KAIST Urban 38 dataset.

The frontend includes all the processes required before optimization or EKF update, such as IMU preintegration, image processing, or LiDAR pointcloud matching.
The camera is reported to be the one that took the most time to process. 
Two subsets of MINS, MINS(I,C) and MINS(I,C,L,W,G), reported different camera timing because we lowered the number of features extracted from images to balance the computation with LiDAR. 
Other frontend timings including IMU, GNSS, wheel, and LiDAR are shown to be very low for all the algorithms.

The backend includes the map management and the optimization (or EKF update) which shows large differences among the algorithms.
Those algorithms using cameras mostly use the map to find loop closures while those using LiDARs build the map to perform scan matching.
VINS-Fusion and ORB-SLAM3 perform the loop closure detection taking 70 ms and 10 ms on average, while MINS does not maintain a map of the camera for computational efficiency.
Both FAST-LIO2 and subsets of MINS utilized ikd-tree \citep{cai2021ikd} to efficiently manage the LiDAR map points, but MINS took more time because it keeps more pointclouds and transforms all the map points to a new anchor frame occasionally (see Sec.~\ref{ch:lidar_mapping}).
Other LiDAR methods recorded similar or larger map-management time.

The optimization (or EKF update) time also shows large variations among the algorithms. 
MINS and FAST-LIO2, which are the filters, showed the shortest time due to their small state size and the number of iterations.
Especially, MINS recorded 1 ms for the EKF update regardless of the number of sensors being used, due to the measurement compression and the dynamic cloning which reduced the state size and the computation.  
Other graph-based methods report one or two orders of magnitude larger optimization time due to their large state size (number of nodes). 
Note we reported the summation of the optimization time if an algorithm runs multiple threads that do optimization, e.g. VINS-Fusion(G) optimization time is computed by summing VINS-Fusion(V) optimization time and global graph optimization time.

Lastly, Table~\ref{tab:time_benchmark}  also shows the total time which can represent the amount of time expected if all processes run in serial. 
Clearly, MINS and FAST-LIO2 reported the smallest computation time due to the small optimization time.
While FAST-LIO2 was the fastest, MINS recorded the second fastest results and showed almost invariant computation time to the number of sensors being used.
Other graph-based methods show larger time records mostly determined by the optimization time.
Note that a large total time does not mean the algorithm cannot run in real-time (actually all the algorithms ran in real-time on the laptop) as the processes can be multi-threaded. 
However, their performance can be largely affected on embedded systems as their computation power is relatively limited.

\begin{table}
\centering
\caption{Timing records on the KAIST Urban 38 (\textit{ms}).}
\newcommand{\tmt}[2]{\multirow{#1}{*}{\rotatebox[origin=c]{90}{#2}}}
\begin{adjustbox}{width=1\columnwidth,center}
\setlength{\tabcolsep}{5.3pt}
\begin{tabular}{cccccccccc}
\toprule
\multicolumn{2}{c}{\multirow{2}{*}{\textbf{Algorithms}}}
& \multicolumn{5}{c}{\textbf{Frontend}} 
& \multicolumn{2}{c}{\textbf{Backend}}
& \multirow{2}{*}{\textbf{Total}}\\ 
& & \textbf{I} & \textbf{C} & \textbf{G} & \textbf{W} & \textbf{L} & \textbf{Map} & \textbf{Opt.} &  \\
\midrule
\tmt{5}{Camera}
& VINS-Fusion(V) & 0 & 27 & - & - & - & - & 68 & 95\\
& VINS-Fusion(L) & 0 & 27 & - & - & - & 70 & 125 & 222\\
& ORB-SLAM3 & 0 & 28 & - & - & - & 10 & 290 & 328 \\
& \textbf{MINS}(I,C) & 0 & 53 & - & - & - & - & 1 & 54\\ 
& \textbf{MINS}(I,C,W) & 0 & 55 & - & 0 & - & - & 1 & 56\\  
\midrule 
\tmt{4}{LiDAR}
& FAST-LIO2 & 0 & - & - & - & 1 & 5 & 14 & 20\\
& LIW-OAM & 0 & - & - & 0 & 7 & 40 & 48 & 95\\
& \textbf{MINS}(I,L) & 0 & - & - & - & 6 & 22 & 1 & 28\\
& \textbf{MINS}(I,L,W) & 0 & - & - & 0 & 7 & 23 & 1 & 31\\
\midrule 
\tmt{3}{GNSS}
& VINS-Fusion(G) & 0 & 27 & - & - & - & - & 457 & 484\\
& Lvio-Fusion & 0 & 18 & 0 & - & 2 & 20 & 106 & 146\\
& \textbf{MINS}(I,C,L,W,G) & 0 & 34 & 0 & 0 & 2 & 10 & 1 & 47\\
\bottomrule 
\end{tabular}
\end{adjustbox}
\label{tab:time_benchmark}
\end{table}

\section{Conclusions and Future Work}

In this paper, we have developed a robust Multi-sensor-aided Inertial Navigation System (MINS) that integrates an IMU, a pair of wheel encoders, and arbitrary numbers of cameras, LiDARs, and GNSS for robust and accurate state estimation. 
The proposed MINS addresses some of the key challenges of multisensor fusion by introducing consistent high-order on manifold state interpolation, dynamic cloning for managing state size and computation, online calibration of sensor parameters, and IMU-wheel combined initialization. 
The proposed approach has been validated through extensive evaluations conducted in realistic simulations and challenging real-world datasets,
showing  superior accuracy and consistency while maintaining lower computational complexity compared to the SOTA methods.
Our future work will explore the optimal weighting of multi-modal measurements, 
which plays a crucial role in multisensor fusion and even can harm the estimator performance if not set properly as shown in our experimental results.

\begin{acks}
This work was partially supported by the University of Delaware (UD) College of Engineering, 
the NSF (IIS-1924897), and the ARL (W911NF-19-2-0226).
\end{acks}

\bibliographystyle{libraries/SageH}
\bibliography{libraries/library, libraries/rpng}

\appendix
\section{Jacobians of Camera Measurements} \label{ch:apdx_cam}

Given the nested functions (see, Eq.~\eqref{eq:cam_meas_nested}), we can leverage the chain rule to find the full Jacobian matrix of the camera measurement model in respect to the state:
\begin{align*}
    \H{C} := \frac{\partial \mathbf{z}_{C}}{\partial \x{k}}
    =
    \underbrace{\frac{\partial \mathbf h_d  }{\partial \mathbf{z}_{n}}}_{\H{d1}}
    \underbrace{\frac{\partial \mathbf h_\rho  }{\partial \p{F}{C_k}}}_{\H{\rho}}
    \underbrace{\frac{\partial \mathbf h_t  }{\partial \x{k}}}_{\H{t}}
    +
    \underbrace{\frac{\partial \mathbf h_d  }{\partial \x{CI}}}_{\H{d2}}
\end{align*}
Note the second term of the above equation actually should be $\frac{\partial \mathbf h_d  }{\partial \x{CI}} \frac{\partial \x{CI}  }{\partial \x{k}}$ which the additional term $\frac{\partial \x{CI}}{\partial \x{k}}$ maps $\frac{\partial \mathbf h_d}{\partial \x{CI}}$ into a proper position in the full Jacobian.
However, we will omit the trivial mapping terms  for  brevity.
\subsection {Distortion} \label{ch:apdx_cam_intrinsic}
The Jacobian with respect to the normalized coordinates can be obtained as:
\begin{align*}
\H{d1} =&
\begin{bmatrix}
h_{d11} & h_{d12} \\
h_{d21} & h_{d22}
\end{bmatrix} \\
h_{d11} =& f_x ((1+k_1r^2+k_2r^4)+(2k_1x_n^2+4k_2x_n^2(x_n^2+y_n^2)) \\ &+2p_1y_n+(2p_2x_n+4p_2x_n))\\
h_{d12} =& f_x (2k_1x_ny_n+4k_2x_ny_n(x_n^2+y_n^2)+2p_1x_n+2p_2y_n)\\
h_{d21} =& f_y (2k_1x_ny_n+4k_2x_ny_n(x_n^2+y_n^2)+2p_1x_n+2p_2y_n)\\
h_{d22} =& f_y ((1+k_1r^2+k_2r^4)+(2k_1y_n^2+4k_2y_n^2(x_n^2+y_n^2))\\ &+(2p_1y_n+4p_1y_n)+2p_2x_n)
\end{align*}
and the Jacobian in respect to the intrinsic parameters:
\begin{align*}
    \H{d2} = 
\setlength\arraycolsep{1.2pt}
    \begin{bmatrix}
    x & 0 & 1 & 0 & f_x x_n r^2 & f_x x_n r^4 & 2f_x x_n y_n) & f_x (r^2+2x_n^2)  \\
    0 & y & 0 & 1 & f_y y_n r^2 & f_y y_n r^4 & f_y (r^2+2y_n^2) & 2f_y x_n y_n
    \end{bmatrix}
\end{align*}
where definition of each variable $x$, $y$, $ x_n$, $y_n$, $f_x$, $f_y$, $c_x$, $c_y$, $k_1$, $k_2$, $p_1$, $p_2$, and $r$ can be found in Sec.~\ref{ch:cam_distortion}.

\subsection {Perspective Projection}
The Jacobian matrix is defined as follows:
\begin{align*}
    \H\rho = 
\begin{bmatrix}
        \frac{1}{{}^{C_k}z} & 0 & \frac{-{}^{C_k}x}{({}^{C_k}z)^2} \\
        0 & \frac{1}{{}^{C_k}z} & \frac{-{}^{C_k}y}{({}^{C_k}z)^2}
    \end{bmatrix}
\end{align*}
where $\p{F}{C_k} = [{}^{C_k}x ~ {}^{C_k}y ~ {}^{C_k}z]^\top$.
\subsection{Euclidean Transformation}
The Jacobian matrix of Euclidean transformation can be represented using the chain rule to the camera pose (interpolated) and the camera feature:
\begin{align*}
    \H{t}
    =& 
    \begin{bmatrix}
        \frac{\partial \mathbf h_t }{\partial \ang{E}{C_k}} ~
        \frac{\partial \mathbf h_t }{\partial \p{C_k}{E}}
    \end{bmatrix}\HA{C,k} + \frac{\partial \mathbf h_t }{\partial \p{F}{E}}
    :=
    \H{e}\HA{C,k} + \H{F} \hspace{-85mm}\\
\frac{\partial \mathbf h_t }{\partial \ang{E}{C_k}} &= \skw{\p{F}{C_k}}, 
&&\frac{\partial \mathbf h_t }{\partial \p{C_k}{E}} = -\R{E}{C_k}, 
&&&\frac{\partial \mathbf h_t }{\partial \p{F}{E}} = \R{E}{C_k}
\end{align*}
where the definition of remaining chain $\HA{C,k}$, the camera pose to the interpolating IMU poses, and the spatiotemporal extrinsic calibration parameters, can be found in Appendix~\ref{ch:apdx_interpolation} (linear interpolation) or Appendix~\ref{ch:apdx_interpolation_high} (high-order interpolation).
The final form of the camera Jacobian matrix is:
\begin{align*}
    \H{C} = \H{d1}\H{\rho}\H{e}\HA{C,k} + \H{d1}\H{\rho}\H{F} + \H{d2}
\end{align*} \section{Wheel Odometry Measurements} \label{ch:apdx_wheel}

\setlength{\textfloatsep}{2pt}\begin{algorithm}
\caption{Wheel Odometry Measurement Update}
\begin{algorithmic}[1] 
    \Procedure{Wheel\_Update}{$\mathbf{x}_{k+1}$, $\{\omega_{ml}, \omega_{mr}\}_{k:k+1}$} 
        \CustomComment{Preintegrate measurement, Jacobian, and noise}
        \State $\mathbf{g}_O = \O{3\times1},~~ \mathbf{G}_O =  \O{3},~~ \mathbf{R}_O = \O{3}$
        \For{$\scalemath{1}{\omega_{ml,\tau},\omega_{mr,\tau} \in \{\omega_{ml}, \omega_{mr}\}_{k:k+1}}$} 
            \State $\mathbf{g}_O \xleftarrow[]{} \mathbf{g}_O + \Delta\mathbf{g}$
            \State $\mathbf{R}_O \xleftarrow[]{} \boldsymbol{\Phi}_{tr,\tau}\mathbf{R}_O \boldsymbol{\Phi}_{tr,\tau}^\top + \boldsymbol{\Phi}_{n,\tau} \mathbf{Q}_\tau\boldsymbol{\Phi}_{n,\tau}^\top$ 
            \State $\mathbf{G}_O \xleftarrow[]{} \boldsymbol{\Phi}_{tr,\tau} \mathbf{G}_O + \boldsymbol{\Phi}_{OI,\tau}$
        \EndFor
        \CustomComment{Compute residual and Jacobian}
        \State $\tilde{\boldsymbol{z}}_O = \mathbf{g}_O(\xhat{OI}) - \mathbf{h}_O(\xhat{k})$
        \State $\mathbf{H}_O = \frac{\partial \mathbf{h}_O}{\partial \xtilde{k}} - \mathbf{G}_O \frac{\partial \xtilde{OI}}{\partial \xtilde{k}}$
        \CustomComment{Perform $\scalemath{.8}{\chi^2}$ test \& update}
        \If{$\scalemath{.8}{\chi^2}(\tilde{\boldsymbol{z}}, \mathbf{H}_O, \mathbf{R}_O) == Pass$}
        \State $\mathrm{EKF\_Update}(\hat{\mathbf{x}}_{k+1}, \tilde{\boldsymbol{z}}_O, \mathbf{H}_O, \mathbf{R}_O)$
        \EndIf
    \EndProcedure
\end{algorithmic}\label{alg:whee_update}
\end{algorithm}

Here, we introduce the essential derivations of how the wheel measurement model (and its linearization) is formulated.
For the full derivations, please refer to our previous tech report \citep{Woosik2020wheelTR}.
The overall process is outlined in Algorithm \ref{alg:whee_update}.

\subsection{Wheel Odometry Preintegration}
Assume we are in the process of wheel odometry preintegration (see Eq.~\eqref{eq:wheel_preintegration}) and it is the turn to integrate measurements at $\t{\tau}{}$.
Based on the kinematic model (see Eq.~\eqref{eq:wheel_kinematic_model}) we perform following integration of measurements:
\begin{align}
\hspace{-4mm}
{}^{O_{\tau+1}}_{O_{k}}\theta 
&\approx
{}^{O_{\tau}}_{O_{k}}\theta \shortminus {}^{O_\tau}{\omega}\Delta t 
\hspace{-2mm}\label{eq:preint_ori}\\
\hspace{-4mm}
{}^{O_{k}}x_{O_{\tau+1}}
&\approx
{}^{O_{k}}x_{O_{\tau}} \shortminus {}^{O_\tau}v(\text{sin}({}^{O_{\tau+1}}_{O_{k}}\theta) \shortminus \text{sin}({}^{O_{\tau}}_{O_{k}}\theta)) / {}^{O_\tau}{\omega} 
\hspace{-2mm}\label{eq:preint_px}\\
\hspace{-4mm}
{}^{O_{k}}y_{O_{\tau+1}}
&\approx
{}^{O_{k}}y_{O_{\tau}} \shortminus {}^{O_\tau}v(\text{cos}({}^{O_{\tau+1}}_{O_{k}}\theta) \shortminus \text{cos}({}^{O_{\tau}}_{O_{k}}\theta)) / {}^{O_\tau}\omega 
\hspace{-2mm}\label{eq:preint_py}
\end{align}
where $\Delta t = t_{\tau+1} - t_{\tau}$. 
Note that we assume constant ${}^{O_\tau}\omega$ and ${}^{O_\tau}v$ (discrete sensor model) but considered the change of heading angle between $t_\tau$ and $t_{\tau+1}$ so that we have a more accurate model than assuming it constant. 
Repeating the process until $t_{k+1}$ finishes wheel odometry preintegration. 

\subsection{Jacobian of Wheel-Encoder Intrinsics}
As evident from Eq.~\eqref{eq:wheel_preint_function}, the wheel odometry integration entangles the intrinsic $\mathbf{x}_{OI}$, therefore the linearization would yield the following form:
\begin{align*}
\hspace{-2mm}
\mathbf z_{O} \simeq \mathbf{g}_O (\{\omega_{ml}, \omega_{mr}\}_{k:k+1}, \xhat{OI})
+ \scalemath{0.85}{\underbrace{\frac{\partial \mathbf{g}_O}{\partial \xtilde{OI}}}_{\mathbf{G}_O}} \xtilde{OI} + \scalemath{0.85}{\frac{\partial \mathbf{g}_O}{\partial \mathbf n_w}} \mathbf n_w 
\hspace{-2mm}
\end{align*}
where $\mathbf{n}_{\omega}$ is the stacked noise vector whose $\tau$-th block is corresponding to 
the encoder measurement noise at $t_\tau\in [t_k, t_{k+1}]$ (i.e., $[n_{\omega_{l,\tau}} ~ n_{\omega_{r,\tau}}]^\top$, see \eqref{eq:wheel_meas_raw}).

Clearly, performing EKF update with this measurement requires the Jacobians with respect to both the intrinsics and the noise.
It is important to note that as the preintegration of $\mathbf{g}_O(\cdot)$ is computed incrementally using the encoders' measurements in the interval $[t_k,t_{k+1}]$, we accordingly calculate the measurement Jacobians incrementally one step at a time.
Note also that since the noise Jacobian and $\mathbf{n}_{\omega}$ are often of high dimensions and may be computationally expensive when computing the stacked noise covariance during the update, 
we instead compute the noise covariance $\mathbf{R}_O$ by performing small matrix operations at each step.

Now we get the Jacobian of $t_\tau$ step integration from Eq.~\eqref{eq:preint_ori}, \eqref{eq:preint_px}, and \eqref{eq:preint_py} (again, the full derivations, which this margin is too narrow to contain, can be found in our previous tech report \citep{Woosik2020wheelTR}, thus here we only show the structure of them):
\begin{align*}
{}^{O_{\tau+1}}_{O_{k}}\tilde{\theta}  
&=
{}^{O_{\tau}}_{O_{k}}\tilde{\theta} + \boldsymbol{\Phi}_{1}\xtilde{OI} + \boldsymbol{\Phi}_{2}\mathbf{n}_{\omega,\tau} \\
{}^{O_{k}}\tilde{x}_{O_{\tau+1}}
&=
{}^{O_{k}}\tilde{x}_{O_{\tau}} + \boldsymbol{\Phi}_{3}{}^{O_{\tau}}_{O_{k}}\tilde{\theta} + \boldsymbol{\Phi}_{4}\xtilde{OI} + \boldsymbol{\Phi}_{5} \mathbf{n}_{\omega,\tau}\\
{}^{O_{k}}\tilde{y}_{O_{\tau+1}}
&=
{}^{O_{k}}\tilde{y}_{O_{\tau}} + \boldsymbol{\Phi}_{6}{}^{O_{\tau}}_{O_{k}}\tilde{\theta} + \boldsymbol{\Phi}_{7}\xtilde{OI} +\boldsymbol{\Phi}_{8}\mathbf{n}_{\omega,\tau} 
\end{align*}
It can be found that the error of $\tau+1$ step preintegration is the linear combination of $\tau$ step preintegration and measurement errors.
With the above equations, we can recursively compute the  noise covariance $\mathbf{R}_{O}$ and the Jacobian $\mathbf{G}_O$ as follows:
\begin{align*}
    \boldsymbol{\Phi}_{tr,\tau} &= \begin{bmatrix} 1 & 0 & 0 \\ \boldsymbol{\Phi}_{3} & 1 & 0 \\ \boldsymbol{\Phi}_{6} & 0 & 1\end{bmatrix},~
    \boldsymbol{\Phi}_{OI,\tau} = \begin{bmatrix} \boldsymbol{\Phi}_{1} \\ \boldsymbol{\Phi}_{4} \\ \boldsymbol{\Phi}_{7}\end{bmatrix},~
    \boldsymbol{\Phi}_{n,\tau} = \begin{bmatrix} \boldsymbol{\Phi}_{2} \\ \boldsymbol{\Phi}_{5} \\ \boldsymbol{\Phi}_{8}\end{bmatrix}\\
\mathbf{R}_{O} &\xleftarrow[]{} \boldsymbol{\Phi}_{tr,\tau}\mathbf{R}_{O} \boldsymbol{\Phi}_{tr,\tau}^\top + \boldsymbol{\Phi}_{n,\tau} \mathbf{Q}_\tau\boldsymbol{\Phi}_{n,\tau}^\top\\
\mathbf{G}_O &\xleftarrow[]{} \boldsymbol{\Phi}_{tr,\tau} \mathbf{G}_O + \boldsymbol{\Phi}_{OI,\tau}
\end{align*}
where $\mathbf{Q}_\tau$ is the noise covariance of wheel encoder measurement at $t_\tau$.
We can recursively compute the measurement noise covariance $\mathbf{R}_{O}$ and the Jacobian matrix $\frac{\partial \mathbf{g}_O}{\partial \xtilde{OI}}$ at the end of preintegration $t_{k+1}$,
based on the zero initial condition (i.e., $\mathbf{R}_{O} = \mathbf{G}_O = \mathbf{0}_3$).
The final structure of $\mathbf{G}_O$ can be shown as:
\begin{align*}
    \mathbf{G}_O
    &= 
    \begin{bmatrix}
    \mathbf{\Gamma}_{\boldsymbol{\theta}1} & \mathbf{\Gamma}_{\boldsymbol{\theta}2} & \mathbf{\Gamma}_{\boldsymbol{\theta}3}\\
    \mathbf{\Gamma}_{\mathbf{x}1} & \mathbf{\Gamma}_{\mathbf{x}2} & \mathbf{\Gamma}_{\mathbf{x}3}\\
    \mathbf{\Gamma}_{\mathbf{y}1} & \mathbf{\Gamma}_{\mathbf{y}2} & \mathbf{\Gamma}_{\mathbf{y}3}
    \end{bmatrix} 
\end{align*}
Assuming $n$ number of measurements are integrated, the structure of the element can be shown as:
\begin{align*}
    \mathbf{\Gamma}_{\boldsymbol{\theta}1}&= \sum^{n}_{i=1}-\Delta t_i\frac{\omega_{l,i}}{b}\\
\mathbf{\Gamma}_{\boldsymbol{\theta}2}&= \sum^{n}_{i=1}\Delta t_i\frac{\omega_{r,i}}{b}\\
\mathbf{\Gamma}_{\boldsymbol{\theta}3}&= \sum^{n}_{i=1}- \Delta t_i\frac{{}^{O_i}\omega}{b}\\
\mathbf{\Gamma}_{\mathbf{x}1}&=\sum_{i = 1}^{n} \bigg \{\omega_{l,i} \bigg (\frac{h_{xv,i}}{2} - \frac{h_{x\omega,i}}{b} \bigg) - h_{x\theta,i}\sum^{i-1}_{j=1}\Delta t_j\frac{\omega_{l,j}}{b} \bigg \}\\
\mathbf{\Gamma}_{\mathbf{x}2}&=\sum_{i = 1}^{n} \bigg \{\omega_{r,i} \bigg (\frac{h_{xv,i}}{2} + \frac{h_{x\omega,i}}{b} \bigg) +h_{x\theta,i}\sum^{i-1}_{j=1}\Delta t_j \frac{\omega_{r,j}}{b} \bigg \}\\
\mathbf{\Gamma}_{\mathbf{x}3}&=\sum_{i = 1}^{n} \bigg \{-h_{x\omega,i}\frac{{}^{O_i}\omega}{b}-h_{x\theta,i}\sum^{i-1}_{j=1}\Delta t_j\frac{{}^{O_j}\omega}{b}\bigg \}\\
\mathbf{\Gamma}_{\mathbf{y}1}&=\sum_{i = 1}^{n} \bigg \{\omega_{l,i} \bigg (\frac{h_{yv,i}}{2} - \frac{h_{y\omega,i}}{b} \bigg) - h_{y\theta,i}\sum^{i-1}_{j=1}\Delta t_j\frac{\omega_{l,j}}{b} \bigg \}\\
\mathbf{\Gamma}_{\mathbf{y}2}&=\sum_{i = 1}^{n} \bigg \{\omega_{r,i} \bigg (\frac{h_{yv,i}}{2} + \frac{h_{y\omega,i}}{b} \bigg) +h_{y\theta,i}\sum^{i-1}_{j=1}\Delta t_j \frac{\omega_{r,j}}{b} \bigg \}\\
\mathbf{\Gamma}_{\mathbf{y}3}&=\sum_{i = 1}^{n} \bigg \{-h_{y\omega,i}\frac{{}^{O_i}\omega}{b}-h_{y\theta,i}\sum^{i-1}_{j=1}\Delta t_j\frac{{}^{O_j}\omega}{b}\bigg \}
\end{align*}
Note this derivation cannot be applied in computing covariance matrix $\mathbf{R}_O$, because the noise $\mathbf{n}_{\omega,\tau}$ has different values for every iteration unlike $\x{OI}$.

\subsection{Full Jacobian of Wheel Odometry} \label{ch:apdx_wheel_intrinsic}
First, we show the Jacobian of wheel measurement model (see Eq.~\eqref{eq:wheel_update_model}) in respect to the state:
\begin{align*}
    \frac{\partial \mathbf h_O }{\partial \xtilde{k}} &= (\H{O1} \HS{O,k} + \H{O2} \HS{O,k \shortminus 1})\xtilde{k} \hspace{-35mm}\\
\H{O1} &= 
    \setlength\arraycolsep{2pt}
    \begin{bmatrix}
        \mathbf{e}_3^\top &  \O{1 \times 3} \\
\O{2 \times 3} & \Lambda\R{E}{O_{k \shortminus 1}}
    \end{bmatrix}, 
\H{O2} = 
    \begin{bmatrix}
        \shortminus \mathbf{e}_3^\top \R{O_{k \shortminus 1}}{O_{k}} & \O{1 \times 3} \\
\Lambda\skw{\p{O_{k}}{O_{k \shortminus 1}}} & \shortminus \Lambda\R{E}{O_{k \shortminus 1}}
    \end{bmatrix}
\end{align*}
where $\H{O1}$ and $\H{O2}$ are the Jacobian matrices of $\mathbf{h}(\cdot)$ in respect to $\{O_k\}$ and $\{O_{k \shortminus 1}\}$; $\HS{O,k}$ is the Jacobian matrix of $\{O_k\}$ in respect to the $\{I_k\}$ and the wheel spatiotemporal extrinsic calibration parameters (see Appendix~\ref{ch:apdx_sync_jacobian} for the definition).

Finally, we construct the full Jacobian matrix of wheel:
\begin{align*}
\hspace{-2mm}
\scalemath{0.95}{
    \mathbf{H}_O = \frac{\partial \mathbf{h}_O}{\partial \xtilde{k}} \shortminus \frac{\partial \mathbf{g}_O}{\partial \xtilde{OI}}
= \H{O1} \HS{O,k} + \H{O2} \HS{O,k \shortminus 1} \shortminus \mathbf{G}_O \frac{\partial \xtilde{OI}}{\partial \xtilde{k}}
\hspace{-2mm}
}
\end{align*} 
Note that $\mathbf{G}_O $ has the minus sign in front of it because $\mathbf{g}_O$ is on the left hand side of the equation while $\mathbf{h}_O$ is on the right. \section{Jacobians of LiDAR Measurements} \label{ch:apdx_lidar}
Here we drive the Jacobian matrix of the LiDAR measurement model (see Eq.~\eqref{eq:lidar_measurement_model}).
First, we linearize the model in respect to each map point $\p{n}{M}$, the new point $\p{F}{M}$ (transformed from $\{L_k\}$ to $\{M\}$), and the plane $\cp{}{M}$:
\begin{align*}
    \tilde{\mathbf{z}}_L
&\approx
    \begin{bmatrix}
        \O{1\times3}\\
        \vdots\\
        \O{1\times3}\\
        \cp{}{M}^\top
    \end{bmatrix} \ptilde{F}{M}
    +
    \begin{bmatrix}
        \p{n_1}{M}^\top\\
        \vdots\\
        \p{n_m}{M}^\top\\
        \p{F}{M}^\top
    \end{bmatrix}
     {}^{M}\tilde{\boldsymbol{\Pi}}
    +
    \begin{bmatrix}
        \cp{}{M}^\top\mathbf{n}_{n_1}\\
        \vdots\\
        \cp{}{M}^\top\mathbf{n}_{n_m}\\
        0
    \end{bmatrix}
\end{align*}
where $\mathbf{n}_{n}$ is zero mean Gaussian noise model of the map points.
To complete the chain of the Jacobian matrix, we compute the Jacobian of Eq.~\eqref{eq:lidar_tr} (here we show the case where the map $\{M\}$ is anchored at $\{I_{k \shortminus 1}\}$, thus, $\{M\} = \{L_{k \shortminus 1}\}$ and $\p{F}{M} = \p{F}{L_{k \shortminus 1}}$):
\begin{align*}
    \ptilde{F}{M} &= 
\underbrace{
    \begin{bmatrix}
        \skw{\p{F}{L_{k - 1}}} & - \R{E}{L_{k - 1}}
    \end{bmatrix}
}_{\mathbf{H}_{L1}}
    \begin{bmatrix}
        \angtilde{E}{L_{k \shortminus 1}} \\ \ptilde{L_{k \shortminus 1}}{E}
    \end{bmatrix}
    \\
    &+
\underbrace{
    \begin{bmatrix}
        \R{L_k}{L_{k \shortminus 1}}\skw{\p{F}{L_{k}}} & - \R{E}{L_{k}}
    \end{bmatrix}
}_{\mathbf{H}_{L2}}
    \begin{bmatrix}
        \angtilde{E}{L_{k}} \\ \ptilde{L_{k}}{E}
    \end{bmatrix}
    +
    \R{L_{k}}{L_{k \shortminus 1}} \mathbf{n}_F \\
&= (\mathbf{H}_{L1} \HA{L,k \shortminus 1} + \mathbf{H}_{L2} \HA{L,k}) \xtilde{k} + \R{L_{k}}{L_{k \shortminus 1}} \mathbf{n}_F
\end{align*}
where $\mathbf{n}_F$ is the zero mean Gaussian noise of the measurement.
Note $\HA{L,k \shortminus 1}$ and $\HA{L,k}$ are the Jacobians of LiDAR poses at $t_{k-1}$ (actually the map) and $t_{k}$ in respect to the state $\xtilde{k}$ where the definitions can be found in Appendix~\ref{ch:apdx_interpolation} (linear interpolation) or Appendix~\ref{ch:apdx_interpolation_high} (high-order interpolation).

Finally, we get the following linear system formulation that is a function of the state, the plane, and the noises:
\begin{align*}
\tilde{\mathbf{z}}_L
    =&
    \underbrace{
    \begin{bmatrix}
        \O{1\times3}\\
        \vdots\\
        \O{1\times3}\\
        \scalemath{0.9}{{}^M\boldsymbol{\Pi}^\top (\mathbf{H}_{L1} \HA{L,k \shortminus 1} + \mathbf{H}_{L2} \HA{L,k})}
    \end{bmatrix}}_{\H{L}} \xtilde{k}
    \\&+
    \underbrace{
    \begin{bmatrix}
        \p{n_1}{M}^\top\\
        \vdots\\
        \p{n_k}{M}^\top\\
        \p{F}{M}^\top
    \end{bmatrix}}_{\H{\boldsymbol{\Pi}}} {}^{M}\tilde{\boldsymbol{\Pi}}
    +
    \underbrace{
    \begin{bmatrix}
        {}^{M}\boldsymbol{\Pi}^\top\mathbf{n}_{n_1}\\
        \vdots\\
        {}^{M}\boldsymbol{\Pi}^\top\mathbf{n}_{n_k}\\
        {}^{M}\boldsymbol{\Pi}^\top \R{L_k}{L_{k-1}}\mathbf{n}_{F}
    \end{bmatrix}}_{\mathbf{n}_{L}}
\end{align*}
Though the above linear system can be directly used to update the state, we additionally perform Cholesky decomposition of the last noise term (in the covariance matrix form $\mathbf{R}_L = \mathbf{L}\mathbf{L}^\top$) and multiply the inverse of lower triangular matrix ($\mathbf{L}^{-1}$) to the above equation
(a.k.a whitening)
to make the following null space projection (see Eq.~\eqref{eq:lidar_res_nullspace_prj}) simple and construct the noise covariance matrix efficiently without tracking the map and measurement noises.
To be more specific, we get the following after the operation (corresponds to the Eq.~\eqref{eq:lidar_linsys}):
\begin{align*}
    \tilde{\mathbf{z}}'_L = \H{L}'\xtilde{k} + \H{\boldsymbol{\Pi}}'{}^M\tilde{\boldsymbol{\Pi}} + \mathbf{n}_{L}'
\end{align*}
Note the noise $\mathbf{n}_{L}'$ is the standard Gaussian ($\mathcal{N}(\mathbf{0}, \mathbf{I})$). \section{GNSS Measurements}

\subsection{Jacobians} \label{ch:apdx_gps}

The Jacobian matrix of GNSS measurement model (see Eq.~\eqref{eq:gps_measurement}) can be simply represented with two matrices (after initialization):
\begin{align}
    \H{G} := \frac{\partial \mathbf{z}_{G}}{\partial \x{k}} =  [\O{3} ~~ \I{3}] \HA{G,k} =\H{g} \HA{G,k} \label{eq:jacob_gnss}
\end{align}
where $\H{g}$ is the derivative of the measurement with respect to the GNSS sensor pose, and the definition of remaining chain $\HA{G,k}$, the GNSS sensor pose to the interpolating IMU poses and the spatiotemporal extrinsic calibration parameters, can be found in Appendix~\ref{ch:apdx_interpolation} (linear interpolation) or Appendix~\ref{ch:apdx_interpolation_high} (high-order interpolation).
\subsection[Observability Analysis: The State in Local]{Observability Analysis: State in $\{W\}$} \label{ch:apdx_gps_obs_local}
For concise presentation, here we consider a simplified case where the state in the local world ${}^{W}\x{k}$ only contains one IMU pose, linear velocity, and $\{W\}$ to $\{E\}$ transformation with perfectly synchronized and calibrated sensors (identity transformation from IMU to GNSS), while the results can be extended to general cases:
\begin{align} 
        {}^{W}\x{k} = 
        (
            \R{W}{I_k}, ~
            \p{I_k}{W}, ~
            \vel{I_k}{W}, ~
            \R{W}{E}, ~
            \p{W}{E}
        ) \label{eq:state_in_local}
\end{align}
The corresponding error state transition matrix can be shown as ($\Delta t = t_{k}-t_0$):
\begin{align*}
    &\mathbf{\Phi}^{W}{(t_{k},t_{0})} =\\
    &
\setlength\arraycolsep{2pt}
    \begin{bmatrix}
        \R{W}{I_k} \R{I_0}{W} & \O{} & \O{} & \O{}\\
-\skw{ \p{I_k}{W} - \p{I_0}{W} - \vel{I_0}{W} \Delta t + \frac{1}{2}\mathbf{g}\Delta t^2} \R{I_0}{W}  & \I{3} & \Delta t \I{3} & \O{} \\
-\skw{\vel{I_k}{W} - \vel{I_0}{W} + \mathbf{g}\Delta t}\R{I_0}{W}  & \O{} & \I{3} & \O{} \\
\O{} & \O{} & \O{} & \I{6}
    \end{bmatrix} 
\end{align*}
The measurement model and the corresponding Jacobian matrix in respect to the local state (see Eq.~\eqref{eq:state_in_local}) are:
\begin{align*}
\mathbf{z}_{G} &= \p{W}{E} + \R{W}{E}\p{I_k}{W} + \mathbf{n}_{G} \\
    \H{G}^W &=
\begin{bmatrix}
        \O{3} & \R{W}{E} & \O{3} & \O{3} & \skw{\R{W}{E}\p{I_k}{W}} & \I{3}
    \end{bmatrix}
\end{align*}
Now we can construct the observability matrix $\mathcal{O}^W$:
\begin{align*}
    \mathcal{O}^W = \begin{bmatrix} 
    \H{G_0}^W \\ 
    \H{G_1}^W\mathbf{\Phi}^{W}{(t_{1},t_{0})} \\
    \vdots \\ 
    \H{G}^W\mathbf{\Phi}^{W}{(t_{k},t_{0})} \\ 
    \vdots \end{bmatrix}
\end{align*}
The $(k+1)$th row matrix of $\mathcal{O}^W$ can be shown as:
\begin{align*}
\mathcal{O}^W_{k+1} &= \H{G}^W\mathbf{\Phi}^{W}{(t_{k},t_{0})}
     = 
     \begin{bmatrix}
     \boldsymbol{\Gamma}_1 &
     \boldsymbol{\Gamma}_2 &
     \boldsymbol{\Gamma}_3 & 
     \boldsymbol{\Gamma}_4  & 
     \I{3}
     \end{bmatrix}\\
\boldsymbol{\Gamma}_1 &= - \R{W}{E} \skw{ \p{I_k}{W} - \p{I_0}{W} - \vel{I_0}{W} \Delta t + \frac{1}{2}\mathbf{g}\Delta t^2 } \R{I_0}{W}\\
\boldsymbol{\Gamma}_2 &= \R{W}{E}\\
\boldsymbol{\Gamma}_3 &= \R{W}{E} \Delta t \\
\boldsymbol{\Gamma}_4 &=  \skw{ \R{W}{E} \p{I_k}{W} }
\end{align*}
It is easily can be shown that the matrix shown below is the null space of $\mathcal{O}^W_{k+1}$ by calculating $\mathcal{O}^W_{k+1} \times \mathbf{N}^W = \mathbf{0}$.
\begin{align*}
    \mathbf{N}^W
    =
    \begin{bmatrix}
        \O{3} & - \R{W}{I_0}\mathbf{g}\\
\I{3} & \skw{\p{I_0}{W}} \mathbf{g}\\
\O{3} & \skw{\vel{I_0}{W}} \mathbf{g}\\
\O{3}  & -\R{W}{E} \mathbf{g}\\
- \R{W}{E} & \mathbf{0}_{3\times1}
    \end{bmatrix}
\end{align*}
As $\mathbf{N}^W$ is the null space of any row matrices of $\mathcal{O}^W$, $\mathbf{N}^W$ is the null space of the local state (see Eq.~\eqref{eq:state_in_local}).
By inspection, the first column of $\mathbf{N}^W$ corresponds to the translation of $\{W\}$ to $\{E\}$ and the second column to the rotation of $\{W\}$ with respect to $\{E\}$ along the axis of gravity.

\subsection[Observability Analysis: The State in Global]{Observability Analysis: State in $\{E\}$}  \label{ch:apdx_gps_obs_global}
In analogy to Appendix~\ref{ch:apdx_gps_obs_local}, here we also investigate the minimal state that has the frame of reference in $\{E\}$ as:
\begin{align} 
        {}^{E}\x{k} = 
        (
            \R{E}{I_k}, ~
            \p{I_k}{E}, ~
            \vel{I_k}{E}, ~
            \R{W}{E}, ~
            \p{W}{E}
        ) \label{eq:state_in_global}
\end{align}
The corresponding error state transition matrix is:
\begin{align*}
    &\mathbf{\Phi}^{E}{(t_{k},t_{0})} =\\
    &
\setlength\arraycolsep{2pt}
    \begin{bmatrix}
        \R{E}{I_k} \R{I_0}{E} & \O{} & \O{} & \O{}\\
-\skw{ \p{I_k}{E} - \p{I_0}{E} - \vel{I_0}{E} \Delta t + \frac{1}{2}\mathbf{g}\Delta t^2} \R{I_0}{E}  & \I{3} & \Delta t \I{3} & \O{} \\
-\skw{\vel{I_k}{E} - \vel{I_0}{E} + \mathbf{g}\Delta t}\R{I_0}{E}  & \O{} & \I{3} & \O{} \\
\O{} & \O{} & \O{} & \I{6}
    \end{bmatrix} 
\end{align*}
The measurement model and the corresponding Jacobian matrix in respect to the local state (see Eq.~\eqref{eq:state_in_global}) are:
\begin{align*}
\mathbf{z}_{G} &= \p{I_k}{E} + \mathbf{n}_{G} \\
    \H{G}^E &=
\begin{bmatrix}
        \O{3} & \I{3} & \mathbf{0}_{3 \times 9}
    \end{bmatrix}
\end{align*}
Clearly, the multiplication of $\mathcal{O}^E_{k+1} := \H{G}^E\mathbf{\Phi}^E{(t_{k},t_{0})}$ with $\mathbf{N}^W$ does not yield a zero matrix which means the 4 unobservable directions of INS \citep{kelly2011visual} are now observable.

\subsection[State Transformation from Local to Global]{State Transformation from $\{W\}$ to $\{E\}$} \label{ch:apdx_gps_state_transform}
As the IMU biases and the calibration parameters remain the same (their frame of the references are fixed to, such as, IMU and sensor frame), thus have identity matrix in this Jacobian matrix of transformation, here we show the minimal relevant state case:
\begin{align*}
    {}^W\x{k} = (\R{W}{I_k}, \p{I_k}{W}, \vel{I_k}{W}, \R{W}{E}, \p{W}{E})
\end{align*}
We linearize the state transform function (see Eq.~\eqref{eq:trans_state_func}) at current estimate to achieve the Jacobian matrix $\bm\Psi$ and propagate the error state with it as:
\begin{align*}
    {}^E\tilde{\mathbf{x}}_k \xleftarrow[]{} \bm\Psi{}^W\tilde{\mathbf{x}}_k
\end{align*}
where
\begin{align*}
    \bm\Psi =
    \begin{bmatrix}
    \I{3} & \O{3} & \O{3} &  {}^{I_k}_W\mathbf{R}\R{W}{E}^\top  & \O{3}\\
    \O{3} & \R{W}{E} & \O{3} & \skw{\R{W}{E}\p{I_k}{W}} & \I{3}\\
    \O{3} & \O{3} & \R{W}{E} & \skw{\R{W}{E}\vel{I_k}{W}}  & \O{3}\\
    \O{3} & \O{3} & \O{3} & \I{3} & \O{3}\\
    \O{3} & \O{3} & \O{3} & \O{3} & \I{3}
    \end{bmatrix}
\end{align*}
Please refer to our previous tech report \citep{Woosik2019TR} for the full derivations. \section{Jacobians of Synchronous Sensors} \label{ch:apdx_sync_jacobian}

In general, there is no reason a sensor measurement is synchronous to other sensors or to the state.
However, in some cases, we can chose the sampling time of the measurements to get the synchronized sensor measurements.
The preintegrated wheel measurement (see Eq.~\eqref{eq:wheel_preint_function}), for example, is one of the case where we decide the measurement integration period and get the synchronized pose measurement.
Therefore, here we introduce how the Jacobian matrices of the pose measurement of the synchronous sensors is formulated.

Assume we have a global pose measurement of sensor $\{X\}$ at time $t_k$.
The measurement can be modeled with a IMU pose $\{I_{k}\}$ and spatial extrinsic calibration parameters $(\R{I}{X}, \p{I}{X})$ between the IMU and the sensor as:
\begin{align}
    \mathbf{z}_{X_k}
    := 
    \begin{bmatrix}
        \ang{E}{X_{k}} \\
        \p{X_{k}}{E}
    \end{bmatrix}
    =
    \begin{bmatrix}
        \Log{\R{I}{X}\R{E}{I_{k}}}\\
        \p{I_{k}}{E} + \R{I_{k}}{E}\p{X}{I}
    \end{bmatrix} + \mathbf{n}_{X_k} \label{eq:meas_synchronized_pose_meas}
\end{align}
We first drive the derivative of the measurement model in respect to the involved parameters:
\begin{align}
    \frac{\partial \ang{E}{X_{k}}}{\partial \ang{E}{I_{k}}} &= \R{I}{X} &
    \frac{\partial \ang{E}{X_{k}}}{\partial \ang{I}{X}} &= \mathbf{I}_3 \label{eq:sync_jacobi1}\\
\frac{\partial \p{X_{k}}{E}}{\partial \ang{E}{I_{k}}} &= -\R{I_{k}}{E} \skw{\p{X}{I}} &
    \frac{\partial \p{X_{k}}{E}}{\partial \p{I_{k}}{E}} &= \mathbf{I}_3 \label{eq:sync_jacobi2}\\
\frac{\partial \p{X_{k}}{E}}{\partial \ang{I}{X}} &= \R{X_k}{E} \skw{\p{I}{X}} & 
    \frac{\partial \p{X_{k}}{E}}{\partial \p{I}{X}} &= - \R{X_{k}}{E} \label{eq:sync_jacobi3}
\end{align}

To account for the difference between sensor clocks and measurement delay, we model an unknown constant time offset between the IMU clock and the sensor clock: ${}^It_k = {}^Xt_k + {}^Xt_I$, where ${}^I t_k$ and ${}^X t_k$ are the times when measurement $\mathbf{z}_{S_k}$ was collected in the IMU and the sensor's clocks, and ${}^Xt_I$ is the time offset between the two time references.
To get the synchronized pose measurement at IMU times ${}^It_{k}$, we use the current best estimate of the time offset ${}^X\hat{t}_I$ and get the measurement at ${}^X t_k = {}^It_k - {}^X\hat{t}_I$, whose corresponding time in the IMU clock is:
\begin{align*}
    {}^I \bar{t}_{k}  := {}^It_{k} - {}^X\hat{t}_I + {}^Xt_I = {}^It_{k} + {}^X\tilde{t}_I
\end{align*}
Note that the above equation revels the error of the time offset ${}^X\tilde{t}_I$ within the chosen measurement time ${}^I \bar{t}_{k}$.
We employ the following first-order approximation to account for this small time-offset error \citep{li2014online}:
\begin{align*}
    {}^{I({}^I \bar{t}_{k})}_E\mathbf{R} &= {}^{I({}^It_{k} + {}^X\tilde{t}_I)}_E\mathbf{R} \approx  (\mathbf{I} - \lfloor {}^{I_k}\boldsymbol{\omega} {}^X\tilde{t}_I \rfloor) {}^{I_k}_E\mathbf{R}\\
{}^E\mathbf{p}_{I({}^I \bar{t}_{k})} &= {}^E\mathbf{p}_{I({}^It_{k} + {}^X\tilde{t}_I)} \approx {}^E\mathbf{p}_{I_k} + {}^E\mathbf{v}_{I_k} {}^X\tilde{t}_I
\end{align*}
By replacing the above approximation to Eq.~\eqref{eq:meas_synchronized_pose_meas}, we get the following additional Jacobians which are related to the time offset:
\begin{align*}
    \frac{\partial \ang{E}{X_{k}}}{\partial {}^Xt_I} &= \frac{\partial \ang{E}{X_{k}}}{\partial \ang{E}{I_{k}}}\frac{\partial \ang{E}{I_{k}}}{\partial {}^Xt_I} = \R{I}{X} {}^{I_k}\boldsymbol{\omega}\\
\frac{\partial \p{X_{k}}{E}}{\partial {}^Xt_I} 
    &= 
    \frac{\partial \p{X_{k}}{E}}{\partial \p{I_{k}}{E}}\frac{\partial \p{I_{k}}{E}}{\partial {}^Xt_I} + 
    \frac{\partial \p{X_{k}}{E}}{\partial \ang{E}{I_{k}}}\frac{\partial \ang{E}{I_{k}}}{\partial {}^Xt_I} \\
    &= \vel{I_k}{E} -\R{I_{k}}{E} \skw{\p{X}{I}} {}^{I_k}\boldsymbol{\omega} \notag
\end{align*}
Finally, we get the following full Jacobian matrix of the synchronized measurement:
\begin{align*}
    \HS{X,k} &= \frac{\partial \mathbf{z}_{X_k}}{\partial \x{k}} \\
&= 
\setlength\arraycolsep{2.5pt}
    \begin{bmatrix}
        \HS{00} & \O{} & \O{} & \cdots & \HS{07} & \O{} & \HS{09} & \O{} & \cdots\\
        \undermat{\{I_k\}}{\HS{10} & \HS{11}} & \O{}  & \cdots & \undermat{\R{I}{X}}{\HS{17}} & \undermat{\p{I}{X}}{\HS{18}} & \undermat{\t{I}{X}}{\HS{19}} & \O{} & \cdots
    \end{bmatrix}\\\\
\HS{00} &= \frac{\partial \ang{E}{X_{k}}}{\partial \ang{E}{I_{k}}} = \R{I}{X} \\
    \HS{07} &= \frac{\partial \ang{E}{X_{k}}}{\partial \ang{I}{X}} = \mathbf{I}_3 \\
    \HS{09} &= \frac{\partial \ang{E}{X_{k}}}{\partial {}^Xt_I} = \R{I}{X} {}^{I_k}\boldsymbol{\omega} \\
\HS{10} &= \frac{\partial \p{X_{k}}{E}}{\partial \ang{E}{I_{k}}} = -\R{I_{k}}{E} \skw{\p{X}{I}} \\
    \HS{11} &= \frac{\partial \p{X_{k}}{E}}{\partial \p{I_{k}}{E}} = \mathbf{I}_3 \\
    \HS{17} &= \frac{\partial \p{X_{k}}{E}}{\partial \ang{I}{X}} = \R{X_k}{E} \skw{\p{I}{X}} \\
    \HS{18} &= \frac{\partial \p{X_{k}}{E}}{\partial \p{X}{I}} = - \R{X_{k}}{E} \\
    \HS{19} &= \frac{\partial \p{X_{k}}{E}}{\partial {}^Xt_I} = \vel{I_k}{E} -\R{I_{k}}{E} \skw{\p{X}{I}} {}^{I_k}\boldsymbol{\omega}
\end{align*}

\section{Jacobians of Asynchronous Sensors: Linear Interpolation} \label{ch:apdx_interpolation}

Here we assume an asynchronous global pose measurement (see Eq.~\eqref{eq:X_pose_measurement}) of a sensor $\{X\}$ is given at time $t_k'$ ($t_{k \shortminus 1} \leq t_k' + \t{I}{X} \leq t_{k}$) and drive the full Jacobian of the measurement model in respect to the two bounding IMU poses ($\{I_{k \shortminus 1}\}$ and $\{I_{k}\}$) and the spatiotemporal extrinsic calibration parameters ($\R{I}{X}$, $\p{I}{X}$, and ${}^X t_I$) using the linear interpolation.
The measurement model and the interpolation can be written as:
\begin{align*}
    \mathbf{z}_{X_k}
    &=
    \begin{bmatrix}
        \ang{E}{X_{k'}} \\
        \p{X_{k'}}{E}
    \end{bmatrix}
    =
    \begin{bmatrix}
        \Log{\R{I}{X}\R{E}{I_{k'}}}\\
        \p{I_{k'}}{E} + \R{I_{k'}}{E}\p{X}{I}
    \end{bmatrix} + \mathbf{n}_{X_k}\\
\R{E}{I_{k'}} &= \textrm{Exp}(\lambda \textrm{Log}(\R{E}{I_{k}}\R{I_{k \shortminus 1}}{E}))\R{E}{I_{k \shortminus 1}}\\
\p{I_{k'}}{E} &= (1 - \lambda)\p{I_{k \shortminus 1}}{E} + \lambda\p{I_{k}}{E}\\
\lambda &= (t_{k}' + {}^Xt_{I} - t_{k \shortminus 1})/(t_{k} - t_{k \shortminus 1}) 
\end{align*}
The derivative of the measurement model in respect to the involved parameters are the same as synchronized measurement (see Eq.~\eqref{eq:sync_jacobi1}, \eqref{eq:sync_jacobi2}, and \eqref{eq:sync_jacobi3}).
Then the derivative of the interpolated IMU pose in respect to the bounding IMU poses and the temporal extrinsic calibration parameter are derived as:
\begin{align*}
    \frac{\partial \ang{E}{I_{k'}}}{\partial \ang{E}{I_{k \shortminus 1}}} &= \boldsymbol{\Upsilon}_1 =  \Jr{\lambda \ang{I_{k}}{I_{k \shortminus 1}}} (\Jlinv{\lambda \ang{I_{k}}{I_{k \shortminus 1}}} - \lambda \Jlinv{\ang{I_{k}}{I_{k \shortminus 1}}}) \\
\frac{\partial \ang{E}{I_{k'}}}{\partial \ang{E}{I_{k}}} &= \boldsymbol{\Upsilon}_2 = \lambda \Jr{\lambda \ang{I_{k}}{I_{k \shortminus 1}}} \Jrinv{\ang{I_{k}}{I_{k \shortminus 1}}} \\
\frac{\partial \p{I_{k'}}{E}}{\partial \p{I_{k \shortminus 1}}{E}} &= (1 - \lambda) \mathbf{I}_3 
~~~~~~~~~~ \frac{\partial \p{I_{k'}}{E}}{\partial \p{I_{k}}{E}} = \lambda \mathbf{I}_3\\
\frac{\partial \ang{E}{I_{k'}}}{\partial {}^Xt_I} &= \frac{\ang{I_{k}}{I_{k \shortminus 1}}}{t_{k} - t_{k \shortminus 1}}   
~~~~~~~~~~ \frac{\partial \p{I_{k'}}{E}}{\partial {}^Xt_I} = \frac{\p{I_{k}}{E} - \p{I_{k \shortminus 1}}{E}}{t_{k} - t_{k \shortminus 1}}
\end{align*}
where $\mathbf{J_l}$ and $\mathbf{J_r}$ are the left and right Jacobian matrices of SO(3) \citep{chirikjian2011stochastic}. 
Finally, we apply the chain rule to get the full Jacobian matrix $\HA{X,k}$:
\begin{align*}
    \HA{X,k} &= \frac{\partial \mathbf{z}_{X_k}}{\partial \x{k}} \\
&= 
\setlength\arraycolsep{1.5pt}
    \begin{bmatrix}
        \HA{1} & \O{} & \O{} & \HA{2} & \O{} & \O{} & \cdots & \HA{3} & \O{} & \HA{4} & \cdots\\
        \undermat{\{I_k\}}{\HA{5} & \HA{6}} & \O{} & \undermat{\{I_{k \shortminus 1}\}}{\HA{7} & \HA{8}} & \O{} & \cdots & \undermat{\R{I}{X}}{\HA{9}} & \undermat{\p{I}{X}}{\HA{10}} & \undermat{\t{I}{X}}{\HA{11}} &  \cdots
    \end{bmatrix}\\\\
\HA{1} &= \frac{\partial \ang{E}{X_{k'}}}{\partial \ang{E}{I_{k'}}} \frac{\partial \ang{E}{I_{k'}}}{\partial \ang{E}{I_{k}}} = \R{I}{X} \boldsymbol{\Upsilon}_2 \\ 
    \HA{2} &= \frac{\partial \ang{E}{X_{k'}}}{\partial \ang{E}{I_{k'}}} \frac{\partial \ang{E}{I_{k'}}}{\partial \ang{E}{I_{k \shortminus 1}}} = \R{I}{X} \boldsymbol{\Upsilon}_1 \\ 
    \HA{3} &= \frac{\partial \ang{E}{X_{k'}}}{\partial \ang{I}{X}} = \mathbf{I}_3 \\ 
    \HA{4} &= \frac{\partial \ang{E}{X_{k'}}}{\partial \ang{E}{I_{k'}}} \frac{\partial \ang{E}{I_{k'}}}{\partial {}^At_I} = \R{I}{X} \frac{1}{t_{k} - t_{k \shortminus 1}} \ang{I_{k}}{I_{k \shortminus 1}} \\
\HA{5} &= \frac{\partial \p{X_{k'}}{E}}{\partial \ang{E}{I_{k'}}} \frac{\partial \ang{E}{I_{k'}}}{\partial \ang{E}{I_{k}}} = -\R{I_{k'}}{E} \skw{\p{X}{I}} \boldsymbol{\Upsilon}_2 \\ 
    \HA{6} &= \frac{\partial \p{X_{k'}}{E}}{\partial \p{I_{k'}}{E}} \frac{\partial \p{I_{k'}}{E}}{\partial \p{I_{k}}{E}} = \lambda \mathbf{I}_3\\ 
    \HA{7} &= \frac{\partial \p{X_{k'}}{E}}{\partial \ang{E}{I_{k'}}} \frac{\partial \ang{E}{I_{k'}}}{\partial \ang{E}{I_{k \shortminus 1}}} = -\R{I_{k'}}{E} \skw{\p{X}{I}} \boldsymbol{\Upsilon}_1 \\ 
    \HA{8} &= \frac{\partial \p{X_{k'}}{E}}{\partial \p{I_{k'}}{E}} \frac{\partial \p{I_{k'}}{E}}{\partial \p{I_{k \shortminus 1}}{E}} = (1 - \lambda) \mathbf{I}_3\\ 
    \HA{9} &= \frac{\partial \p{X_{k'}}{E}}{\partial \ang{I}{X}} = \R{X_{k'}}{E} \skw{\p{I}{X}} \\ 
    \HA{10} &= \frac{\partial \p{X_{k'}}{E}}{\partial \p{I}{X}} = - \R{X_{k'}}{E}\\ 
    \HA{11} &= \frac{\partial \p{X_{k'}}{E}}{\partial \ang{E}{I_{k'}}}\frac{\partial \ang{E}{I_{k'}}}{\partial {}^At_I} + \frac{\partial \p{X_{k'}}{E}}{\partial \p{I_{k'}}{E}}\frac{\partial \p{I_{k'}}{E}}{\partial {}^At_I} \\
    &= \frac{\p{I_{k}}{E} - \p{I_{k \shortminus 1}}{E} -\R{I_{k'}}{E} \skw{\p{X}{I}}  \ang{I_{k}}{I_{k \shortminus 1}}}{t_{k} - t_{k \shortminus 1}}
\end{align*} 
\section{Jacobians of Asynchronous Sensors: High-order Interpolation} \label{ch:apdx_interpolation_high}

Here we assume an asynchronous global pose measurement (see Eq.~\eqref{eq:X_pose_measurement}) of a sensor $\{X\}$ is given at time $t_k$ and drive the full Jacobian of the measurement model in respect to the IMU poses (say $\{I_0\}, \hdots, \{I_n\}$) used in $n+1$ order interpolation and the spatiotemporal extrinsic calibration parameters ($\R{I}{X}$, $\p{I}{X}$, and ${}^X t_I$).
The measurement model and the interpolation can be written as \citep{Eckenhoff2021TRO}:
\begin{align*}
    \mathbf{z}_{X_k}
    &=
    \begin{bmatrix}
        \ang{E}{X_{k}} \\
        \p{X_{k}}{E}
    \end{bmatrix}
    =
    \begin{bmatrix}
        \Log{\R{I}{X}\R{E}{I_{k}}}\\
        \p{I_{k}}{E} + \R{I_{k}}{E}\p{X}{I}
    \end{bmatrix} + \mathbf{n}_{X_k}\\
\R{E}{I_k} &= \Exp{\sum_{i = 1}^{n} \mathbf{a}_i \Delta t_k^i} \R{E}{I_0} \\
    \p{I_k}{E} &= \p{I_0}{E} + \sum_{i = 1}^{n} \mathbf{b}_i\Delta t_k^i
\end{align*}
where $\Delta t_k = t_k + {}^{X}t_{I} - t_0$,
Finally, we apply the chain rule to get the full Jacobian matrix $\HA{X,k}$:
\begin{align*}
    \HA{X,k} &= \frac{\partial \mathbf{z}_{X_k}}{\partial \x{k}} \\
&= 
\setlength\arraycolsep{1.5pt}
    \begin{bmatrix}
       \cdots & \HA{1} & \O{} & \cdots & \HA{2} & \O{} &\cdots & \HA{3} & \O{} & \HA{4} & \cdots\\
       \cdots & \undermat{\{I_i\}}{\HA{5} & \HA{6}} &  \cdots & \undermat{\{I_{0}\}}{\HA{7} & \HA{8}} & \cdots & \undermat{\R{I}{X}}{\HA{9}} & \undermat{\p{I}{X}}{\HA{10}} & \undermat{\t{I}{X}}{\HA{11}} & \cdots
    \end{bmatrix}\\\\
\HA{1} &= \frac{\partial \ang{E}{X_{k}}}{\partial \ang{E}{I_{k}}} \frac{\partial \ang{E}{I_{k}}}{\partial \ang{E}{I_{i}}} = -\R{I}{X}\Jl{\mathbf{M}_{t}\mathbf{a}}\mathbf{M}_{t}\mathbf{V}^{-1}\mathbf{J_{A,0}} \\ 
    \HA{2} &= \frac{\partial \ang{E}{X_{k}}}{\partial \ang{E}{I_{k}}} \frac{\partial \ang{E}{I_{k}}}{\partial \ang{E}{I_{0}}} \\ 
           &= -\R{I}{X}(\Jl{\mathbf{M}_{t}\mathbf{a}}\mathbf{M}_{t}\mathbf{V}^{-1}\mathbf{J_{A,i}} + \Exp{\mathbf{M}_{t}\mathbf{a}})\\ 
    \HA{3} &= \frac{\partial \ang{E}{X_{k}}}{\partial \ang{I}{X}} = \mathbf{I}_3 \\ 
    \HA{4} &= \frac{\partial \ang{E}{X_{k}}}{\partial \ang{E}{I_{k}}} \frac{\partial \ang{E}{I_{k}}}{\partial {}^At_I} = -\R{I}{X} \Jl{\mathbf{M}_{t} \mathbf{a}} \mathbf{M}_{dt} \mathbf{a}  \\
\HA{5} &= \frac{\partial \p{X_{k}}{E}}{\partial \ang{E}{I_{k}}} \frac{\partial \ang{E}{I_{k}}}{\partial \ang{E}{I_{k}}} = \R{I_{k}}{E} \skw{\p{X}{I}} \Jl{\mathbf{M}_{t}\mathbf{a}}\mathbf{M}_{t}\mathbf{V}^{-1}\mathbf{J_{A,0}} \\ 
    \HA{6} &= \frac{\partial \p{X_{k}}{E}}{\partial \p{I_{k}}{E}} \frac{\partial \p{I_{k}}{E}}{\partial \p{I_{k}}{E}} = \I{3} + \mathbf{M}_t \mathbf{V}^{-1} \mathbf{J_{B,0}}\\ 
    \HA{7} &= \frac{\partial \p{X_{k}}{E}}{\partial \ang{E}{I_{k}}} \frac{\partial \ang{E}{I_{k}}}{\partial \ang{E}{I_{k \shortminus 1}}} \\
           &= \R{I_{k}}{E} \skw{\p{X}{I}} (\Jl{\mathbf{M}_{t}\mathbf{a}}\mathbf{M}_{t}\mathbf{V}^{-1}\mathbf{J_{A,i}} + \Exp{\mathbf{M}_{t}\mathbf{a}}) \\ 
    \HA{8} &= \frac{\partial \p{X_{k}}{E}}{\partial \p{I_{k}}{E}} \frac{\partial \p{I_{k}}{E}}{\partial \p{I_{k \shortminus 1}}{E}} = \mathbf{M}_t \mathbf{V}^{-1} \mathbf{J_{B,i}}\\ 
    \HA{9} &= \frac{\partial \p{X_{k}}{E}}{\partial \ang{I}{X}} = \R{X_{k}}{E} \skw{\p{I}{X}} \\ 
    \HA{10} &= \frac{\partial \p{X_{k}}{E}}{\partial \p{I}{X}} = - \R{X_{k}}{E}\\ 
    \HA{11} &= \frac{\partial \p{X_{k}}{E}}{\partial \ang{E}{I_{k}}}\frac{\partial \ang{E}{I_{k}}}{\partial {}^At_I} + \frac{\partial \p{X_{k}}{E}}{\partial \p{I_{k}}{E}}\frac{\partial \p{I_{k}}{E}}{\partial {}^At_I} \\
    &=\R{I_{k}}{E} \skw{\p{X}{I}} \Jl{\mathbf{M}_{t} \mathbf{a}} \mathbf{M}_{dt} \mathbf{a} + \mathbf{M}_{dt} \mathbf{b} \\
\text{where} \\
    \mathbf{M}_{t} &= \begin{bmatrix} \Delta t_k \I{3} & \Delta t_k^2 \I{3} & \cdots & \Delta t_k^n \I{3} \end{bmatrix} \\
\mathbf{M}_{dt} &= \begin{bmatrix} \I{3} & 2\Delta t_k \I{3} & \cdots & n\Delta t_k^{n-1} \I{3} \end{bmatrix} \\
\mathbf{V} &= \begin{bmatrix}
         \Delta t_1  & \cdots & \Delta t_1^n \\
         \vdots & \ddots & \vdots \\
         \Delta t_n  & \cdots & \Delta t_n^n
    \end{bmatrix} \\
\mathbf{a} &= \begin{bmatrix} \mathbf{a}_1^\top & \mathbf{a}_2^\top & \cdots & \mathbf{a}_n^\top \end{bmatrix}^\top \\
\mathbf{b} &= \begin{bmatrix} \mathbf{b}_1^\top & \mathbf{b}_2^\top & \cdots & \mathbf{b}_n^\top \end{bmatrix}^\top \\
\mathbf{J_{A,0}} &= \begin{bmatrix}
         -\Jlinv{\Log{\R{E}{I_1}\R{I_0}{E}}} \\ \vdots  \\ -\Jlinv{\Log{\R{E}{I_n}\R{I_0}{E}}}
    \end{bmatrix} \\
\mathbf{J_{A,i}} &= \begin{bmatrix}
        \O{3} \\ \vdots \\ \Jlinv{\Log{\R{E}{I_i}\R{I_0}{E}}}\R{E}{I_i}\R{I_0}{E} \\ \vdots \\ \O{3}
    \end{bmatrix} \\
\mathbf{J_{B,0}} &= \begin{bmatrix}
         -\I{3} & \cdots & -\I{3}
    \end{bmatrix}^\top \\
\mathbf{J_{B,i}} &= \begin{bmatrix}
         \O{3} & \cdots & \I{3} & \cdots & \O{3}
    \end{bmatrix}^\top \\
\end{align*}

\end{document}